\def\isarxiv{1}

\ifdefined\isarxiv

\documentclass[sigconf,authorversion]{acmart}
\setcopyright{none}
\renewcommand\footnotetextcopyrightpermission[1]{}
\else

\documentclass[sigconf]{acmart}

\fi

\usepackage[utf8]{inputenc} 
\usepackage[T1]{fontenc}    
\usepackage{hyperref}       
\usepackage{url}            
\usepackage{booktabs}       
\usepackage{amsfonts}       
\usepackage{nicefrac}       
\usepackage{microtype}      
\usepackage{xcolor}         

\usepackage{bm}
\usepackage{geometry}
\usepackage{multirow}
\usepackage{booktabs}
\usepackage{algorithm}
\usepackage{algorithmicx}
\usepackage{algpseudocode}
\usepackage{amsmath}
\usepackage{comment}
\usepackage{ulem}
\usepackage{tikz}
\usepackage{amsthm}
\usepackage{enumitem}
\usepackage{subcaption}
\usepackage[most]{tcolorbox} 

\allowdisplaybreaks[4]

\makeatletter
\def\th@acmplain{%
  \itshape 
  \setlength{\parindent}{0pt}%
  \thm@headfont{\bfseries\scshape}
  \thm@notefont{\normalfont\scshape}
}
\makeatother

\definecolor{lightgraybg}{RGB}{235,235,235}
\newtcolorbox{graybox}{
    colback=lightgraybg,
    colframe=lightgraybg,
    fonttitle=\bfseries,
    boxsep=2pt,
    left=0pt,
    right=0pt,
    top=0pt,
    bottom=0pt,
    breakable,
    arc=0mm,
    before skip=\topsep,
    after skip=\topsep
}

\newtheorem{theorem}{Theorem}[section]

\newtheorem{claim}[theorem]{Claim}
\newtheorem{corollary}[theorem]{Corollary}
\newtheorem{lemma}[theorem]{Lemma}
\newtheorem{proposition}[theorem]{Proposition}
\newtheorem{definition}[theorem]{Definition}
\newtheorem{remark}[theorem]{Remark}
\newtheorem{fact}[theorem]{Fact}

\tcolorboxenvironment{theorem}{%
    colback=lightgraybg,
    colframe=lightgraybg,
    fonttitle=\bfseries,
    boxsep=2pt,
    left=0pt,
    right=0pt,
    top=0pt,
    bottom=0pt,
    breakable,
    arc=0mm,
    before skip=\topsep,
    after skip=\topsep
}

\tcolorboxenvironment{problem}{%
    colback=lightgraybg,
    colframe=lightgraybg,
    fonttitle=\bfseries,
    boxsep=2pt,
    left=0pt,
    right=0pt,
    top=0pt,
    bottom=0pt,
    breakable,
    arc=0mm,
    before skip=\topsep,
    after skip=\topsep
}

\tcolorboxenvironment{claim}{%
    colback=lightgraybg,
    colframe=lightgraybg,
    fonttitle=\bfseries,
    boxsep=2pt,
    left=0pt,
    right=0pt,
    top=0pt,
    bottom=0pt,
    breakable,
    arc=0mm,
    before skip=\topsep,
    after skip=\topsep
}

\tcolorboxenvironment{corollary}{%
    colback=lightgraybg,
    colframe=lightgraybg,
    fonttitle=\bfseries,
    boxsep=2pt,
    left=0pt,
    right=0pt,
    top=0pt,
    bottom=0pt,
    breakable,
    arc=0mm,
    before skip=\topsep,
    after skip=\topsep
}

\tcolorboxenvironment{lemma}{%
    colback=lightgraybg,
    colframe=lightgraybg,
    fonttitle=\bfseries,
    boxsep=2pt,
    left=0pt,
    right=0pt,
    top=0pt,
    bottom=0pt,
    breakable,
    arc=0mm,
    before skip=\topsep,
    after skip=\topsep
}

\tcolorboxenvironment{proposition}{%
    colback=lightgraybg,
    colframe=lightgraybg,
    fonttitle=\bfseries,
    boxsep=2pt,
    left=0pt,
    right=0pt,
    top=0pt,
    bottom=0pt,
    breakable,
    arc=0mm,
    before skip=\topsep,
    after skip=\topsep
}

\tcolorboxenvironment{definition}{%
    colback=lightgraybg,
    colframe=lightgraybg,
    fonttitle=\bfseries,
    boxsep=2pt,
    left=0pt,
    right=0pt,
    top=0pt,
    bottom=0pt,
    breakable,
    arc=0mm,
    before skip=\topsep,
    after skip=\topsep
}

\tcolorboxenvironment{remark}{%
    colback=lightgraybg,
    colframe=lightgraybg,
    fonttitle=\bfseries,
    boxsep=2pt,
    left=0pt,
    right=0pt,
    top=0pt,
    bottom=0pt,
    breakable,
    arc=0mm,
    before skip=\topsep,
    after skip=\topsep
}

\tcolorboxenvironment{fact}{%
    colback=lightgraybg,
    colframe=lightgraybg,
    fonttitle=\bfseries,
    boxsep=2pt,
    left=0pt,
    right=0pt,
    top=0pt,
    bottom=0pt,
    breakable,
    arc=0mm,
    before skip=\topsep,
    after skip=\topsep
}

\newcommand{\R}{\mathbb{R}}
\newcommand{\wh}[1]{\widehat{#1}}
\newcommand{\wt}[1]{\widetilde{#1}}
\renewcommand{\d}{\mathrm{d}}
\newcommand{\argmin}{\mathop{\arg\min}}

\newcommand{\thetab}{\boldsymbol{\theta}}

\newcommand{\taub}{\boldsymbol{\tau}}
\newcommand{\Xib}{\boldsymbol{\Xi}}

\newcommand{\eb}{\mathbf{e}}

\newcommand{\hb}{\mathbf{h}}

\newcommand{\mb}{\mathbf{m}}

\newcommand{\pb}{\mathbf{p}}
\newcommand{\qb}{\mathbf{q}}

\newcommand{\wb}{\mathbf{w}}
\newcommand{\xb}{\mathbf{x}}
\newcommand{\yb}{\mathbf{y}}
\newcommand{\zb}{\mathbf{z}}

\newcommand{\Ab}{\mathbf{A}}
\newcommand{\Bb}{\mathbf{B}}
\newcommand{\Cb}{\mathbf{C}}
\newcommand{\Db}{\mathbf{D}}

\newcommand{\Hb}{\mathbf{H}}
\newcommand{\Ib}{\mathbf{I}}

\newcommand{\Pb}{\mathbf{P}}

\newcommand{\Tb}{\mathbf{T}}

\newcommand{\Xb}{\mathbf{X}}

\newcommand{\Zb}{\mathbf{Z}}

\newcommand{\Ecal}{\mathcal{E}}
\newcommand{\Fcal}{\mathcal{F}}
\newcommand{\Gcal}{\mathcal{G}}

\newcommand{\Kcal}{\mathcal{K}}
\newcommand{\Lcal}{\mathcal{L}}

\newcommand{\Ncal}{\mathcal{N}}

\newcommand{\Pcal}{\mathcal{P}}
\newcommand{\Qcal}{\mathcal{Q}}

\newcommand{\Vcal}{\mathcal{V}}

\newcommand{\Xcal}{\mathcal{X}}

\newcommand{\AppendixName}{\textit{Supplementary}}

\AtBeginDocument{%
  }

\setcopyright{acmlicensed}

\copyrightyear{2026}
\acmYear{2026}
\setcopyright{cc}
\setcctype{by}
\acmConference[WSDM '26]{Proceedings of the Nineteenth ACM International Conference on Web Search and Data Mining}{February 22--26, 2026}{Boise, ID, USA}
\acmBooktitle{Proceedings of the Nineteenth ACM International Conference on Web Search and Data Mining (WSDM '26), February 22--26, 2026, Boise, ID, USA}
\acmPrice{}
\acmDOI{10.1145/3773966.3777929}
\acmISBN{979-8-4007-2292-9/2026/02}

\begin{document}

\title{Unlearning Inversion Attacks for Graph Neural Networks}


\author{Jiahao Zhang}
\affiliation{%
  \institution{The Pennsylvania State University}
  \city{University Park}
  \country{United States}
  }
\email{jiahao.zhang@psu.edu}
\author{Yilong Wang}
\affiliation{%
  \institution{The Pennsylvania State University}
  \city{University Park}
  \country{United States}
  }
\email{yvw5769@psu.edu}
\author{Zhiwei Zhang}
\affiliation{%
  \institution{The Pennsylvania State University}
  \city{University Park}
  \country{United States}
  }
\email{zbz5349@psu.edu}
\author{Xiaorui Liu}
\affiliation{%
  \institution{North Carolina State University}
  \city{Raleigh}
  \country{United States}
  }
\email{xliu96@ncsu.edu}
\author{Suhang Wang}
\affiliation{%
  \institution{The Pennsylvania State University}
  \city{University Park}
  \country{United States}
  }
\email{szw494@psu.edu}

\renewcommand{\shortauthors}{Zhang et al.}

\begin{abstract}
  
Graph unlearning methods aim to efficiently remove the impact of sensitive data from trained GNNs without full retraining, assuming that deleted information cannot be recovered. In this work, we challenge this assumption by introducing the \emph{graph unlearning inversion} attack: given only black-box access to an unlearned GNN and partial graph knowledge, can an adversary reconstruct the removed edges? We identify two key challenges: varying probability-similarity thresholds for unlearned versus retained edges, and the difficulty of locating unlearned edge endpoints, and address them with \textbf{TrendAttack}. First, we derive and exploit the \emph{confidence pitfall}, a theoretical and empirical pattern showing that nodes adjacent to unlearned edges exhibit a large drop in model confidence. Second, we design an adaptive prediction mechanism that applies different similarity thresholds to unlearned and other membership edges. Our framework flexibly integrates existing membership inference techniques and extends them with trend features. Experiments on four real-world datasets demonstrate that TrendAttack significantly outperforms state-of-the-art GNN membership inference baselines, exposing a critical privacy vulnerability in current graph unlearning methods. \textbf{The source code is available on GitHub\footnote{{\color{blue}\bf \url{https://github.com/QwQ2000/WSDM26-Graph-Unlearning-Inversion}}}}.

\end{abstract}




\maketitle

\section{Introduction}

Graph-structured data is prevalent in numerous real-world applications, such as recommender systems~\cite{he2020lightgcn,zhang2024linear,liu2026continuous}, social media platforms~\cite{fan2019graph,sankar2021graph,liu2024score}, financial transaction networks~\cite{dou2020enhancing,liu2020alleviating,wang2025enhance}, and computational biology~\cite{jing2021fast,fan2025computational,xu2025dualequinet}. Graph Neural Networks (GNNs)~\cite{hamilton2017inductive,velickovic2018graph,wu2019simplifying} have emerged as powerful tools for modelling such data, leveraging their ability to capture both node attributes and graph topology. The effectiveness of GNNs relies on the message-passing mechanism~\cite{gilmer2017neural,feng2022powerful}, which iteratively propagates information between nodes and their neighbors. This enables GNNs to generate rich node representations, facilitating key tasks like node classification~\cite{kipf2017semi,xue2023lazygnn,lin2024trojan}, link prediction~\cite{zhang2018link,yang2024conversational,wang2025bridging}, and graph classification~\cite{lee2018graph,fu2021sagn,al2025graph}.

Despite their success, GNNs raise significant concerns about privacy risks due to the sensitive nature of graph data~\cite{sajadmanesh2021locally,dai2024comprehensive}. Real-world datasets often contain private information, such as purchasing records in recommender systems~\cite{zhang2021membership,xin2023user} or loan histories in financial networks~\cite{pedarsani2011privacy,wang2024data}. During training, GNNs inherently encode such sensitive information into their model parameters. When these trained models are shared via model APIs, privacy breaches may occur. These risks have led to regulations like the GDPR~\cite{mantelero2013eu}, CCPA~\cite{ccpa}, and PIPEDA~\cite{pipeda}, which enforce the \emph{right to be forgotten}, allowing users to request the removal of their personal data from systems and models. This demand has necessitated the development of methods to remove the influence of specific data points from trained GNNs, a process known as \emph{graph unlearning}~\cite{chien2022certified,chen2022graph,wu2023gif}.

A straightforward way of graph unlearning is to retrain the model from scratch on a cleaned training graph. However, this approach is computationally infeasible for large-scale graphs, such as purchase networks in popular e-commerce platforms~\cite{wang2018billion,jin2023amazon} and social networks on social media~\cite{armstrong2013linkbench,ching2015one}, which may involve billions of nodes and edges. 
To address this, a wide range of recent graph unlearning methods focus on efficient parameter manipulation techniques~\cite{wu2023gif,wu2023certified,wu2024graphguard}, which approximate the removal of data by adjusting model parameters based on their influence. 
After unlearning, membership inference attacks for GNNs~\cite{olatunji2021membership,he2021stealing,zhang2023demystifying} could become unable to determine whether the unlearned edges were part of the training set, as the GNN no longer memorizes them. 
Hence, these techniques are often regarded as privacy-preserving, assuming that the unlearned data cannot be reconstructed. 

\begin{figure*}[!t]
  \centering
  \includegraphics[width=0.8\linewidth]{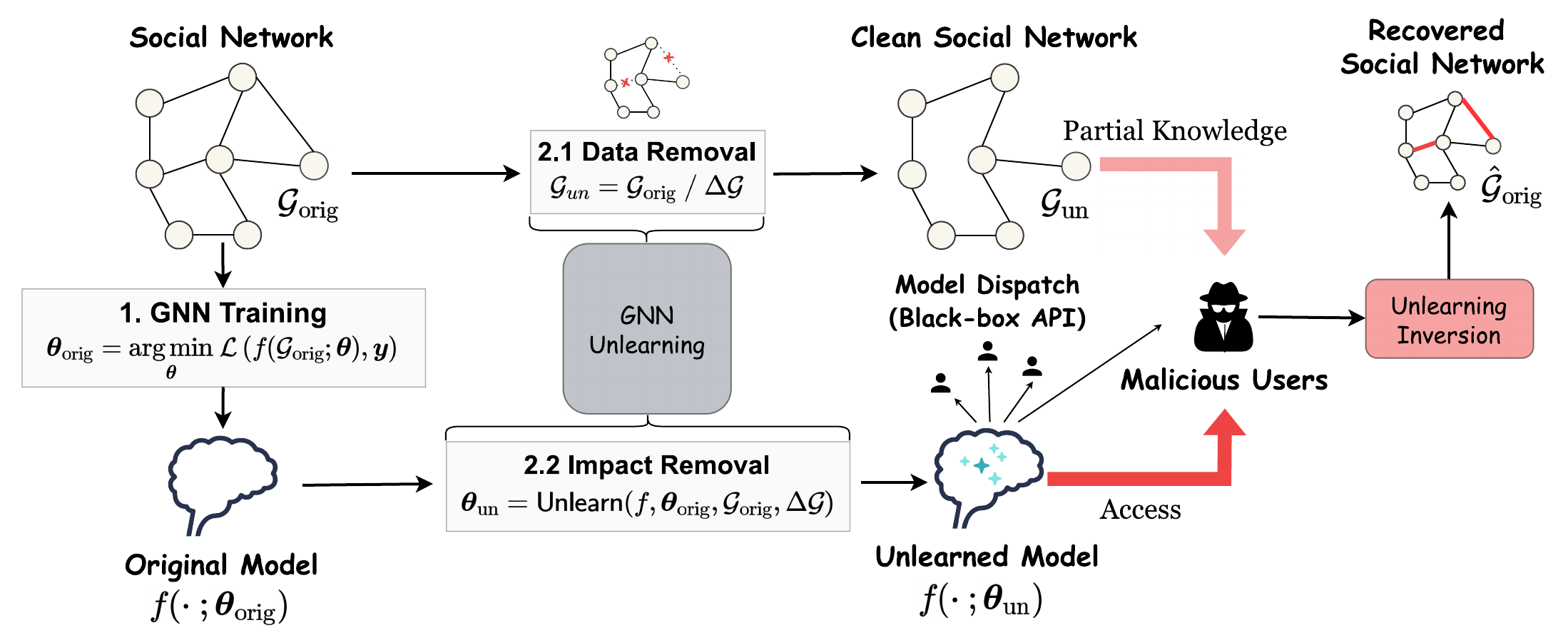}
  \vskip -1.4em
  \caption{\textbf{Illustration of the unlearning inversion attack}. Considering an online social network $\mathcal{G}_{\mathrm{orig}}$, where a user requests the deletion of sensitive friendship information, resulting in a cleaned graph $\mathcal{G}_{\mathrm{un}}$ and updated model parameters $\thetab_{\mathrm{un}}$. The GNN model may be shared with third-parties via black-box APIs. If an attacker, leveraging the model API and auxiliary information about $\mathcal{G}_{\mathrm{un}}$, can reconstruct the removed knowledge $\Delta \mathcal{G}$ through an unlearning inversion attack, sensitive relationships may be exposed, severely compromising user privacy.}
  \label{fig:prob_setting}
  \vskip -0.15in
\end{figure*}

In this paper, we challenge the assumption that existing graph unlearning methods are robust to privacy attacks. 
Specifically, we study a novel and important problem of whether unlearned information can be recovered through
\emph{unlearning inversion attacks}~\cite{hu2024learn} as illustrated in Figure~\ref{fig:prob_setting}. This attack is important because it reveals critical privacy vulnerabilities in graph unlearning methods and shows a concerning scenario where “forgotten” user information on the Web can be reconstructed through carefully designed attacks. Our central research question is: 

\begin{graybox}
    {\it Can third parties exploit model APIs of unlearned GNNs to recover sensitive information that was meant to be forgotten?}
\end{graybox}

Although several pioneering studies have examined the privacy risks of unlearned machine learning (ML) models~\cite{chen2021machine,thudi2022necessity,hu2024learn,bertran2024reconstruction}, unlearning inversion attacks on GNNs still remain largely unexplored. Existing general-purpose unlearning inversion attacks~\cite{hu2024learn,bertran2024reconstruction} typically adopt a ``two-model'' setting, where the attacker is assumed to have access to model parameters~\cite{bertran2024reconstruction} or output probabilities~\cite{hu2024learn} both before and after unlearning. This assumption is often unrealistic in practice, as pre-unlearning models may contain sensitive information and are usually inaccessible to third parties, necessitating a more practical ``one-model'' setting where the attacker only has access to the model after unlearning. Moreover, these methods are designed for i.i.d. data (e.g., images) and fail to capture the non-i.i.d. nature of graphs, including structural dependencies and the propagation of unlearning effects across the network. It is also worth noting that other privacy attacks on GNNs, such as membership inference attacks (MIAs)~\cite{he2021node,he2021stealing,olatunji2021membership,zhang2022inference}, may also be less effective in this context, as they aim to extract knowledge explicitly memorized by the model, while unlearned GNNs have already forgotten most of this sensitive information.

Therefore, in this work, we study a novel problem of link-level unlearning inversion attack for black-box unlearned GNN models, with the goal of accurately recovering unlearned links using only black-box GNN outputs and partial knowledge of the unlearned graph $\mathcal{G}_{\mathrm{un}}$. 
A straightforward solution is to train a link classifier on the output probabilities of the unlearned GNN model $f(\cdot;\thetab_{\mathrm{un}})$ to decide whether an edge belongs to the training data (regardless of whether the edge was unlearned). 
However, this naive approach faces two major technical challenges:
(i) The output probabilities of two nodes connected by an unlearned edge often do not show a strong similarity signal, making these unlearned edges hard to separate from ordinary non-existent edges. 
(ii) Even if we could measure such similarity, it is fundamentally difficult to identify which nodes are associated with unlearned edges, because we only observe the post-unlearning model and have no access to the pre-unlearning model as a reference. 

In response to these challenges, we propose a novel link-level unlearning inversion attack for black-box unlearned GNN models, namely \textbf{TrendAttack}. To address challenge (i), we design an adaptive prediction mechanism that applies different thresholds to infer two types of training graph edges: unlearned edges and other membership edges, enhancing TrendAttack's flexibility to accommodate varying similarity levels. For challenge (ii), we identify a key phenomenon called the \emph{confidence pitfall}, which enables the distinction between nodes connected to unlearned edges and others, using only black-box model outputs. This phenomenon describes how the model’s confidence in nodes near unlearned edges tends to decrease, and is supported by both empirical and theoretical evidence, as detailed in Section~\ref{sec:design_motivation}. By jointly incorporating the adaptive treatment of different edge types and the confidence trend pattern, TrendAttack achieves strong membership inference performance for both unlearned and other membership edges. Our \textbf{main contributions} are summarized as follows: 
\begin{itemize}[leftmargin=*]
    \item We formulate a novel problem of graph unlearning inversion attacks (Section~\ref{sec:problem}), highlighting the vulnerability of existing graph unlearning methods. {\bf This work is among the first to study GNN unlearning vulnerabilities via privacy attacks.}
    \item We identify a simple yet effective pattern, the \emph{confidence pitfall} (Section~\ref{sec:design_motivation}), which distinguishes nodes connected to unlearned edges, supported by both empirical and theoretical evidence. 
    \item We introduce a novel unlearning inversion attack, namely \textbf{TrendAttack} (Section~\ref{sec:trendattack}), which leverages confidence pitfall to accurately identify unlearned edges from black-box GNN outputs.
    \item Comprehensive evaluation (Section~\ref{sec:experiments}) on four real-world datasets shows that our method consistently outperforms state-of-the-art GNN privacy attacks on recovering unlearned edges.
\end{itemize}

\section{Related Works}\label{sec:rel_works}

\textbf{Graph Unlearning.} Graph Unlearning enables the efficient removal of unwanted data’s influence from trained graph ML models~\cite{chen2022graph,said2023survey,fan2025opengu}. This removal process balances model utility, unlearning efficiency, and removal guarantees, following two major lines of research: retrain-based unlearning and approximate unlearning. 
Retrain-based unlearning partitions the original training graph into disjoint subgraphs, training independent submodels on them, enabling unlearning through retraining on a smaller subset of the data. Specifically, GraphEraser~\cite{chen2022graph} pioneered the first retraining-based unlearning framework for GNNs, utilizing balanced clustering methods for subgraph partitioning and ensembling submodels in prediction with trainable fusion weights to enhance model utility. 
Subsequently, many studies~\cite{wang2023inductive,li2023ultrare,zhang2024graph,li2025community} have made significant contributions to improving the Pareto front of utility and efficiency in these methods, employing techniques such as data condensation~\cite{li2025community} and enhanced clustering~\cite{li2023ultrare,zhang2024graph}. 
Approximate unlearning efficiently updates model parameters to remove unwanted data. Certified graph unlearning~\cite{chien2022certified} provides an important early exploration of approximate unlearning in SGC~\cite{wu2019simplifying}, with provable unlearning guarantees. GraphGuard~\cite{wu2024graphguard} introduces a comprehensive system to mitigate training data misuse in GNNs, featuring a significant gradient ascent unlearning method as one of its core components. 
GIF~\cite{wu2023gif} presents a novel influence function-based unlearning approach tailored to graph data, considering feature, node, and edge unlearning settings. 
Recent innovative works have further advanced the scalability~\cite{pan2023unlearning,li2024tcgu,yi2025scalable,yang2025erase,zhang2025dynamic} and model utility~\cite{li2024towards,zhang2025node} of approximate unlearning methods. 
In this paper, we explore the privacy vulnerabilities of graph unlearning by proposing a novel membership inference attack tailored to unlearned GNN models, introducing a new defense frontier that graph unlearning should consider from a security perspective.

\noindent\textbf{Membership Inference Attack for GNNs.} Membership Inference Attack (MIA) is a privacy attack targeting ML models, aiming to distinguish whether a specific data point belongs to the training set~\cite{shokri2017membership,hu2022membership}. 
Recently, MIA has been extended to graph learning, where a pioneering work~\cite{duddu2020quantifying} explored the feasibility of membership inference in classical graph embedding models. Subsequently, interest has shifted towards attacking graph neural networks (GNNs), with several impactful and innovative studies revealing GNNs' privacy vulnerabilities in node classification~\cite{he2021node,he2021stealing,olatunji2021membership,zhang2022inference} and graph classification tasks~\cite{wu2021adapting}, covering cover node-level~\cite{he2021node,olatunji2021membership}, link-level~\cite{he2021stealing}, and graph-level~\cite{wu2021adapting,zhang2022inference} inference risks. Building on this, GroupAttack~\cite{zhang2023demystifying} presents a compelling advancement in link-stealing attacks~\cite{he2021stealing} on GNNs, theoretically demonstrating that different edge groups exhibit varying risk levels and require distinct attack thresholds, while a label-only attack has been proposed to target node-level privacy vulnerabilities~\cite{conti2022label} with a stricter setting. Another significant line of research involves graph model inversion attacks, which aim to reconstruct the graph structure using model gradients from white-box models~\cite{zhang2021graphmi} or approximated gradients from black-box models~\cite{zhang2022model}. 

Despite the impressive contributions of previous MIA studies in graph ML models, existing approaches overlook GNNs containing unlearned sensitive knowledge and do not focus on recovering such knowledge from unlearned GNN models. 
Additional related works are in \AppendixName~\ref{sec:more_rel_works}.

\section{Preliminaries}\label{sec:preliminaries}

In this section, we present the notations used in this paper and give preliminaries on node classification and graph unlearning. 

\noindent\textbf{Notations}. In this paper, 
bold uppercase letters (e.g., $\Xb$) denote matrices, bold lowercase letters (e.g., $\xb$) denote column vectors, and normal letters (e.g., $x$) indicate scalars. 
We use $\mathcal{A} \setminus \mathcal{B}:=\{x:x\in\mathcal{A}, x\notin\mathcal{B}\}$ to denote the set difference between sets $\mathcal{A}$ and $\mathcal{B}$. 
Let $\mathcal{G} = (\mathcal{V}, \mathcal{E})$ denote a graph, where $\mathcal{V} = \{v_1, \cdots, v_n\}$ is the node set and $\mathcal{E} \subseteq \mathcal{V} \times \mathcal{V}$ is the edge set. The node feature of node $v_i \in \Vcal$ is denoted by $\xb_i \in \R^d$, and the feature matrix for all nodes is denoted by $\Xb = [\xb_1, \cdots, \xb_n]^\top \in \R^{n \times d}$. The adjacency matrix $\mathbf{A} \in \{0, 1\}^{n \times n}$ encodes the edge set, where $A_{i,j} = 1$ if $(v_i, v_j) \in \mathcal{E}$, and $\mathbf{A}_{i,j} = 0$ otherwise. Specifically, $\mathbf{D}\in\R^{n\times n}$ is the degree matrix, where the diagonal elements $D_{i,i} = \sum_{j=1}^nA_{i,j}$. We use $\Ncal(v)$ to denote the neighborhood of node $v_i \in \Vcal$, and use $\wh{\Ncal}(v)$ to denote a subset of $\Ncal(v)$. Specifically, $\Ncal^{(k)}(v_i)$ represents all nodes in $v_i$'s $k$-hop neighborhood. 
A full list of notations is in \AppendixName~\ref{sec:notations_append}.

\noindent\textbf{Semi-supervised Node Classification}. We focus on a semi-superv\\-ised node classification task in a transductive setting, which is common in real-world applications~\cite{kipf2017semi,wu2020comprehensive}. In this setting, the training graph $\mathcal{G}$ includes a small subset of labeled nodes $\mathcal{V}_L = \{u_1, \cdots, u_{|\mathcal{V}_L|}\} \subseteq \mathcal{V}$, where each node is annotated with a label $y\in\mathcal{Y}$. The remaining nodes are unlabeled and belong to the subset $\mathcal{V}_U$, where $\mathcal{V}_U \cap \mathcal{V}_L = \emptyset$. The test set $\mathcal{V}_T$ is a subset of the unlabeled nodes, represented as $\mathcal{V}_T \subseteq \mathcal{V}_U$. We denote the GNN output for a specific target node $v\in\mathcal{V}$ as $f_{\mathcal{G}}(v;\thetab)$, where $\mathcal{G}$ is the graph used for neighbor aggregation and $\theta$ is the model parameters. 

\noindent\textbf{Graph Unlearning}. Graph unlearning aims to remove the impact of some undesirable training data from trained GNN models under limited computational overhead~\cite{chen2022graph,wu2023gif}. Specifically, consider the original training graph $\mathcal{G}_{\mathrm{orig}} := (\mathcal{V}_{\mathrm{orig}}, \mathcal{E}_{\mathrm{orig}})$ including both desirable data and undesirable data. Normally, the parameters $\thetab_{\mathrm{orig}}$ of the GNN model trained on the original graph is given as:
\begin{align}\label{eq:param_orig}
    \thetab_{\mathrm{orig}} := \mathop{\arg\min}_{\thetab} \sum_{v \in \mathcal{V}_L} \mathcal{L}(f_{\mathcal{G}_{\mathrm{orig}}}(v; \thetab_{}), y_v),
\end{align}
where $\mathcal{V}_L\subseteq\mathcal{V}_{\mathrm{orig}}$ is the set of labeled nodes, and $\mathcal{L}$ is a loss function for node classification (e.g., cross-entropy~\cite{zhang2018generalized}). 

Let the undesirable knowledge be a subgraph $\Delta\mathcal{G} := (\Delta\mathcal{V}, \Delta\mathcal{E})$ of the original graph, where $\Delta\mathcal{V} \subseteq \mathcal{V}_{\mathrm{orig}}$ and $\Delta\mathcal{E} \subseteq \mathcal{E}_{\mathrm{orig}}$. The unlearned graph is defined as $\mathcal{G}_{\mathrm{un}} := (\mathcal{V} \setminus \Delta \mathcal{V}, \mathcal{E} \setminus \Delta \mathcal{E})$, which excludes the undesirable knowledge. The goal of graph unlearning is to obtain parameters $\thetab_{\mathrm{un}}$ using an efficient algorithm $\textsc{Unlearn}$ (e.g., gradient ascent~\cite{wu2024graphguard,zhang2025catastrophic}, or influence function computation~\cite{wu2023gif,wu2023certified}), such that $\thetab_{\mathrm{un}}$ closely approximates the retrained parameters $\thetab_{\mathrm{re}}$ from the cleaned graph $\mathcal{G}_{\mathrm{un}}$, while being significantly more efficient than retraining from scratch, i.e.,
\begin{align}\label{eq:param_unlearned} 
    \thetab_{\mathrm{un}} := & ~ \textsc{Unlearn}(f, \thetab_{\mathrm{orig}}, \mathcal{G}_{\mathrm{orig}}, \Delta\mathcal{G}) \notag \\
    \approx  & ~ \mathop{\arg\min}_{\thetab} \sum_{v \in \mathcal{V}_L \setminus \Delta\mathcal{V}} \mathcal{L}(f_{\mathcal{G}_{\mathrm{orig}} \setminus \Delta\mathcal{G}}(v; \thetab), y_v) ,
\end{align}
where the right optimization problem denotes retrain from scratch on the unlearned graph. We focus on the \textbf{edge unlearning} setting~\cite{wu2023certified}, where $\Delta \mathcal{V} = \emptyset$, i.e., only edges are removed. This setting captures practical scenarios such as users requesting the removal of private friendship links from social media or the deletion of sensitive purchase records from recommender systems.

\section{Problem Formulation}\label{sec:problem}

In this section, we begin by describing the threat model associated with the link-level graph unlearning inversion attack, and then present a formal problem definition. 

\subsection{Threat Model}\label{sec:threat_model}

\textbf{Attacker's Goal.} The adversary aims to recover links in the original training graph $\Gcal_{\mathrm{orig}}$, i.e., $\Ecal_{\mathrm{orig}}$. It includes both the unlearned edges $\Delta \Ecal$ and the remaining membership edges $\Ecal_{\mathrm{orig}} \setminus \Delta\Ecal$. As both types of edges can reveal private user information, with unlearned edges typically being more sensitive, the attacker’s goal is to accurately infer both. Specifically, given any pair of nodes $v_i, v_j \in \Vcal$, the attacker aims to determine whether the edge $(v_i, v_j)$ existed in $\Ecal_{\mathrm{orig}}$, i.e., whether $(v_i, v_j) \in \Ecal_{\mathrm{orig}}$. We leave the study of node-level and feature-level unlearning inversion as future work.

\noindent\textbf{Attacker's Knowledge and Capability.} We consider a black-box setting, motivated by the widespread deployment of machine learning models as a service via APIs~\cite{shokri2017membership,truex2019demystifying}. The attacker can query the unlearned GNN to obtain output probabilities $f_{\Gcal_{\mathrm{un}}}(\cdot; \thetab_{\mathrm{un}})$ for target nodes in $\Vcal$. Additionally, the attacker has partial access to the unlearned graph $\Gcal_{\mathrm{un}}$. For any pair of target nodes $v_i, v_j \in \Vcal$, the attacker has access to the following information: 
\textbf{(i)} The model output probabilities $f_{\Gcal_{\mathrm{un}}}(v_i; \thetab_{\mathrm{un}})$ and $f_{\Gcal_{\mathrm{un}}}(v_j; \thetab_{\mathrm{un}})$;
\textbf{(ii)} The input features $\xb_i$ and $\xb_j$; and
\textbf{(iii)} A subset of the $k$-hop neighborhood of $v_i$ and $v_j$, i.e., $\wh{\Ncal}^{(k)}(v_i)$ and $\wh{\Ncal}^{(k)}(v_j)$. Furthermore, similar to many previous privacy attacks on GNNs~\cite{he2021stealing,olatunji2021membership,dai2023unified} that use a shadow dataset to train a surrogate model, we also assume \textbf{(iv)} the attacker has access to a shadow dataset $\Gcal^{\mathrm{sha}}$ with a distribution similar to the training graph $\Gcal_{\mathrm{orig}}$ before unlearning. 

\noindent{\bf Discussion.} 
Although our attack setting extends the typical probab\\-ility-only black-box attack in (i), this level of attacker knowledge is realistic in practice, since (ii) can be obtained from users’ public profiles on social media platforms, and (iii) can be retrieved by querying public friend lists or social connections.

We note that such shadow datasets in (iv) are often accessible in practice. For example, in social network attacks, many public datasets are available~\cite{dai2020ginger,aliapoulios2021large,nielsen2022mumin}, enabling us to train an attack model on one platform and apply it to another (e.g., from X (Twitter) to Facebook), or transfer it across regions (e.g., Facebook networks in the US to those in the UK). 
In FinTech scenarios (e.g., GNN-based APIs for credit recommendation or fraud detection), attackers could build shadow datasets from public transaction networks (e.g., blockchain data)~\cite{weber2019anti,pareja2020evolvegcn,chai2023towards}, train TrendAttack models on surrogate victim models, and transfer them to real victim models to infer sensitive transactions from API outputs. For instance, the Bitcoin transaction network can serve as a shadow dataset, allowing the attack model to be transferred to other De-Fi ecosystems such as Ethereum or newly emerging platforms. 


\subsection{Graph Unlearning Inversion}\label{sec:gun_inv_formal}  
With the threat model presented in Section~\ref{sec:threat_model}, we now formalize the graph unlearning inversion attack under concrete conditions and attack objectives. We begin by defining the query set and the adversary's partial knowledge. 

\begin{definition}[Query Set]\label{dfn:query_set}
    The query set is a set of $Q$ node pairs of the adversary's interest, defined as $\Qcal := \{(v_{i_k}, v_{j_k}): v_{i_k}, v_{j_k} \in \Vcal\}_{k=1}^{Q}$, where $|\Qcal| = Q$.
\end{definition}

\begin{definition}[Partial Knowledge of Query Set]\label{dfn:partial_k_of_q}
    Let $f_{\Gcal_{\mathrm{un}}}(\cdot; \thetab_{\mathrm{un}})$ be an unlearned GNN model where the parameters $\thetab_{\mathrm{un}}$ are obtained from Eq.~\eqref{eq:param_unlearned}, and $\Delta \Ecal$ denotes the unlearned edge set. For a given query set $\Qcal$, the adversary's partial knowledge on the unlearned graph $\Gcal_{\mathrm{un}}$ is defined as a tuple $\Kcal_\Qcal := (\Pcal_\Qcal, \Fcal_\Qcal)$, where: (i) Probability set $\Pcal_\Qcal := \{f_{\Gcal_{\mathrm{un}}}(v_i; \thetab_{\mathrm{un}}): v_i \in \wh{\Ncal}^{(k)}(v_j),   v_j \in \Qcal\}$ represents the black-box model output probabilities for nodes in the $k$-hop neighborhood of the query nodes; and (ii) Feature set $\Fcal_\Qcal := \{\xb_i : v_i \in \Qcal\}$ contains the input features of nodes in the query pairs.
\end{definition}

In this setting, the attacker aims to find out whether node pairs in the query set $\Qcal$ (Definition~\ref{dfn:query_set}) were connected in the original graph $\Gcal_{\mathrm{orig}}$ before edge removal and GNN unlearning. Under the partial knowledge setting in Definition~\ref{dfn:partial_k_of_q}, the adversary has only black-box access to the unlearned model $f(\cdot;\thetab_{\mathrm{un}})$ on the clean graph $\Gcal_{\mathrm{un}}$ (i.e., the graph without the unlearned edges), which contains minimal information about removed edges, making our attack highly non-trivial. Compared with the strictest MIA setting (i.e., where only black-box model outputs are available)~\cite{olatunji2021membership}, the attacker also knows the IDs of part of neighboring nodes in the query set $\Qcal$ for black-box model queries, and has access to the input features of the nodes. Note that such assumptions are common in membership inference attacks, e.g.,  StealLink~\cite{he2021stealing} also assumes access to node features and a partial view of the attack graph.

With the above definitions, the graph unlearning inversion problem is then formalized as:
\begin{definition}[Graph Unlearning Inversion Problem]\label{dfn:gun_inv}  
    Let $\Qcal$ be the adversary’s node pairs of interest (Definition~\ref{dfn:query_set}), with partial knowledge $\Kcal_\Qcal$ on the unlearned graph $\Gcal_{\mathrm{un}}$ (Definition~\ref{dfn:partial_k_of_q}). Suppose the adversary also has access to a shadow graph $\Gcal^{\mathrm{sha}}$ drawn from a distribution similar to $\Gcal_{\mathrm{orig}}$. The goal of the graph unlearning inversion attack is to predict, for each $(v_i, v_j) \in \Qcal$, whether $(v_i, v_j) \in \Ecal_{\mathrm{orig}}$ (label $1$) or $(v_i, v_j) \notin \Ecal_{\mathrm{orig}}$ (label $0$).
\end{definition}

More specifically, the query set $\Qcal$ can be divided into three disjoint subsets: \textbf{(i)} $\Qcal^+_{\mathrm{un}} := \{(v_i, v_j): (v_i, v_j) \in \Delta \Ecal\}$, representing the positive unlearned edges; \textbf{(ii)} $\Qcal^+_{\mathrm{mem}} := \{(v_i, v_j): (v_i, v_j) \in \Ecal_{\mathrm{orig}} \setminus \Delta \Ecal\}$, representing the remaining membership edges in the original graph; and \textbf{(iii)} $\Qcal^- := \{(v_i, v_j): (v_i, v_j) \notin \Ecal_{\mathrm{orig}}\}$, denoting non-member (negative) pairs. The goal of this work is to distinguish both $\Qcal^+_{\mathrm{un}}$ and $\Qcal^+_{\mathrm{mem}}$ from $\Qcal^-$, enabling accurate membership inference on both positive subsets. This highlights a key difference between our work and prior MIA methods for GNNs~\cite{he2021node,he2021stealing,zhang2023demystifying}, which focus solely on distinguishing $\Qcal^+_{\mathrm{mem}}$ from $\Qcal^-$ and may fall short of inferring $\Qcal^+_{\mathrm{un}}$. Moreover, our work also differs substantially from previous unlearning inversion attacks~\cite{bertran2024reconstruction,hu2024learn}. A summary is provided in the Remark~\ref{rmk:strictest_setting} below, and detailed comparisons on attack settings can be found in \AppendixName~\ref{sec:compare_prev_inv_atk}.

\begin{remark}[Strictness of Attack Setting]\label{rmk:strictest_setting}
    Despite using a shadow dataset $\Gcal^{\mathrm{sha}}$, our attack still works under one of the strictest settings in the unlearning inversion literature. Unlike previous works~\cite{bertran2024reconstruction,hu2024learn} that require white-box or black-box access to both the original model $f(\cdot;\thetab_{\mathrm{orig}})$ and the unlearned model $f(\cdot;\thetab_{\mathrm{un}})$,
    we only require black-box access to the unlearned model $f(\cdot;\thetab_{\mathrm{orig}})$.
\end{remark}

\begin{remark}[Difference to GNN MIAs]\label{rmk:diff_mia}
    Existing GNN MIA methods~\cite{he2021node,he2021stealing,zhang2023demystifying} aim to tell whether $(v_i,v_j) \in \Ecal_{\mathrm{orig}}$ by using the black-box output of the original GNN model $f(\cdot;\thetab_{\mathrm{orig}})$, which may memorize such edges. In contrast, our attack in Definition~\ref{dfn:gun_inv} works in a harder setting: we infer whether $(v_i,v_j) \in \Ecal_{\mathrm{orig}}$ from the unlearned model $f(\cdot;\thetab_{\mathrm{un}})$, with minimal knowledge about the removed edges $\Delta \Ecal \subseteq \Ecal_{\mathrm{orig}}$. Remarkably, even if $(v_i,v_j) \in \Delta \Ecal$ and has been unlearned, we can still recover it accurately.
\end{remark}
\section{Proposed Method}

\begin{figure*}[!t]
  \centering
  \includegraphics[width=0.96\linewidth]{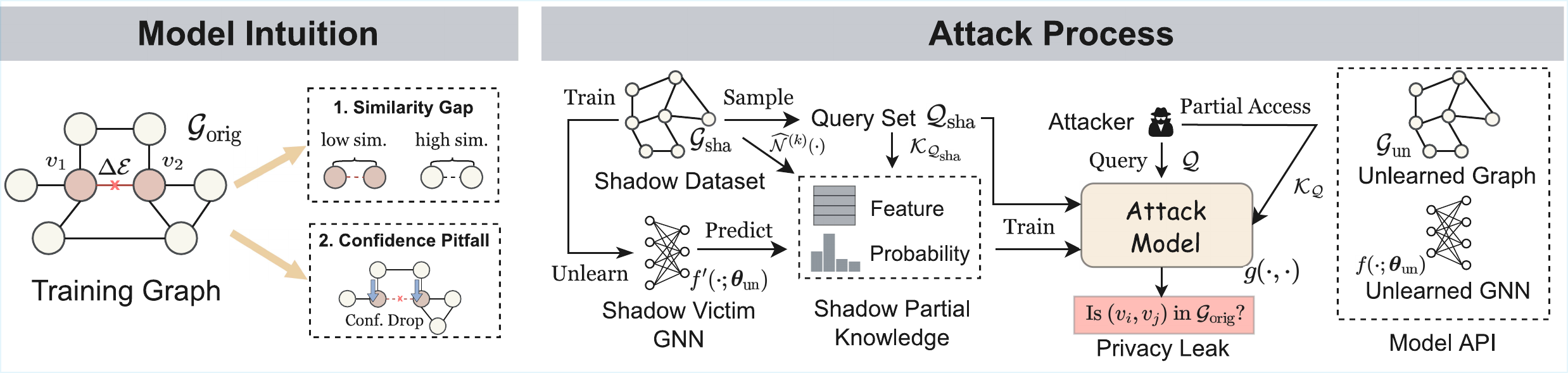}
  \vskip -1.5em
  \caption{\textbf{Illustration of the proposed TrendAttack}. }
  \vskip -0.15in
  \label{fig:model}
\end{figure*}

In this section, we first present the key motivation behind our method, focusing on the adaptive threshold and confidence pitfalls. We then introduce our proposed TrendAttack framework. An illustration of the proposed TrendAttack is in Figure~\ref{fig:model}. 

\subsection{Design Motivation}\label{sec:design_motivation}

As discussed in Section~\ref{sec:gun_inv_formal}, a key challenge that differentiates graph unlearning inversion attacks from traditional MIAs on GNNs is the need to distinguish both unlearned edges in $\Qcal_{\mathrm{un}}^+$ and memorized edges in $\Qcal_{\mathrm{mem}}^+$ from non-edge pairs in $\Qcal^-$. However, these two types of positive edges exhibit different behaviors: unlearned edges typically require a lower similarity threshold to be separated from $\Qcal^-$, while memorized edges require a higher threshold. This raises the question of how to adaptively decide which threshold to apply. 

Our first claim confirms the necessity of different thresholds, and our second claim shows that confidence trends provide a useful signal for classifying whether an edge belongs to $\Qcal_{\mathrm{un}}^+$ or $\Qcal_{\mathrm{mem}}^+$, thereby guiding the choice of threshold.

\noindent\textbf{Probability Similarity Gap.} 
Existing MIA methods typically decide whether a link exists between $v_i$ and $v_j$ by measuring the similarity of their input features $\mathrm{sim}(\xb_i, \xb_j)$~\cite{he2021stealing} and/or the similarity of their output probabilities $\mathrm{sim}(\pb_i, \pb_j)$~\cite{he2021stealing,olatunji2021membership}, where $\pb_i = f_{\Gcal_{\mathrm{un}}}(v_i; \thetab_{\mathrm{un}})$ denotes the GNN output. While feature similarity $\mathrm{sim}(\xb_i, \xb_j)$ is not affected by unlearning and can still be applied in our setting, it does not fully capture graph structure information. Probability similarity $\mathrm{sim}(\pb_i, \pb_j)$, on the other hand, encodes both node features and the structural signals learned by the GNN, and thus plays a central role in distinguishing edges from non-edges.

However, unlearning would affect the similarity of probability. When an edge $(v_i, v_j)$ belongs to $\Delta\Ecal$ and is removed from the graph, its structural contribution to the model is explicitly erased during unlearning. As a result, the similarity between $\pb_i$ and $\pb_j$ weakens compared to edges that remain memorized by the GNN in $\Qcal_{\mathrm{mem}}^+$. 
These considerations lead us to the following claim, which is empirically verified by our preliminary study in \AppendixName~\ref{sec:prelim_exp}:

\begin{claim}[Probability Similarity Gap]\label{clm:prob_sim_gap}
    The average similarity between predicted probabilities across the three query subsets follows the order:
    \begin{align*}
        \mathrm{ProbSim}(\Qcal^-) < \mathrm{ProbSim}(\Qcal^+_{\mathrm{un}}) < \mathrm{ProbSim}(\Qcal^+_{\mathrm{mem}}),
    \end{align*}
    where $\mathrm{ProbSim}(\Qcal_0):= |\Qcal_0|^{-1} \sum_{(v_i, v_j) \in \Qcal_0} \mathrm{sim}(\pb_i, \pb_j)$ denotes the average probability similarity over a query subset $\Qcal_0$.
\end{claim}

The intuition for this claim is that, memorized edges $\Qcal^+_{\mathrm{mem}}$ tend to show the highest probability similarity because the model has fully maintained their structural signals. After unlearning, the edges in $\Qcal^+_{\mathrm{un}}$ lose part of this signal, so their similarity becomes lower than memorized edges. However, practical unlearning methods are often approximate~\cite{chien2022certified,wu2023gif}, which means that a small residual effect of the removed edges may still remain in the model, preventing their similarity from dropping to the level of non-existent edges $\Qcal^-$. In contrast, non-edges naturally exhibit the lowest similarity because they lack any structural signal.  

This gap indicates that, to effectively perform the privacy attack with model output probabilities, one must first distinguish between unlearned and memorized edges, and then apply an adaptive similarity threshold when predicting edges versus non-edges. 

\noindent\textbf{Confidence Pitfall.} Our analysis of the probability similarity gap (Claim~\ref{clm:prob_sim_gap}) suggests that unlearned edges $\Qcal^+_{\mathrm{un}}$ and remaining memorized edges $\Qcal^+_{\mathrm{mem}}$ require different probability-similarity thresholds for accurate prediction. The remaining challenge, however, is that the attacker cannot directly tell whether a given query $(v_i, v_j)$ belongs to $\Qcal^+_{\mathrm{un}}$ or $\Qcal^+_{\mathrm{mem}}$, which makes it difficult to decide which similarity threshold to choose for the query $(v_i,v_j)$. 




To address this, we analyze how removing a specific edge $(v_i, v_j)$ affects model outputs, using the widely used analytical framework of influence functions~\cite{koh2017understanding,wu2023gif,wu2023certified}. This analysis separates two effects: (i) the immediate impact of dropping the edge from the training graph $\Gcal_{\mathrm{orig}}$, and (ii) the subsequent adjustment of model parameters by the unlearning procedure. We focus on a linear GCN model, which, despite its simplicity, captures the core behavior of many GNN architectures by choosing a different propagation matrix $\Cb$ (see Remark~\ref{rmk:lgn_universality} in \AppendixName~\ref{sec:proof}).

\begin{theorem}[Single‐Edge to Single‐Output Influence, Informal]\label{thm:influence_one_edge_one_node}
    Let $f(\Cb, \Xb; \wb^{\star}) := \Cb \Xb \wb^\star$ be a linear GCN with propagation matrix $~\Cb\in\R^{n\times n}$ and parameters $\wb^{\star}$ obtained by least‐squares on labels $\yb\in\R^d$. The influence of an undirected edge $(v_i, v_j)\in\Delta\Ecal$ on the model output $\pb_k:=f(\Cb, \Xb; \wb^{\star})_k$ of node $v_k\in\Vcal$ can be decomposed as follows:
    \begin{align*}
        \mathcal{I}(\pb_k) =&~ \underbrace{\boldsymbol{1}{\{v_k=v_i\}}\cdot (\xb^\top_i \wb^\star) + \boldsymbol{1}{\{v_k=v_j\}}\cdot (\xb^\top_j \wb^\star)}_{\mathrm{edge~influence}} \\
        &~~~~- \underbrace{\langle(\xb_j^\top \wb^\star)\zb_i+(\xb_i^\top \wb^\star) \zb_j, \zb_k\rangle_{\Hb^{-1}}}_{\mathrm{magnitude~weight~influence}} \\
        &~~~~ 
        + \underbrace{\langle (y_j - \zb_j^\top \wb^\star)\xb_i+ (y_i - \zb_i^\top \wb^\star) \xb_j, \zb_k\rangle_{\Hb^{-1}}}_{\mathrm{error~weight~influence}},
    \end{align*}
    where $\Zb := \Cb\Xb$ and  $\zb_l$ is the $l$‐th row of $~\Zb$, $\boldsymbol{1}{\{\cdot\}}$ denotes the indicator function, $\Hb$ is the Hessian of the least squares loss evaluated at $\wb^\star$, and $\langle \cdot, \cdot \rangle_{\Ab}$ denotes the weighted inner product with any positive semi-definite (PSD) matrix $\Ab$.
\end{theorem}
\begin{proof}
    Please see Theorem~\ref{cor:lgcn_out_pernode_infl_spec_one_edge} in \AppendixName~\ref{sec:proof}.
\end{proof}
In the theorem above, the third error weight influence term is governed by the empirical residuals at the optimal weight, $(y_i - \zb_i^\top \wb^\star)$ and $(y_j - \zb_j^\top \wb^\star)$. Under the assumption of a well‐trained (i.e., learnable) graph ML problem, these residuals become negligibly small, making the error weight influence term nearly zero. 

Meanwhile, the first edge influence term activates only when $v_k$ coincides with one of the endpoints $v_i$ or $v_j$, producing a direct and large perturbation. The second magnitude weight influence term, which depends on the inner‐product similarity $\langle(\xb_j^\top \wb^\star)\zb_i + (\xb_i^\top \wb^\star)\zb_j, \zb_k\rangle_{\Hb^{-1}}$, is also relatively larger when $v_k\in\{v_i,v_j\}$ and smaller otherwise. As a result, for $v_k = v_i$ or $v_j$, the sum of a substantial positive edge influence and a significant negative magnitude weight influence yields a dramatic net effect on the model output. Therefore, the influence on the endpoints of unlearned edges may be more significant than on other nodes. With this, we reasonably assume their confidence will drop and make the following claim:
\begin{claim}[Confidence Pitfall]\label{clm:prob_trend}
    The average model confidence of nodes appearing in unlearned edges is lower than that of other nodes, i.e., 
    \begin{align*}
        \mathrm{AvgConf}(\Vcal_{\Qcal^+_{\mathrm{un}}}) < \mathrm{AvgConf}(\Vcal \setminus \Vcal_{\Qcal^+_{\mathrm{un}}}),
    \end{align*}
    where $\Vcal_{\Qcal_0} := \{v_i:(v_i, v_j) \in \Qcal_0 ~\mathrm{ or }~ (v_j, v_i) \in \Qcal_0 \}$ is the set of all nodes involved in query subset $\Qcal_0$, and $\mathrm{AvgConf}(\Vcal_0):= |\Vcal_0|^{-1} \sum_{v_i \in \Vcal_0} \max_{\ell}\pb_{i, \ell}$ is the average model confidence of $\Vcal_0$.
\end{claim}
We also support this claim with a preliminary study that analyzes GNN model confidence at varying distances from unlearned edges on real-world datasets, as detailed in \AppendixName~\ref{sec:prelim_exp}.


\subsection{The Proposed TrendAttack}\label{sec:trendattack}


In the previous subsection, we presented two key claims that guide our model design. First, Claim~\ref{clm:prob_sim_gap} indicates that predicting the existence of a link for a node pair $(v_i,v_j)$ may require different similarity thresholds for the model outputs: a lower threshold for unlearned edges $(v_i,v_j) \in \Qcal^+_{\mathrm{un}}$ and a higher threshold for memorized edges $(v_i,v_j) \in \Qcal^+_{\mathrm{mem}}$ to distinguish them from non-existent edges. Second, Claim~\ref{clm:prob_trend} shows that we can distinguish these two types of edges using the model's confidence trend, which informs which threshold to apply. In this subsection, we instantiate these principles into concrete model components, resulting in a simple, flexible, and effective inversion framework called \textbf{TrendAttack}.

\subsubsection{\textbf{Attack Model}}
It is well established that the edge existence between $v_i$ and $v_j$ can be inferred from the similarity between their features and output probabilities~\cite{olatunji2021membership,he2021stealing}. Since prior GNN privacy attacks have developed a variety of similarity computation frameworks for this task, we adopt a general formulation that computes a scalar similarity score between $v_i$ and $v_j$ as $\phi([\xb_i~||~\pb_i], [\xb_j~||~\pb_j])$, where $\phi$ is an arbitrary similarity function. This formulation is flexible and covers several existing MIA methods. For example, when $\phi([\xb_i~||~\pb_i], [\xb_j~||~\pb_j]) = \hb^\top \cdot \mathrm{MLP}(\pb_i, \pb_j)$, the model recovers MIA-GNN~\cite{olatunji2021membership}. It can also recover the StealLink attack~\cite{he2021stealing} by incorporating their manually defined similarity features (e.g., cosine similarity, JS divergence, etc.) into $\phi$. This flexible structure allows our attack model to incorporate any existing MIA method as the backbone, ensuring it performs at least as well as prior approaches.

In addition, as indicated by Claim~\ref{clm:prob_sim_gap} and Claim~\ref{clm:prob_trend}, a key challenge in graph unlearning inversion is to distinguish nodes associated with unlearned edges from others and to apply an adaptive similarity threshold. Therefore, relying solely on $\phi(\cdot, \cdot)$ may be insufficient to capture this complexity. Thus, to explicitly capture confidence trends and address Claim~\ref{clm:prob_trend}, we define scalar-valued confidence trend features for each node $v_i \in \Vcal$ as follows:
\begin{align}\label{eq:trend_feature_1}
    \tau^{(0)}_i :=  & ~ \max_{\ell}\pb_{i, \ell}, \notag \\ 
    \tau^{(l)}_i :=  & ~ \sum_{v_j \in \wh{\Ncal}^{(1)}(v_i)} \wt{A}_{i, j} \tau^{(l-1)}_i \quad (l = 1,2,\ldots,k), 
\end{align}
where $k$ is the maximum order, $\wt{\Ab} := \Db^{-0.5} \Ab \Db^{-0.5}$ is the normalized adjacency matrix, and $\wt{A}_{i,j}$ is its $(i,j)$-th entry. 

Here, the zeroth-order feature $\tau^{(0)}_i$ is the model confidence of $f_{\Gcal_{\mathrm{un}}}(\cdot; \thetab_{\mathrm{un}})$ for node $v_i$, based on its output probability $\pb_i$. Higher-order features $\tau^{(l)}_i$ are obtained by aggregating lower-order features from the neighborhood of $v_i$. Note that computing $\wt{A}_{i,j}$ may not necessarily require the full adjacency matrix $\Ab$, since the normalization can be done using only the nodes accessible to the attacker.

Next, we define the confidence difference between orders as $\Delta \tau_i^{(k)} := \tau_i^{(k)} - \tau_i^{(k-1)}$ for $k \geq 1$. Based on Claim~\ref{clm:prob_trend}, the sign of these differences (e.g., between zeroth and first order, and first and second order) serves as a useful signal for identifying nodes that are endpoints of unlearned edges. Thus, we define the following binary-valued trend feature for each node $v_i$:
\begin{align}\label{eq:trend_feature_2}
    \wt{\boldsymbol{\tau}}_i^{(k)}
:= 
\mathop{\Big\|}_{l=1}^k 
\left[   \boldsymbol{1}\{\Delta\tau_i^{(l)}<0\}, 
  \boldsymbol{1}\{\Delta\tau_i^{(l)}>0\}   \right],
\end{align}
where $\|$ denotes vector concatenation, $\boldsymbol{1}\{\cdot\}$ denotes the indicator function, and the resulting $k$-th order trend feature $\wt{\boldsymbol{\tau}}_i^{(k)}$ captures whether the confidence difference is positive or negative at each hop of neighbor for node $v_i$.

In practice, we find that using only the first-order trend feature (i.e., $\wt{\boldsymbol{\tau}}_i^{(1)}$) is sufficient to achieve superior attack performance (see Figure~\ref{fig:ablation_trend_order} in \AppendixName~\ref{sec:more_experiments}). This means we only need the black-box model output for the first-hop neighbors of the node pair of interest $(v_i, v_j)$, demonstrating that our method only requires minimal additional knowledge. The final attack model is given as:
\begin{align}\label{eq:attack_model}
    g(v_i, v_j):=~ \sigma(\underbrace{\phi([\xb_i~||~\pb_i], [\xb_j~||~\pb_j])}_{\mathrm{Similarity}} + \underbrace{\hb^\top [\wt{\boldsymbol{\tau}}_i~||~\wt{\boldsymbol{\tau}_j}]}_{\mathrm{Trend}}).
\end{align}
In the equation above, the similarity term estimates membership from similarity in features and probabilities, and the trend term compensates for the similarity gap by distinguishing node types as indicated by Claim~\ref{clm:prob_sim_gap}. The sigmoid function $\sigma(\cdot)$ maps the output to $[0,1]$, with 0 indicating non-member and 1 indicating member. This design allows the attack model to incorporate an adaptive threshold based on node-specific trend features, while leveraging any existing MIA framework with $\phi(\cdot,\cdot)$ as its base.

\subsubsection{\textbf{Shadow Victim Model Training}} To train an attack model that predicts whether $(v_i,v_j) \in \Gcal_{\mathrm{orig}}$, we need ground-truth labels to optimize its parameters. To obtain these labels, we simulate the training and unlearning process of the target GNN by creating a shadow dataset $\Gcal^{\mathrm{sha}}$. We first train a shadow victim model on this dataset and then unlearn a subset of edges, allowing us to generate the necessary ground-truth labels for training the attack model. Below, we provide details of the shadow and attack models.

To construct an attack model $g(v_i, v_j)$ that accurately predicts membership information given partial knowledge of the unlearned graph $\Kcal_\Qcal := (\Pcal_\Qcal, \Fcal_\Qcal)$, we first simulate the victim model’s behavior using a shadow dataset $\Gcal^{\mathrm{sha}}$ and a shadow victim GNN $f'_{\Gcal_{\mathrm{un}}^{\mathrm{sha}}}(\cdot;\thetab_{\mathrm{un}}')$. This model is trained on a node classification task and then unlearns a small subset of edges $\Delta\Gcal^{\mathrm{sha}}$. Specifically, the pre- and post-unlearning parameters are obtained similarly via Eq.~\eqref{eq:param_orig} and Eq.~\eqref{eq:param_unlearned} in Section~\ref{sec:preliminaries}.


This training setup explicitly models unlearning behavior, in contrast to prior MIA approaches that rely solely on the original model $f'_{\Gcal^{\mathrm{sha}}_{\mathrm{orig}}}(\cdot;\thetab_{\mathrm{orig}}')$ and do not account for the effects of unlearning. Using this shadow victim model, we construct a shadow query set $\Qcal_{\mathrm{sha}}$ and its associated partial knowledge $\Kcal_{\Qcal_{\mathrm{sha}}} := (\Pcal_{\Qcal_{\mathrm{sha}}}, \Fcal_{\Qcal_{\mathrm{sha}}})$ from the features, connectivity, and outputs of $f'$ on $\Gcal^{\mathrm{sha}}$, which are then used to train the attack model. 

\subsubsection{\textbf{Attack Model Training}}
We train the attack model $g$ on the shadow dataset $\Gcal^{\mathrm{sha}}$ using the outputs of the shadow victim model $f'$. Given the shadow query set $\Qcal_{\mathrm{sha}}$ with known membership labels, we optimize a link‐prediction loss:
\begin{align}\label{eq:attack_train}
    \Lcal_{\mathrm{attack}}(\Qcal_{\mathrm{sha}}) := \!\!\!\!\!\!\!\sum_{(v_i, v_j)\in \Qcal_{\mathrm{sha}}^+} \!\!\!\!\!\!\log g(v_i, v_j) -\!\!\!\!\!\!\!\sum_{(v_i, v_j)\in \Qcal_{\mathrm{sha}}^-} \!\!\!\!\!\!\log(1 - g(v_i, v_j)).
\end{align}
The pseudo-code for building the shadow victim model and training the attack model $g$ on the shadow dataset $\Gcal^{\mathrm{sha}}$ is provided in Algorithm~\ref{alg:shadow_training} in \AppendixName~\ref{sec:model_details}.

\subsubsection{\textbf{Performing TrendAttack}.}  
After training, we apply the attack model $g$ to the target unlearned graph $\Gcal_{\mathrm{un}}$ by computing, for each query pair $(v_i,v_j)\in\Qcal$, the feature-probability similarity and trend features from the real model output $f_{\Gcal_{\mathrm{un}}}(\cdot;\thetab_{\mathrm{un}})$ and partial graph knowledge $\Kcal$. We then evaluate $g(v_i,v_j)$ on predicting whether $(v_i,v_j)\in\Ecal_{\mathrm{orig}}$. By combining both similarity and trend signals learned on the shadow graph, $g$ effectively supports unlearning inversion on the real unlearned graph. The pseudo-code for the attack process is provided in Algorithm~\ref{alg:attack} in \AppendixName~\ref{sec:model_details}.

\section{Experiments}\label{sec:experiments}

In this section, we describe our experimental setup and present the main empirical results of this work.

\subsection{Experiment Settings}\label{sec:exp_settings}

\textbf{Datasets.} We evaluate our attack method on four standard graph ML benchmark datasets: Cora~\cite{yang2016revisiting}, Citeseer~\cite{yang2016revisiting}, Pubmed~\cite{yang2016revisiting}, and LastFM-Asia~\cite{yang2016revisiting}. To construct the shadow dataset $\Gcal_{\mathrm{sha}}$ and the real attack dataset $\Gcal_{\mathrm{orig}}$, which should share similar distributions, we use METIS to partition the entire training graph into two balanced subgraphs, one for shadow and one for attack. Following the setting in GIF~\cite{wu2023gif}, we use 90\% of the nodes in each subgraph for training and the remaining for testing. All edges between the two subgraphs are removed, and there are no shared nodes, simulating real-world scenarios where shadow and attack datasets are disconnected.

\noindent\textbf{Victim Model.} We adopt a two-layer GCN~\cite{kipf2017semi} as the victim model, trained for node classification. Our GCN implementation follows the standard settings in GIF~\cite{wu2023gif} for consistency and reproducibility. After training the GCN on both the shadow and attack datasets, we perform edge unlearning on 5\% of randomly selected edges using standard graph unlearning methods, including GIF~\cite{wu2023gif}, CEU~\cite{wu2023certified}, and Gradient Ascent (GA)~\cite{wu2024graphguard}. We follow the official settings from each method's paper and codebase to ensure faithful reproduction.

\noindent\textbf{Baselines.} We do not compare with unlearning inversion attacks~\cite{hu2024learn} that require access to pre-unlearning models, which is unrealistic under our setting. To demonstrate that unlearning inversion cannot be solved by naive link prediction, we evaluate a simple GraphSAGE model and a state-of-the-art link prediction method, NCN~\cite{wang2024neural}. For membership inference attacks (MIAs) under the same black-box assumption as ours, we consider three widely used methods: StealLink~\cite{he2021stealing}, MIA-GNN~\cite{olatunji2021membership}, and GroupAttack~\cite{zhang2023demystifying}. We follow their official hyperparameter settings from their respective papers and repositories. All experiments are run five times, and we report the mean and the standard error.

\noindent\textbf{Evaluation Metrics.} We evaluate our attack on the attack dataset using a specific query set $\Qcal$. We randomly select 5\% of the edges as unlearned edges $\Qcal^+_{\mathrm{un}}$ (label 1), and another 5\% of remaining edges as regular member edges $\Qcal^+_{\mathrm{mem}}$ (label 1). We then sample an equal number of non-existent (negative) edges $\Qcal^-$ (label 0) such that $|\Qcal^-| = |\Qcal^+_{\mathrm{un}}| + |\Qcal^+_{\mathrm{mem}}|$. We use AUC as the primary evaluation metric. Since both unlearned edges and regular member edges are important for membership inference, we compute the overall AUC on the full query set $\Qcal = \Qcal^+_{\mathrm{un}} \cup \Qcal^+_{\mathrm{mem}} \cup \Qcal^-$.

More experimental details, including model parameters, baselines, and datasets, are in \AppendixName~\ref{sec:append_exp_settings}.

\subsection{Unlearning Inversion Attack Performance}\label{sec:comp_exp}

\begin{table*}[!t]
\caption{\textbf{Main Comparison Results}. We present the AUC scores for attack methods across different edge groups. The best results are highlighted in \textbf{bold}, while the second-best results are \underline{underlined}.}
\label{tab:main}
\centering
\vskip -1em
\resizebox{\textwidth}{!}{
\begin{tabular}{llccc|ccc|ccc|ccc}
\toprule
\multirow{2}{*}{\textbf{Unlearn method}} & \multirow{2}{*}{\textbf{Attack}}
  & \multicolumn{3}{c}{\textbf{Cora}}
  & \multicolumn{3}{c}{\textbf{Citeseer}}
  & \multicolumn{3}{c}{\textbf{Pubmed}}
  & \multicolumn{3}{c}{\textbf{LastFM-Asia}} \\
\cmidrule(lr){3-5} \cmidrule(lr){6-8} \cmidrule(lr){9-11} \cmidrule(lr){12-14}
 & & Unlearned & Original & All
   & Unlearned & Original & All
   & Unlearned & Original & All
   & Unlearned & Original & All \\
\midrule
\textbf{GIF} & GraphSAGE    & 0.5356 & 0.5484 & 0.5420   & 0.5275 & 0.5216 & 0.5246   & 0.6503 & 0.6457 & 0.6480   & 0.6914 & 0.6853 & 0.6884 \\
             & NCN        & 0.7403 & 0.7405 & 0.7404   & 0.6750 & 0.6872 & 0.6811   & 0.6661 & 0.6718 & 0.6690   & 0.7283 & 0.7273 & 0.7278 \\
             & MIA-GNN    & 0.7547 & 0.7916 & 0.7732   & 0.7802 & 0.8245 & 0.8023   & 0.7028 & 0.7902 & 0.7465   & 0.5955 & 0.5744 & 0.5850 \\
             & StealLink  & 0.7841 & 0.8289 & 0.8065   & 0.7369 &\underline{0.8404} & 0.7887   & 0.8248 & 0.8964 & 0.8606   & \underline{0.8472} & \underline{0.9037} & \underline{0.8755} \\
             & GroupAttack& 0.7982 & 0.8053 & 0.8018   & 0.7771 & 0.7618 & 0.7695   & 0.6497 & 0.6554 & 0.6525   & 0.7858 & 0.7850 & 0.7854 \\
             & \textbf{TrendAttack-MIA}   & \underline{0.8240} & \underline{0.8448} & \underline{0.8344}   & \underline{0.8069} & 0.8078 & \underline{0.8073}   & \underline{0.8950} & \underline{0.9171} & \underline{0.9060}   & 0.7795 & 0.7649 & 0.7722 \\
             & \textbf{TrendAttack-SL} & \textbf{0.8309} & \textbf{0.8527} & \textbf{0.8418}   & \textbf{0.8410} & \textbf{0.8430} & \textbf{0.8420}   & \textbf{0.9524} & \textbf{0.9535} & \textbf{0.9529}   & \textbf{0.9078} & \textbf{0.9134} & \textbf{0.9106} \\
\midrule
\textbf{CEU} & GraphSAGE    & 0.5356 & 0.5484 & 0.5420   & 0.5275 & 0.5216 & 0.5246   & 0.6503 & 0.6457 & 0.6480   & 0.6914 & 0.6853 & 0.6884 \\
             & NCN        & 0.7403 & 0.7405 & 0.7404   & 0.6750 & 0.6872 & 0.6811   & 0.6661 & 0.6718 & 0.6690   & 0.7283 & 0.7273 & 0.7278 \\
             & MIA-GNN    & 0.7458 & 0.7810 & 0.7634   & 0.7718 & 0.8248 & 0.7983   & 0.6626 & 0.6561 & 0.6593   & 0.6004 & 0.5811 & 0.5908 \\
             & StealLink  & 0.7901 & \underline{0.8486} & 0.8193   & 0.7643 & \textbf{0.8450} & \underline{0.8046}   & 0.8467 & 0.9088 & 0.8777   & \underline{0.8416} & \underline{0.9021} & \underline{0.8719} \\
             & GroupAttack& 0.7941 & 0.7976 & 0.7958   & 0.7557 & 0.7458 & 0.7508   & 0.6388 & 0.6430 & 0.6409   & 0.7845 & 0.7817 & 0.7831 \\
             & \textbf{TrendAttack-MIA}   & \underline{0.8194} & 0.8333 & \underline{0.8263}   & \underline{0.7933} & 0.8041 & 0.7987   & \underline{0.8982} & \underline{0.9184} & \underline{0.9083}   & 0.7676 & 0.7576 & 0.7626 \\
             & \textbf{TrendAttack-SL} & \textbf{0.8467} & \textbf{0.8612} & \textbf{0.8539}   & \textbf{0.8514} & \underline{0.8400} & \textbf{0.8457}   & \textbf{0.9550} & \textbf{0.9579} & \textbf{0.9565}   & \textbf{0.9037} & \textbf{0.9088} & \textbf{0.9062} \\
\midrule
\textbf{GA}  & GraphSAGE    & 0.5356 & 0.5484 & 0.5420   & 0.5275 & 0.5216 & 0.5246   & 0.6503 & 0.6457 & 0.6480   & 0.6914 & 0.6853 & 0.6884 \\
             & NCN        & 0.7403 & 0.7405 & 0.7404   & 0.6750 & 0.6872 & 0.6811   & 0.6661 & 0.6718 & 0.6690   & 0.7283 & 0.7273 & 0.7278 \\
             & MIA-GNN    & 0.7676 & 0.8068 & 0.7872   & 0.7798 & 0.8353 & 0.8076   & 0.7242 & 0.8039 & 0.7641   & 0.6200 & 0.6057 & 0.6129 \\
             & StealLink  & 0.7862 & 0.8301 & 0.8082   & 0.7479 & \underline{0.8431} & 0.7955   & 0.8203 & 0.8898 & 0.8550   & \underline{0.8342} & \underline{0.8947} & \underline{0.8644} \\
             & GroupAttack& 0.7945 & 0.8042 & 0.7993   & 0.7662 & 0.7563 & 0.7613   & 0.6458 & 0.6493 & 0.6475   & 0.7746 & 0.7760 & 0.7753 \\
             & \textbf{TrendAttack-MIA}   & \underline{0.8193} & \textbf{0.8397} & \underline{0.8295}   & \underline{0.8080} & 0.8249 & \underline{0.8165}   & \underline{0.8932} & \underline{0.9158} & \underline{0.9045}   & 0.7255 & 0.7099 & 0.7177 \\
             & \textbf{TrendAttack-SL} & \textbf{0.8270} & \underline{0.8382} & \textbf{0.8326}   & \textbf{0.8628} & \textbf{0.8614} & \textbf{0.8621}   & \textbf{0.9531} & \textbf{0.9537} & \textbf{0.9534}   & \textbf{0.9041} & \textbf{0.9119} & \textbf{0.9080} \\
\bottomrule
\end{tabular}
}
\vskip -1em
\end{table*}

In this study, we present a comprehensive comparison of all our baselines mentioned in Section~\ref{sec:exp_settings}, and the results are shown in Table~\ref{tab:main}. Specifically, to demonstrate that our method better captures unlearned edges, we evaluate AUC across three groups: \emph{Unlearned} ($\Qcal^+_{\mathrm{un}} \cup \Qcal^-$), \emph{Original} ($\Qcal^+_{\mathrm{mem}} \cup \Qcal^-$), and \emph{All} (the entire $\Qcal$). A small gap between Unlearned and Original AUCs indicates a better balance in the model's predictions. The full table with standard deviation can be found in \AppendixName~\ref{sec:more_experiments}. 

We consider two variants of our method, TrendAttack-MIA and TrendAttack-SL, which adopt MIA-GNN~\cite{olatunji2021membership} and StealLink~\cite{he2021stealing} as their respective backbone models. From the table, we observe: \textbf{(i)} Compared with their MIA prototypes, both variants of TrendAttack significantly improve the gap between Unlearned and Original AUCs, as well as the overall AUC. This demonstrates the effectiveness of our proposed trend-based attack and unlearning-aware training design; \textbf{(ii)} Overall, our proposed attack, especially TrendAttack-SL, achieves the best performance on most datasets and unlearning methods. The only failure case is on Citeseer + CEU + Original, where the gap to StealLink is small. This does not undermine our contribution, as our primary focus is on the unlearned sets, and this case is a special exception. TrendAttack-MIA achieves the second-best results across many settings, while its relatively lower performance is mainly due to the weak backbone model (MIA-GNN), not the attack framework itself; and \textbf{(iii)} Among all baselines, attack-based methods consistently outperform link prediction methods, showing that link prediction alone cannot effectively solve the unlearning inversion problem. Among the attack baselines, StealLink is the strongest, while MIA-GNN performs poorly as it only uses output probabilities and ignores features.

\subsection{Ablation Studies}

\begin{figure*}[!t]
  \centering
  \begin{subfigure}[b]{0.29\textwidth}
    \centering
    \includegraphics[width=\linewidth]{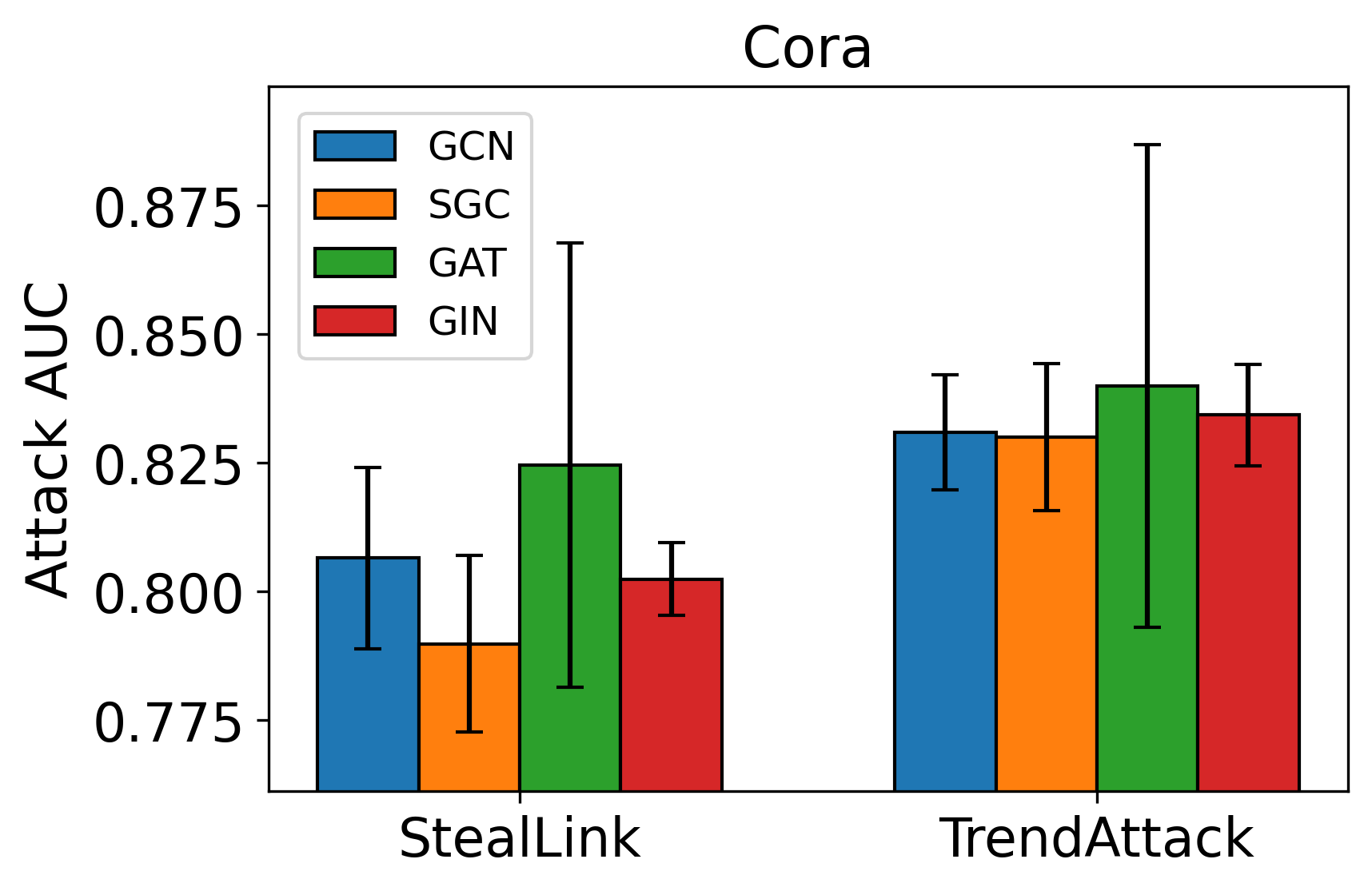}
    \caption{Cora results.}
    \label{fig:ablation_cora}
  \end{subfigure}\hfill
  \begin{subfigure}[b]{0.29\textwidth}
    \centering
    \includegraphics[width=\linewidth]{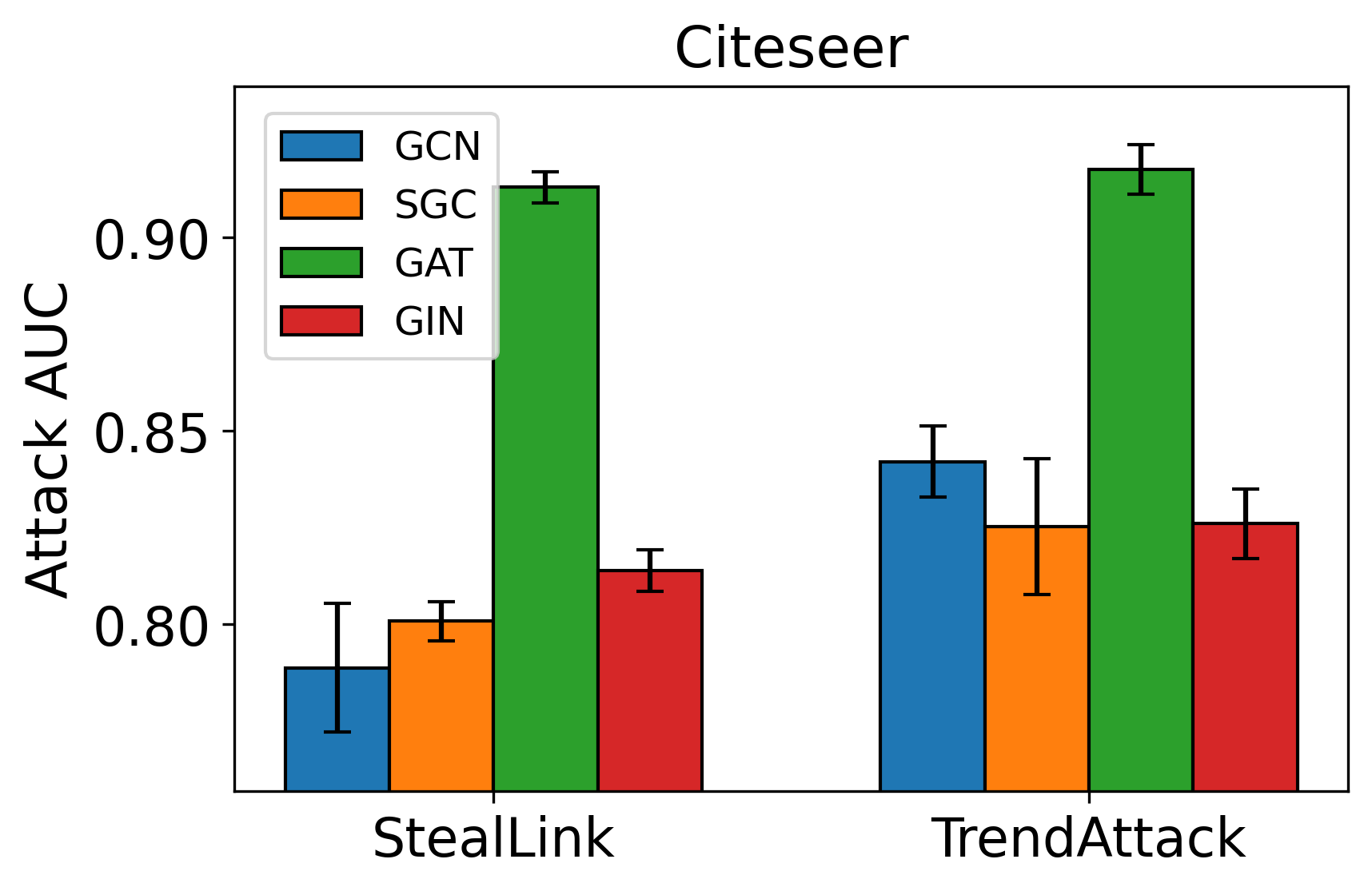}
    \caption{Citeseer results.}
    \label{fig:ablation_citeseer}
  \end{subfigure}\hfill
  \begin{subfigure}[b]{0.29\textwidth}
    \centering
    \includegraphics[width=\linewidth]{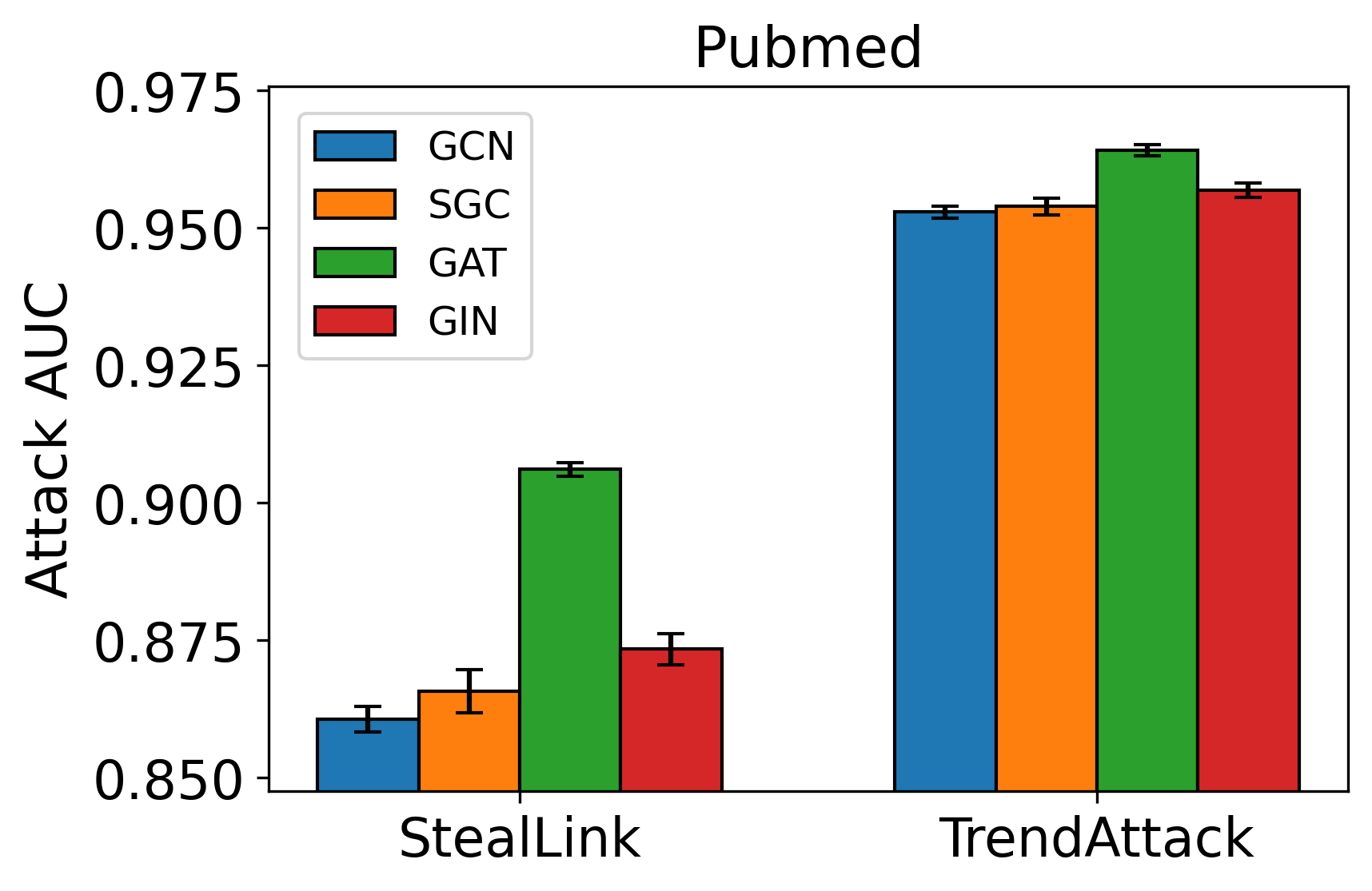}
    \caption{Pubmed results.}
    \label{fig:ablation_pubmed}
  \end{subfigure}
  \vskip -1em
  \caption{\textbf{Ablation study on the impact of victim models}.}
  \label{fig:ablation_model}
\end{figure*}

\textbf{Impact of Victim Models.} From Table~\ref{tab:main}, we observe that the proposed attack remains stable against unlearning methods. In this study, we further investigate whether TrendAttack's performance maintains stability with respect to changes in victim models. Specifically, we fix the unlearning method to GIF and use TrendAttack-SL as our model variant. We compare it against the best-performing baseline, StealLink, with results shown in Figure~\ref{fig:ablation_model}. From the figure, we can find that our proposed model consistently outperforms the baseline, demonstrating stability across different victim models. An interesting observation is that performance significantly improves on GAT compared to other baselines, suggesting that GAT may be more vulnerable to model attacks.

\begin{figure*}[!ht]
  \vskip -1em
  \centering
  \begin{subfigure}[b]{0.29\textwidth}
    \centering
    \includegraphics[width=\linewidth]{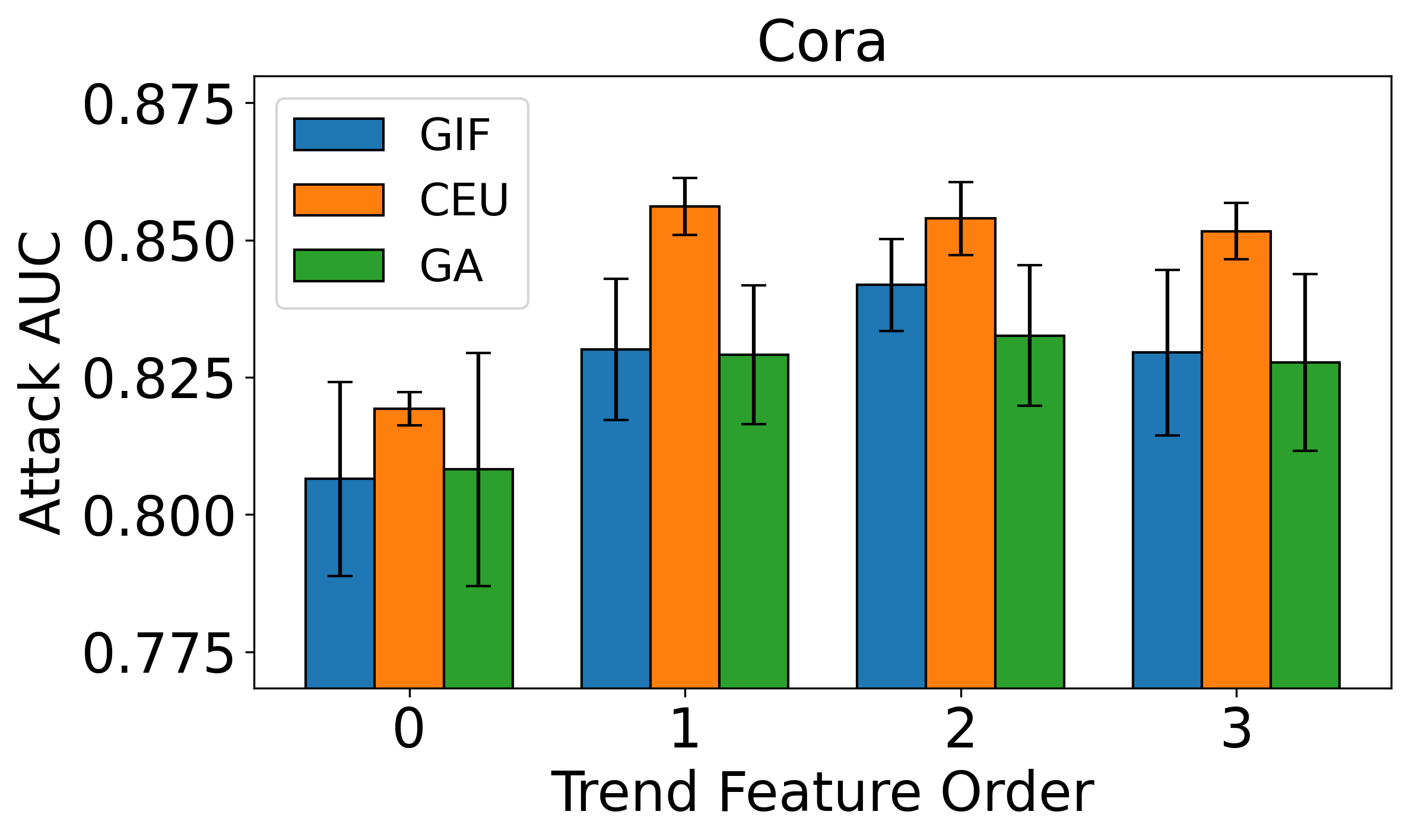}
    \caption{Cora results.}
    \label{fig:trend_order_cora}
  \end{subfigure}\hfill
  \begin{subfigure}[b]{0.29\textwidth}
    \centering
    \includegraphics[width=\linewidth]{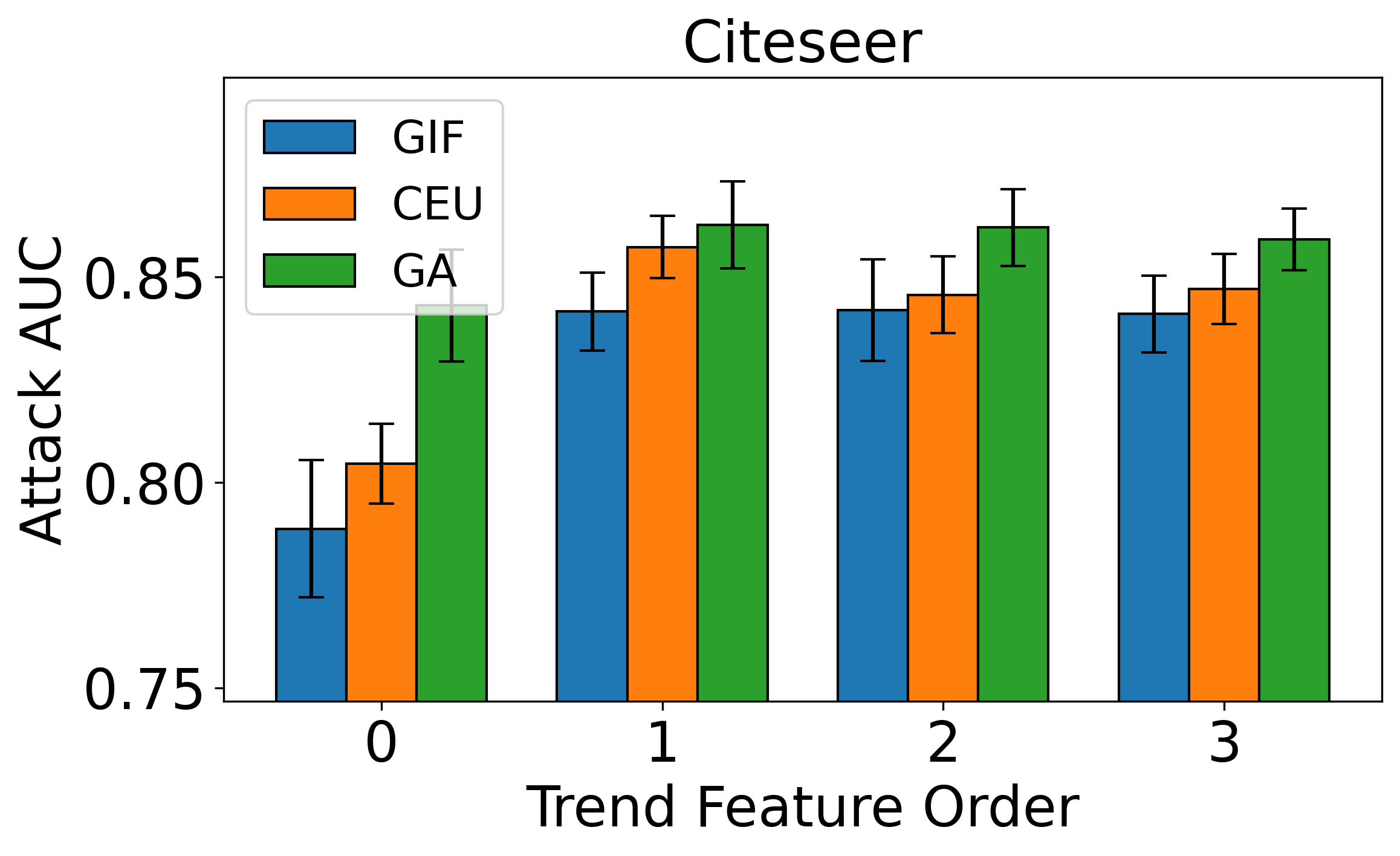}
    \caption{Citeseer results.}
    \label{fig:trend_order_citeseer}
  \end{subfigure}\hfill
  \begin{subfigure}[b]{0.29\textwidth}
    \centering
    \includegraphics[width=\linewidth]{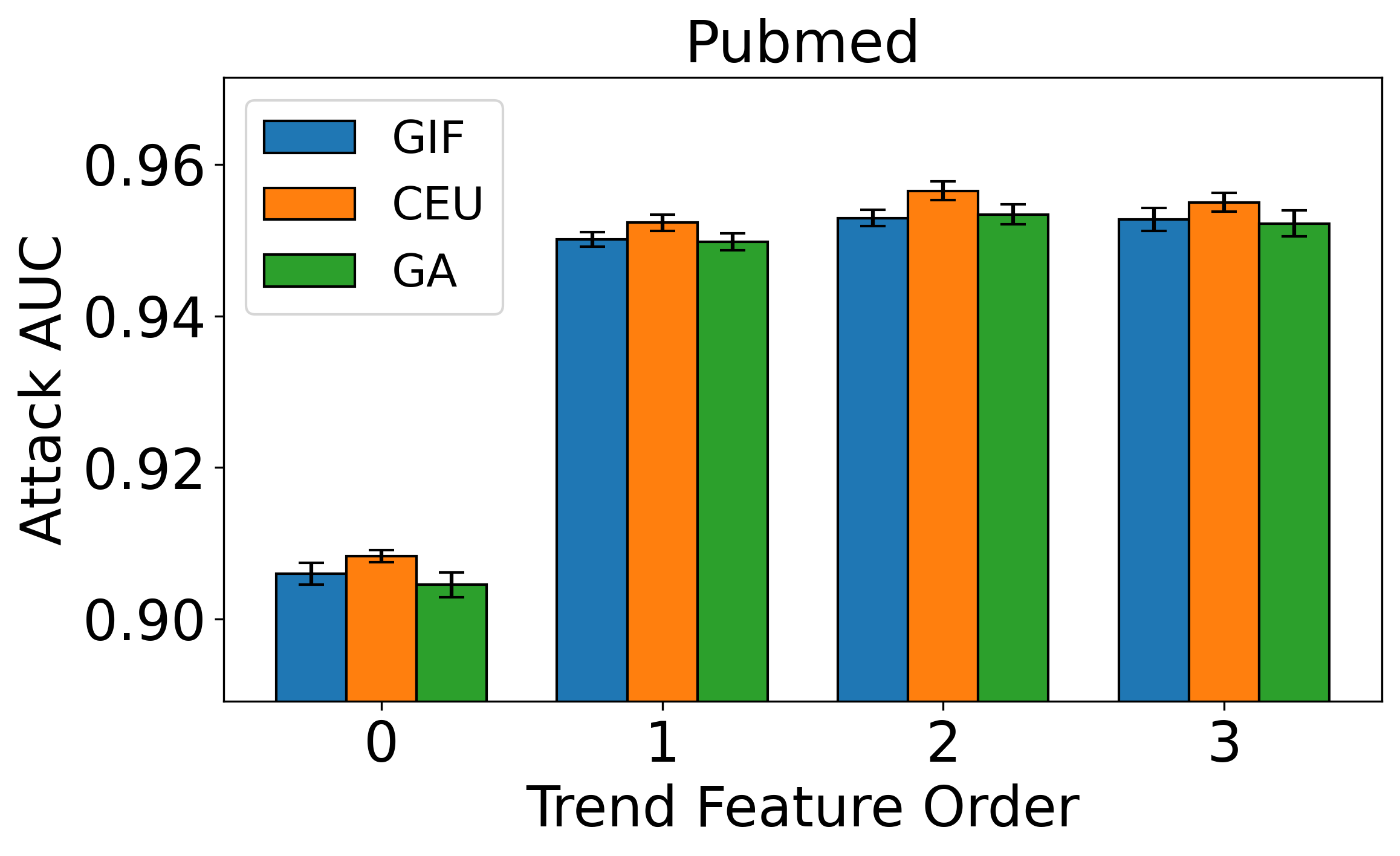}
    \caption{Pubmed results.}
    \label{fig:trend_order_pubmed}
  \end{subfigure}
  \vskip -1em
  \caption{\textbf{Ablation study on the impact of trend feature order}. Overall attack AUC as a function of the trend feature order (0–3) for three unlearn methods across three datasets.}
  \label{fig:ablation_trend_order}
  \vskip -1em
\end{figure*}

\noindent{\bf Impact of Trend Feature Orders.} 
We also study the influence of trend feature orders on the performance of the proposed TrendAttack. Following the experimental setup described in Section~\ref{sec:comp_exp}, we employ the most effective variant, TrendAttack-SL, for this ablation study. Specifically, we vary the trend feature order $k$ (see Algorithm~\ref{alg:attack}) to assess the trade-off between incorporating additional neighborhood information and achieving high attack performance. The results are presented in Figure~\ref{fig:ablation_trend_order}, from which we make the following observations:
\textbf{(i)} Incorporating trend features of orders 1, 2, or 3 significantly improves attack performance compared to the 0-th order (which corresponds to a degenerate form of TrendAttack that reduces to a simple StealLink attack). This highlights the effectiveness of our proposed trend feature design; and \textbf{(ii)} The attack performance remains relatively stable across orders 1 to 3, indicating that the method is robust to the choice of trend order. Notably, lower-order features (e.g., order 1) still have strong performance while requiring less auxiliary neighborhood information, making them more practical in real-world inversion attack scenarios.

\noindent{\bf Impact of Edge Unlearning Ratio.} 
Due to space limitations, more experiments on the impact of edge unlearning ratio can be found in Figure~\ref{fig:ablation_ratio} in \AppendixName~\ref{sec:more_experiments}.

\begin{figure*}[!ht]
    \centering
    \begin{subfigure}{0.29\textwidth}
        \centering
        \includegraphics[width=\linewidth]{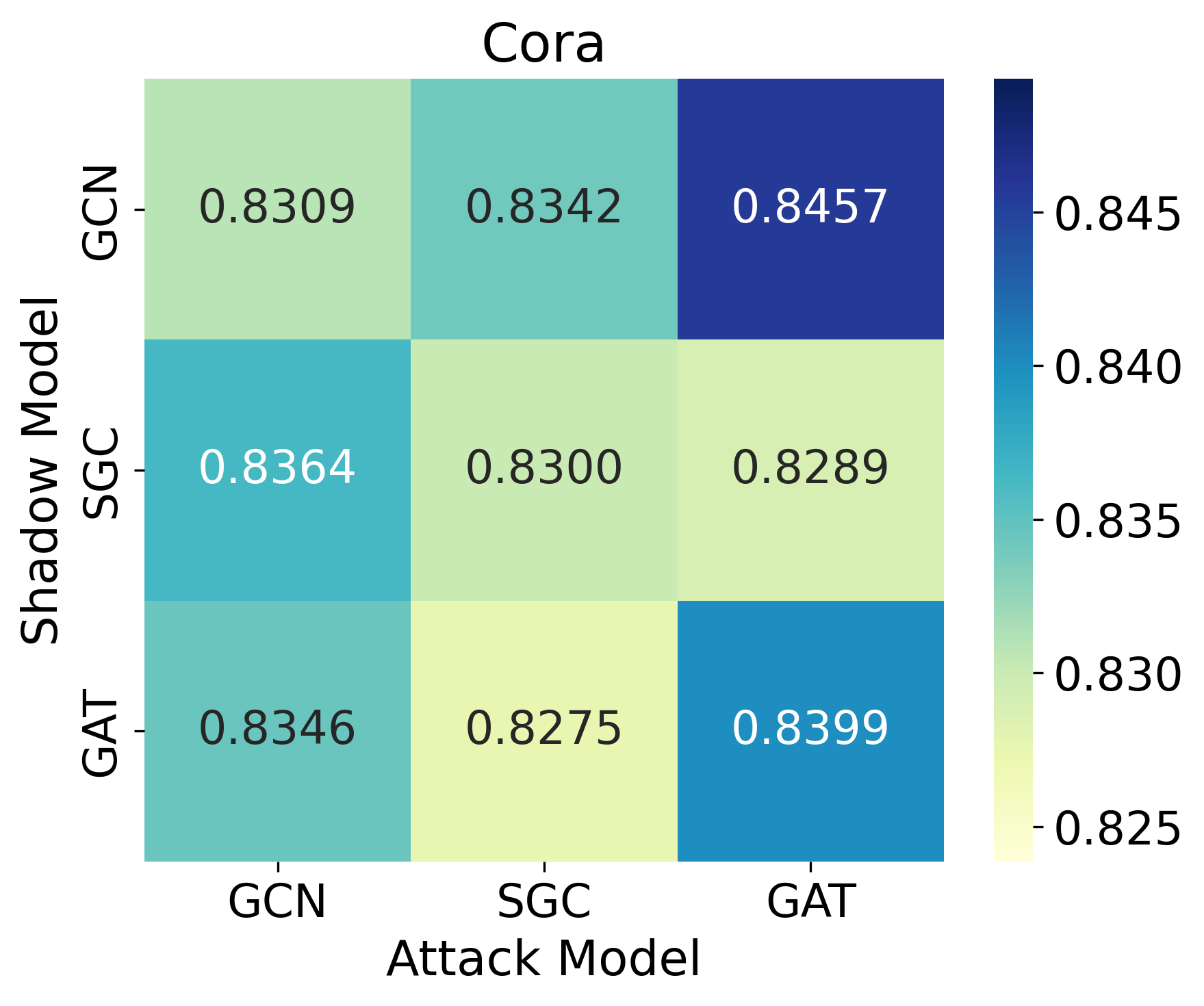}
        \caption{Cora results.}
    \end{subfigure}
    \hfill
    \begin{subfigure}{0.29\textwidth}
        \centering
        \includegraphics[width=\linewidth]{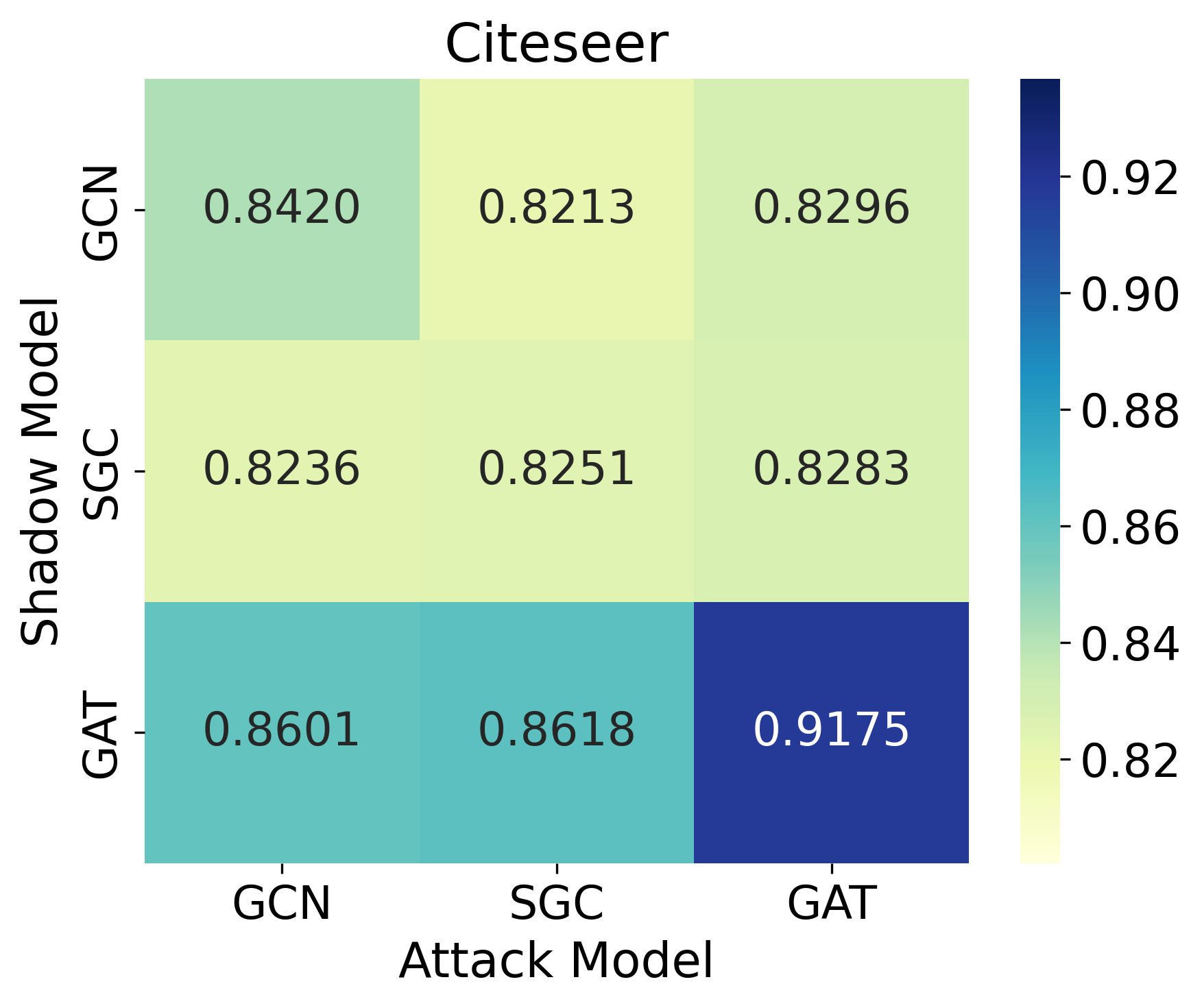}
        \caption{Citeseer results.}
    \end{subfigure}
    \hfill
    \begin{subfigure}{0.29\textwidth}
        \centering
        \includegraphics[width=\linewidth]{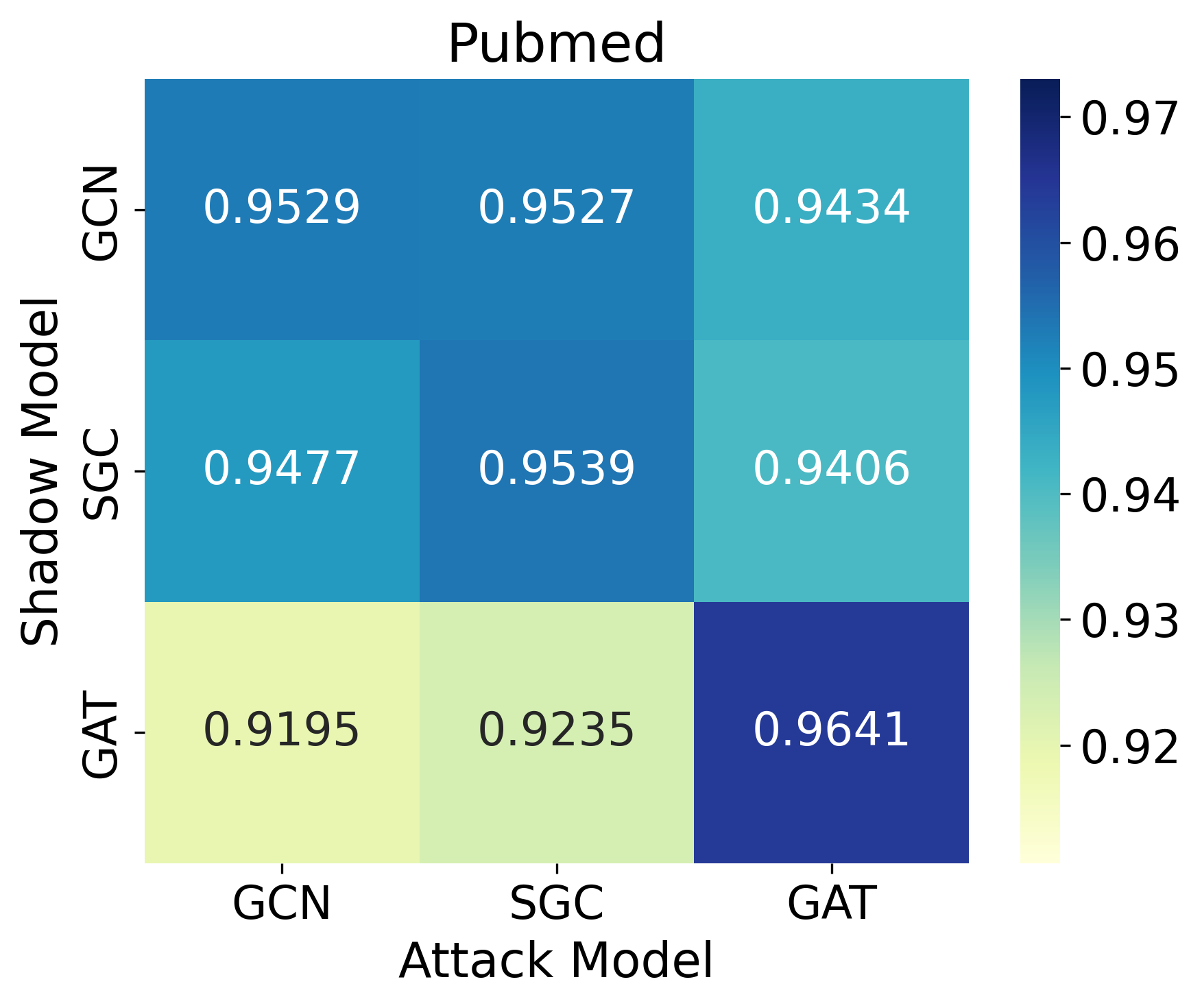}
        \caption{Pubmed results.}
    \end{subfigure}
    \vskip -0.1in
    \caption{Transferability of attack models across different GNN architectures on Cora, Citeseer, and Pubmed. Each heatmap shows attack AUC when the shadow model (rows) and attack model (columns) differ.}
    \label{fig:model_transfer}
    \vskip -0.1in
\end{figure*}

\subsection{Transferability of TrendAttack} \label{sec:transferability}    

{\bf Model Transferability.} 
We study the transferability of attack models across different GNN architectures. Specifically, we evaluate how attacks trained on one shadow victim model generalize to a different target victim model. We use GIF as the victim unlearning method and select TrendAttack-SL, our strongest model variant, for evaluation. Figures~\ref{fig:model_transfer} (a)-(c) show the results on Cora, Citeseer, and Pubmed, respectively.  

In each figure, the diagonal entries correspond to cases where the shadow and target models are the same, while off-diagonal entries represent transferability experiments where the shadow and target models differ. We make two key observations:  
\textbf{(i)} Our attack is robust to differences between shadow and target models. Most transfer results show only minor performance drops. The largest drop occurs for GAT on Citeseer, where the AUC decreases from 0.9175 to 0.8601 when the shadow model is GCN, which remains acceptable.  
\textbf{(ii)} Choosing a strong shadow victim model can slightly improve attack performance. As shown in Figure~\ref{fig:ablation_model}, GAT has the highest attack AUC. Using SGC or GCN to attack GAT results in a performance drop, while using GAT to attack other models can improve results. Therefore, selecting a shadow model with high attack performance can slightly boost transferability.

\noindent{\bf Dataset Transferability.} Due to space limitations, we defer the additional dataset transferability results to Figure~\ref{fig:data_transfer} in \AppendixName~\ref{sec:more_experiments}.
\section{Conclusion}\label{sec:conclusion}

In this work, we address the novel and challenging problem of graph unlearning inversion attacks, aiming to recover unlearned edges using black-box GNN outputs and partial knowledge of the unlearned graph. We identify two key intuitions, probability similarity gap and confidence pitfall, which inspire a simple yet effective attack framework, TrendAttack, revealing potential privacy risks in current graph unlearning methods. 
Several promising directions remain. While our study focuses on link-level inversion, extending TrendAttack to node- or feature-level attacks would be valuable. Moreover, although the probability similarity gap intuition is supported empirically, a formal theoretical justification could further strengthen the understanding of graph unlearning inversion. We hope these future directions can further advance this early study of privacy risks in graph unlearning and provide valuable insights.

\section*{Acknowledgments}
This material is based upon work supported by, or in part by, the Army Research Office (ARO) under grant number W911NF-2110198, the Department of Homeland Security (DHS) under grant number 17STCIN00001-05-00, and Cisco Faculty Research Award. Xiaorui Liu is supported by the National Science Foundation (NSF) Award under grant number IIS-2443182. The findings in this paper do not necessarily reflect the view of the funding agencies.


\ifdefined\isarxiv

\else
\clearpage
\section*{Ethical Consideration}

In this work, we propose a novel privacy attack targeting unlearned GNNs, revealing critical privacy vulnerabilities in existing graph unlearning methods. Our goal is to raise awareness about privacy protection in Web services and to inspire future research on privacy-preserving graph machine learning. While our method introduces a new attack strategy, all experiments are conducted on publicly available benchmark datasets. Given the gap between this pioneering study and real-world deployment scenarios, we do not foresee significant negative societal implications from this work.
\bibliographystyle{ACM-Reference-Format}
\bibliography{ref}

\fi


\appendix
\newpage

\twocolumn[
\begin{center}
    {\LARGE \bf Supplementary Material}
\end{center}
]

\section*{List of Contents}
In this appendix, we provide the following additional information:  
\begin{itemize}[leftmargin=0.05\linewidth]
    \item \textbf{Section~\ref{sec:notations_append}}: Notations.
    \item \textbf{Section~\ref{sec:more_rel_works}}: Additional Related Works.
    \item \textbf{Section~\ref{sec:model_details}}: Model Details.
    \item \textbf{Section~\ref{sec:proof}}: Missing Proofs in Section~\ref{sec:design_motivation}.
    \item \textbf{Section~\ref{sec:prelim_exp}}: Preliminary Study. 
    \item \textbf{Section~\ref{sec:append_exp_settings}}: Experimental Settings.
    \item \textbf{Section~\ref{sec:more_experiments}}: Additional Experiments.
    \item \textbf{Section~\ref{sec:discussion}}: Discussions.
\end{itemize}

\section{Notations} \label{sec:notations_append}

 In this paper, calligraphy uppercase letters (e.g., $\Xcal$) denote sets, bold uppercase letters (e.g., $\Xb$) denote matrices, bold lowercase letters (e.g., $\xb$) denote column vectors, and normal letters (e.g., $x$) indicate scalars.  
Let $\eb_u\in\R^d$ be the column vector with the $u$-th element as $1$ and all others as $0$. 
We use $||$ to denote concatenating two vectors. 
We define the weighted inner product with a PSD matrix $\Hb$ as $\langle \xb, \yb \rangle_\Hb := \xb^\top \Hb \yb$. 
We use $\boldsymbol{1}\{\cdot\}$ to denote the indicator function, which returns 1 if the condition in brackets is satisfied and 0 otherwise. 
We use $\mathcal{A} \setminus \mathcal{B}:=\{x:x\in\mathcal{A}, x\notin\mathcal{B}\}$ to denote the set difference between sets $\mathcal{A}$ and $\mathcal{B}$. 

Let $\mathcal{G} = (\mathcal{V}, \mathcal{E})$ denote a graph, where $\mathcal{V} = \{v_1, \cdots, v_n\}$ is the node set and $\mathcal{E} \subseteq \mathcal{V} \times \mathcal{V}$ is the edge set. The node feature of node $v_i \in \Vcal$ is denoted by $\xb_i \in \R^d$, and the feature matrix for all nodes is denoted by $\Xb = [\xb_1, \cdots, \xb_n]^\top \in \R^{n \times d}$. The adjacency matrix $\mathbf{A} \in \{0, 1\}^{n \times n}$ encodes the edge set, where $A_{i,j} = 1$ if $(v_i, v_j) \in \mathcal{E}$, and $\mathbf{A}_{i,j} = 0$ otherwise. Specifically, $\mathbf{D}\in\R^{n\times n}$ is the degree matrix, where the diagonal elements $D_{i,i} = \sum_{j=1}^nA_{i,j}$. We use $\Ncal(v)$ to denote the neighborhood of node $v_i \in \Vcal$, and use $\wh{\Ncal}(v)$ to denote a subset of $\Ncal(v)$. Specifically, $\Ncal^{(k)}(v_i)$ represents all nodes in $v_i$'s $k$-hop neighborhood, where $\Ncal^{(0)}(v_i) := \{v_i\}$ and for $k \geq 1$, $\Ncal^{(k)}(v_i) := \{ v_j: (v_i, v_j) \in \Ecal, v_i \in \Ncal^{(k-1)}(v_i) \} \cup \Ncal^{(k-1)}(v_i)$.
\section{Additional Related Works}\label{sec:more_rel_works}

In this section, we present additional related works for this paper. We first review prior works on general-purpose machine unlearning and the privacy vulnerabilities of unlearning. Next, we provide a comprehensive list of related works on graph unlearning and membership inference attacks (MIA) on GNNs, supplementing our earlier discussion in Section~\ref{sec:rel_works}.

\subsection{Machine Unlearning} 

Machine unlearning aims to remove the influence of specific training samples from a trained model, balancing utility, removal guarantees, and efficiency~\cite{bourtoule2021machine,liu2024breaking}. 
Unlike full retraining, unlearning seeks practical alternatives to efficiently revoke data. These methods can help removing data effect in simple learning settings like statistical query learning~\cite{cao2015towards}, and recently have broad applications in LLMs~\cite{yao2024large,liu2024large,liu2025rethinking}, generative models~\cite{li2024machine,zhang2024unlearncanvas,zhang2024defensive}, e-commerce~\cite{chen2022recommendation,li2023ultrare,zhang2024recommendation}, and graph learning~\cite{chen2022graph,wu2023gif,dong2024idea}. 

Existing machine unlearning methods broadly fall into two categories: exact and approximate unlearning. Exact unlearning methods, such as SISA~\cite{bourtoule2021machine}, partition data into shards, train sub-models in each shard independently, and merge them together, which allows targeted retraining for removing certain data from a shard. ARCANE~\cite{yan2022arcane} improves this by framing unlearning as one-class classification and caching intermediate states, enhancing retraining efficiency. These ideas have been extended to other domains, including graph learning~\cite{chen2022graph,wang2023inductive} and ensemble methods~\cite{brophy2021machine}. However, exact methods often degrade performance due to weak sub-models and straightforward ensembling.

Approximate unlearning methods update model parameters to emulate removal effects with better efficiency-performance trade-offs. For instance, Jia et al.~\cite{jia2023model} uses pruning and sparsity-regularized fine-tuning to approximate exact removal. Tarun et al.~\cite{tarun2023fast} propose a model-agnostic class removal technique using error-maximizing noise and a repair phase. Liu et al.~\cite{liu2023muter} tackle adversarially trained models with a closed-form Hessian-guided update, approximated efficiently without explicit inversion.

\subsection{Privacy Vulnerabilities of Machine Unlearning}  
A growing line of research investigates the privacy vulnerabilities of machine unlearning. Chen et al.~\cite{chen2021machine} show that if an adversary has access to both the pre- and post-unlearning black-box model outputs, they can infer the membership status of a target sample. A recent work~\cite{thudi2022necessity} highlights the limitations of approximate unlearning by arguing that unlearning is only well-defined at the algorithmic level, and parameter-level manipulations alone may be insufficient. Unlearning inversion attacks~\cite{hu2024learn} go further by reconstructing both features and labels of training samples using two model versions, while Bertran et al.~\cite{bertran2024reconstruction} provide an in-depth analysis of such reconstruction attacks in regression settings.

These findings motivate our investigation into unlearned GNNs, where we aim to capture residual signals left by the unlearned data. While prior attacks mostly target general-purpose models and focus on red-teaming under strong assumptions, such as access to both pre- and post-unlearning models~\cite{chen2021machine,hu2024learn,bertran2024reconstruction}, our work considers a more practical black-box setting motivated by real-world social network scenarios. Furthermore, our attack explicitly accounts for the unique structural and relational properties of graphs, in contrast to prior work primarily focused on i.i.d. data.

\subsection{Graph Unlearning} Graph Unlearning enables the efficient removal of unwanted data’s influence from trained graph ML models~\cite{chen2022graph,said2023survey,fan2025opengu}. This removal process balances model utility, unlearning efficiency, and removal guarantees, following two major lines of research: retrain-based unlearning and approximate unlearning. 

Retrain-based unlearning partitions the original training graph into disjoint subgraphs, training independent submodels on them, enabling unlearning through retraining on a smaller subset of the data. Specifically, GraphEraser~\cite{chen2022graph} pioneered the first retraining-based unlearning framework for GNNs, utilizing balanced clustering methods for subgraph partitioning and ensembling submodels in prediction with trainable fusion weights to enhance model utility. 
GUIDE~\cite{wang2023inductive} presents an important follow-up by extending the retraining-based unlearning paradigm to inductive graph learning settings. 
Subsequently, many studies~\cite{wang2023inductive,li2023ultrare,zhang2024graph,li2025community} have made significant contributions to improving the Pareto front of utility and efficiency in these methods, employing techniques such as data condensation~\cite{li2025community} and enhanced clustering~\cite{li2023ultrare,zhang2024graph}. 

Approximate unlearning efficiently updates model parameters to remove unwanted data. Certified graph unlearning~\cite{chien2022certified} provides an important early exploration of approximate unlearning in SGC~\cite{wu2019simplifying}, with provable unlearning guarantees. GraphGuard~\cite{wu2024graphguard} introduces a comprehensive system to mitigate training data misuse in GNNs, featuring a significant gradient ascent unlearning method as one of its core components. 
These gradient ascent optimization objectives can be approximated using Taylor expansions, inspiring several beautiful influence function-based methods~\cite{koh2017understanding,wu2023gif,wu2023certified}. Specifically, 
GIF~\cite{wu2023gif} presents a novel influence function-based unlearning approach tailored to graph data, considering feature, node, and edge unlearning settings, while CEU~\cite{wu2023certified} focuses on edge unlearning and establishes theoretical bounds for removal guarantees. 
Building on this, IDEA~\cite{dong2024idea} proposes a theoretically sound and flexible approximate unlearning framework, offering removal guarantees across node, edge, and feature unlearning settings. 
Recent innovative works have further advanced the scalability~\cite{pan2023unlearning,li2024tcgu,yi2025scalable,yang2025erase}, model utility~\cite{li2024towards,zhang2025node}, and dynamic adaptivity~\cite{zhang2025dynamic} of approximate unlearning methods, pushing the boundaries of promising applications. Additionally, a noteworthy contribution is GNNDelete~\cite{cheng2023gnndelete}, which unlearns graph knowledge by employing an intermediate node embedding mask between GNN layers. This method offers a unique solution that differs from both retraining and model parameter modification approaches. 

In this paper, we explore the privacy vulnerabilities of graph unlearning by proposing a novel membership inference attack tailored to unlearned GNN models, introducing a new defence frontier that graph unlearning should consider from a security perspective.

\subsection{Membership Inference Attack for GNNs} Membership Inference Attack (MIA) is a privacy attack targeting ML models, aiming to distinguish whether a specific data point belongs to the training set~\cite{shokri2017membership,hu2022membership,dai2023unified,liu2025exposing}.
Such attacks have profoundly influenced ML across various downstream applications, including computer vision~\cite{chen2020gan,choquette2021label}, NLP~\cite{mattern2023membership,meeus2024did}, and recommender systems~\cite{zhang2021membership,yuan2023interaction}. 
Recently, MIA has been extended to graph learning, where a pioneering work~\cite{duddu2020quantifying} explored the feasibility of membership inference in classical graph embedding models. Subsequently, interest has shifted towards attacking graph neural networks (GNNs), with several impactful and innovative studies revealing GNNs' privacy vulnerabilities in node classification~\cite{he2021node,he2021stealing,olatunji2021membership,zhang2022inference} and graph classification tasks~\cite{wu2021adapting,lin2024stealing}, covering cover node-level~\cite{he2021node,olatunji2021membership}, link-level~\cite{he2021stealing}, and graph-level~\cite{wu2021adapting,zhang2022inference} inference risks. Building on this, GroupAttack~\cite{zhang2023demystifying} presents a compelling advancement in link-stealing attacks~\cite{he2021stealing} on GNNs, theoretically demonstrating that different edge groups exhibit varying risk levels and require distinct attack thresholds, while a label-only attack has been proposed to target node-level privacy vulnerabilities~\cite{conti2022label} with a stricter setting. Another significant line of research involves graph model inversion attacks, which aim to reconstruct the graph structure using model gradients from white-box models~\cite{zhang2021graphmi} or approximated gradients from black-box models~\cite{zhang2022model}. 

Despite the impressive contributions of previous MIA studies in graph ML models, existing approaches overlook GNNs containing unlearned sensitive knowledge and do not focus on recovering such knowledge from unlearned GNN models.

\section{Model Details}\label{sec:model_details}

In this section, we present the detailed computation of the TrendAttack, considering both the training and attack processes.

\subsection{Shadow Training Algorithm for TrendAttack}

\begin{algorithm}[!ht]
\caption{Attack Model Training}
\label{alg:shadow_training}
\begin{algorithmic}[1]
\Require Shadow dataset $\Gcal^{\mathrm{sha}}_{\mathrm{orig}}$, GNN architecture $f'$
\Ensure Attack model $g(\cdot, \cdot)$
\State \emph{// Train and unlearn the victim model}
\State Train the shadow victim model on $\Gcal^{\mathrm{sha}}_{\mathrm{orig}}$ with Eq.~\eqref{eq:param_orig} to obtain $\thetab_{\mathrm{orig}}'$ \Comment{Train victim model}
\State Randomly select some unlearned edges $\Delta\Ecal^{\mathrm{sha}}$ from original edges $\Ecal^{\mathrm{sha}}_{\mathrm{orig}}$
\State Unlearn $\Delta\Ecal^{\mathrm{sha}}$ with Eq.~\eqref{eq:param_unlearned} to obtain $\thetab_{\mathrm{un}}'$ \Comment{Unlearn victim model}
\State $\Ecal^{\mathrm{sha}}_{\mathrm{un}} \gets \Ecal^{\mathrm{sha}}_{\mathrm{orig}} \setminus \Delta\Ecal^{\mathrm{sha}}$ \Comment{Remove the unlearned edges from $\Gcal^{\mathrm{sha}}_{\mathrm{orig}}$}
\State \emph{// Construct the query set $\Qcal_{\mathrm{sha}}$}
\State Randomly select some existing edges from $\Ecal^{\mathrm{sha}}_{\mathrm{un}}$ to obtain $\Qcal^+_{\mathrm{mem}}$ \Comment{$\Qcal^+_{\mathrm{mem}} \subseteq \Ecal^{\mathrm{sha}}_{\mathrm{un}}$}
\State Randomly select some negative edges $\Qcal^-$ \Comment{$\Qcal^- \cap \Ecal^{\mathrm{sha}}_{\mathrm{un}} = \emptyset$}
\State Initialize the unlearned set of edges $\Qcal^+_{\mathrm{un}} \gets \Delta\Ecal^{\mathrm{sha}}$ 
\State Initialize the entire query set $\Qcal_{\mathrm{sha}} \gets \Qcal^+_{\mathrm{un}} \cup \Qcal^+_{\mathrm{mem}} \cup \Qcal^-$
\State \emph{// Attack model training}
\State Obtain the partial knowledge for the query set $\Kcal_{\Qcal_{\mathrm{sha}}} = (\Pcal_{\Qcal_{\mathrm{sha}}}, \Fcal_{\Qcal_{\mathrm{sha}}})$
\State Train the attack model $g$ on $\Qcal_{\mathrm{sha}}$ with $\Kcal_{\Qcal_{\mathrm{sha}}}$ and $\Lcal_{\mathrm{attack}}$ in Eq.~\eqref{eq:attack_train}
\State \textbf{return} Attack model $g$
\end{algorithmic}
\end{algorithm}

In this algorithm, we first train and unlearn the shadow victim model $f'$ (lines 1–5), and then construct the query set $\Qcal_{\mathrm{sha}}$ (lines 6–10), which includes three different types of edges. Next, we train the attack model with the link prediction objective $\Lcal_{\mathrm{attack}}$ in Eq.~\eqref{eq:attack_train}, with $\Qcal_{\mathrm{sha}}$ as the pairwise training data (lines 11–13). Then, we return the attack model $g$ for future attacks (line 14).

\subsection{The Attack Process of TrendAttack}

\begin{algorithm}[!ht]
\caption{Attack Process}
\label{alg:attack}
\begin{algorithmic}[1]
\Require Unlearned graph $\Gcal_{\mathrm{un}}$, black-box unlearned model $f_{\Gcal_{\mathrm{un}}}(\cdot; \thetab_{\mathrm{un}})$, query set $\Qcal$, trained attack model $g(\cdot, \cdot)$, trend feature order $k$
\Ensure Membership predictions on the query set $\wh{\yb}$
\State \emph{// Request the partial knowledge $\Kcal_\Qcal = (\Pcal_\Qcal, \Fcal_\Qcal)$}
\State Obtain the nodes of interest $\Vcal_\Qcal \gets \{v_i:(v_i,v_j) \in \Qcal~\mathrm{or}~(v_j, v_i) \in \Qcal\}$
\State $\Pcal_\Qcal\gets \emptyset, \Fcal_\Qcal\gets \emptyset$
\For{$v_i \in \Vcal_\Qcal$}
    \State $\Fcal_\Qcal \gets \Fcal_\Qcal \cup \{(v_i, \xb_i)\} $ \Comment{Feature knowledge}
    \State Request the available neighborhood $\wh{\Ncal}^{(0)}(v_i), \cdots,\wh{\Ncal}^{(k)}(v_i)$ from $\Gcal_{\mathrm{un}}$ 
    \For{$v_j \in \bigcup_{r=0}^k \wh{\Ncal}^{(r)}(v_i)$}  \Comment{Probability knowledge}
        \State $\pb_j \gets f_{\Gcal_{\mathrm{un}}}(v_j; \thetab_{\mathrm{un}})$
        \State $\Pcal_\Qcal \gets \Pcal_\Qcal \cup \{(v_j, \pb_j)\} $
    \EndFor
\EndFor
\State \emph{// Membership inference}
\For{$v_i \in \Vcal_\Qcal$} \Comment{Build trend features}
    \State Compute trend features $\wt{\taub}_i$ for node $v_i$ with Eq.~\eqref{eq:trend_feature_1} and Eq.~\eqref{eq:trend_feature_2}
\EndFor
\For{$(v_i, v_j) \in \Qcal$} \Comment{Predict with attack model $g$}
    \State Compute the model prediction $\wh{y}_{i, j} \gets g(v_i, v_j)$ with Eq.~\eqref{eq:attack_model}
\EndFor
\State \textbf{return} Membership predictions $\wh{\yb}$
\end{algorithmic}
\end{algorithm}

In this algorithm, we first request the partial knowledge $\Kcal_\Qcal$ from the unlearned graph $\Gcal_{\mathrm{un}}$ and the black-box victim model $f_{\Gcal_{\mathrm{un}}}(\cdot; \thetab_{\mathrm{un}})$ for our links of interest $\Qcal$ (lines 1–11). This process is highly flexible and can effortlessly incorporate many different levels of the attacker's knowledge. For instance, the feature knowledge $\Fcal_\Qcal$ (line 5) is optional if the model API owner does not respond to node feature requests on the unlearned graph $\Gcal_{\mathrm{un}}$. 

Moreover, for the probability knowledge $\Pcal_\Qcal$ (lines 6–10), the more we know about the neighborhood $\wh{\Ncal}^{(0)}, \cdots,\wh{\Ncal}^{(k)}$ of nodes of interest $\Vcal_\Qcal$, the more accurately we can construct our trend features. This approach is fully adaptive to any level of access to the unlearned graph and model API. In the strictest setting, where we only have access to the node of interest itself, we have $k=0$, and our model perfectly recovers previous MIA methods with no performance loss or additional knowledge requirements. An empirical study on the impact of trend feature orders can be found in Figure~\ref{fig:ablation_trend_order}.

After requesting the partial knowledge, we compute the trend features and then predict the membership information (lines 12–18). In the end, we return our attack results (line 19).
\section{Missing Proofs in Section~\ref{sec:design_motivation}}\label{sec:proof}

In this section, we provide formal definitions for all concepts introduced in Section~\ref{sec:design_motivation} and supplement the missing technical proofs. 
We begin by describing the architecture and training process of linear GCN, and then compute the edge influence on model weights. Next, we analyze how edge influence affects the model output, considering both the direct impact from the edge itself and the indirect influence mediated through model weights.

\subsection{Linear GCN}

In this analysis, we consider a simple but effective variant of GNNs, linear GCN, which is adapted from the classical SGC model~\cite{wu2019simplifying}.

\begin{definition}[Linear GCN]\label{dfn:linear_gcn}
   Let $\Cb\in[0, 1]^{n\times n}$ be the propagation matrix, $Xb\in\R^{n\times d}$ be the input feature matrix, and $\wb\in\R^d$ denote the learnable weight vector. Linear GCN computes the model output as follows:
   \begin{align*}
       f_{\mathrm{LGN}}(\Cb,\Xb; \wb):= \Cb \Xb \wb.
   \end{align*}
\end{definition}
\begin{remark}[Universality of Linear GCN]\label{rmk:lgn_universality}
    The propagation matrix $\Cb$ is adaptable to any type of convolution matrices, recovering multiple different types of GNNs, including but not limited to:
    \begin{itemize}
        \item One-layer GCN: $\Cb_{\text{1-GCN}}: = (\Db+\Ib)^{-0.5}(\Ab+\Ib)(\Db+\Ib)^{-0.5}$;
        \item $k$-layer SGC: $\Cb_{\text{k-SGC}}: = \Cb_{\text{1-GCN}}^k$;
        \item Infinite layer PPNP: $\Cb_{\text{PPNP}} := (\Ib-(1-\alpha)\Cb_{\text{1-GCN}})^{-1}$;
        \item $k$-layer APPNP: $\Cb_{\text{k-APPNP}}:= (1-\alpha)^k \Cb_{\text{1-GCN}}^k +\alpha\sum_{l=0}^{k-1}(1-\alpha)^l\Cb_{\text{1-GCN}}^l$;
        \item One-layer GIN: $\Cb_{\text{GIN}}:=\Ab+\Ib$.
    \end{itemize}
\end{remark}

To analyze the output of a single node, we state the following basic result without proof.

\begin{fact}[Single Node Output]\label{fact:single_node_output}
    The output of Linear GCN for a single node $v_i\in \Vcal$ is $$f_{\mathrm{LGN}}(\Cb,\Xb; \wb)_i = (\Cb \Xb  \wb)_i =  \sum_{j=1}^n \Cb_{j, i}(\xb_i^\top \wb).$$
\end{fact}

In the training process of Linear GCNs, we consider a general least squares problem~\cite{shin2023efficient}, which applies to both regression and binary classification tasks.
\begin{definition}[Training of Linear GCNs]\label{dfn:lin_gcn_optim}
    Let $\yb\in\R^d$ denote the label vector, training the linear GCN in Definition~\ref{dfn:linear_gcn} is equivalent to minimize the following loss function:
    \begin{align*}
        \Lcal(\Cb,\Xb,\wb,\yb):=\frac{1}{2}\lVert \yb - \Cb \Xb \wb\rVert_2^2.
    \end{align*}
\end{definition}

\begin{proposition}[Closed-form Solution for Linear GCN Training]\label{prop:lgcn_closed_form}
    The following optimization problem:
    \begin{align*}
        \min_\wb   \Lcal(\Cb,\Xb,\wb,\yb)
    \end{align*}
    has a closed-form solution
    \begin{align*}
        \wb^\star = (\Xb^\top \Cb^\top \Cb \Xb)^{-1}\Xb^\top \Cb^\top \yb.
    \end{align*}
\end{proposition}
\begin{proof}
    The loss function $\Lcal$ is convex. Following the first-order optimality condition, we can set the gradient to zero and solve for $\wb$, which trivially yields the desired result.
\end{proof}

\subsection{Edge Influence on Model Weight}

In this section, we examine how a specific set of edges $\Delta\Ecal$ influences the retrained model weights $\wb^\star$. We first define the perturbation matrix, which is the adjacency matrix for an edge subset $\Delta\Ecal$, and can be used to perturb the original propagation matrix $\Cb$ for influence evaluation.

\begin{definition}[Perturbation Matrix]\label{dfn:perturb_matrix}
    For an arbitrary edge subset $\Delta\Ecal$, the corresponding perturbation matrix $\Xib^{\Delta\mathcal{E}}$ is defined as:
    $$\Xib^{\Delta\mathcal{E}}:=\sum_
    {(v_i,v_j)\in\Delta\mathcal{E}}\eb_i\eb_j^\top.$$ 
\end{definition}

For notation simplicity, we sometimes ignore $\Delta\Ecal$ and use $\Xib$ as a shorthand notation for $\Xib$ in this paper. 

Next, we begin with a general case of edge influence that does not address the linear GCN architecture. 
\begin{lemma}[Edge Influence on Model Weight, General Case]\label{lem:lgcn_weight_infl}
    Let $\Xib^{\Delta\mathcal{E}}\in\R^{n\times n}$ be a perturbation matrix as defined in Definition~\ref{dfn:perturb_matrix}. 
    Let $\epsilon\in\R$ be a perturbation magnitude and $\wb^\star(\epsilon)$ be the optimal model weight after perturbation. Considering an edge perturbation $\Cb(\epsilon) := \Cb + \epsilon\Xib^{\Delta\Ecal}$ on the original propagation matrix $\Cb$, the sensitivity of the model weight $\wb^\star$ can be characterized by:
    \begin{align*}
        \mathcal{I}_\epsilon(\wb^\star):= & ~ \frac{\d \wb^\star(\epsilon)}{\d \epsilon}\Bigr|_{\epsilon=0}\\
        = & ~ -\Hb^{-1}\frac{\partial}{\partial\epsilon}\nabla_{\wb}\Lcal(\Cb + \epsilon\Xib^{\Delta\Ecal}, \Xb, \wb^\star, \yb)\Bigr|_{\epsilon=0},
    \end{align*}
    where $\Hb:=\nabla_\wb^2\Lcal(\Cb,\Xb,\wb^\star,\yb)$.
\end{lemma}
\begin{proof}
    For notation simplicity, we use $\Xib$ as a shorthand notation for $\Xib^{\Delta\Ecal}$ in this proof. 
    Since $\wb^\star(\epsilon)$ is the minimizer of the perturbed loss function $\Lcal(\Cb(\epsilon), \Xb, \wb, \yb)$, we have:
    \begin{align*}
        \wb^\star(\epsilon) = \argmin_\wb \Lcal(\Cb+\epsilon\Xib, \Xb, \wb, \yb).
    \end{align*}
    Examining the first-order optimality condition of the minimization problem, we have:
    \begin{align}\label{eq:grad_Lw_eq_0}
        \nabla_\wb \Lcal(\Cb+\epsilon\Xib, \Xb, \wb^\star(\epsilon), \yb) = 0.
    \end{align}
    Let us define the change in parameters as:
    \begin{align*}
        \Delta \wb &:=~\wb^\star(\epsilon) - \wb^\star.
    \end{align*}
    Therefore, since $\wb^\star(\epsilon)\rightarrow \wb^\star$ as $\epsilon\rightarrow 0$, we expand Eq.~\eqref{eq:grad_Lw_eq_0} with multi-variate Taylor Series at the local neighborhood of $(\wb^\star, 0)$ to approximate the value of $\nabla_\wb \Lcal$ at $(\wb^\star + \Delta \wb, \epsilon)$:
    \begin{equation}
    \begin{aligned}
    \label{eq:grad_Lw_eq_0_taylor}
        0 &=~ \nabla_\wb \Lcal(\Cb+\epsilon\Xib, \Xb, \wb^\star(\epsilon), \yb) \\ 
        &=~ \nabla_\wb \Lcal(\Cb+\epsilon\Xib, \Xb, \wb^\star +\Delta \wb, \yb) \\ 
        &=~ \nabla_\wb \Lcal(\Cb + \epsilon\Xib, \Xb, \wb^\star, \yb) + \nabla^2_{\wb}\Lcal(\Cb + \epsilon\Xib, \Xb, \wb^\star, \yb)\Delta \wb \\ &\quad\quad\quad +\epsilon \frac{\partial}{\partial\epsilon}\nabla_{\wb}\Lcal(\Cb + \epsilon\Xib, \Xb, \wb^\star, \yb)\Bigr|_{\epsilon=0} +o(\lVert\Delta \wb\rVert + \lvert \epsilon\rvert) \\ 
        &=~ \nabla_\wb \Lcal(\Cb, \Xb, \wb^\star, \yb) + \nabla^2_{\wb}\Lcal(\Cb, \Xb, \wb^\star, \yb)\Delta \wb \\
        &\quad\quad\quad + \epsilon \frac{\partial}{\partial\epsilon}\nabla_{\wb}\Lcal(\Cb + \epsilon\Xib, \Xb, \wb^\star, \yb)\Bigr|_{\epsilon=0} +o(\lVert\Delta \wb\rVert + \lvert \epsilon\rvert) \\
        &=~ \nabla^2_{\wb}\Lcal(\Cb, \Xb, \wb^\star, \yb)\Delta \wb+ \epsilon \frac{\partial}{\partial\epsilon}\nabla_{\wb}\Lcal(\Cb + \epsilon\Xib, \Xb, \wb^\star, \yb)\Bigr|_{\epsilon=0}\\
        &\quad\quad\quad +o(\lVert\Delta \wb\rVert + \lvert \epsilon\rvert)
    \end{aligned}
    \end{equation}
    where the first equality follows from Eq.~\eqref{eq:grad_Lw_eq_0}, the second equality follows from the definition of $\Delta \wb$, the third equality follows from Taylor Series, the fourth equality follows from $\epsilon \rightarrow 0$, and the last equality follows from the fact that $\wb^\star$ is the minimizer of loss function $L$. 
    
    Let the Hessian matrix at $\wb^\star$ be $\Hb:=\nabla_\wb^2\Lcal(\Cb,\Xb,\wb^\star, \yb)$. Ignoring the remainder term and rearrange Eq.~\eqref{eq:grad_Lw_eq_0_taylor}, we can conclude that:
    \begin{align*}
        \Delta \wb = -\epsilon \Hb^{-1} \frac{\partial}{\partial\epsilon}\nabla_{\wb}\Lcal(\Cb + \epsilon\Xib, \Xb, \wb^\star, \yb)\Bigr|_{\epsilon=0}.
    \end{align*}
    Dividing both sides by $\epsilon$ and taking the limit $\epsilon\rightarrow 0$, we have:
    \begin{align*}
        \frac{\d \wb^\star(\epsilon)}{\d \epsilon}\Bigr|_{\epsilon=0} = -\Hb^{-1}\frac{\partial}{\partial\epsilon}\nabla_{\wb}\Lcal(\Cb + \epsilon\Xib, \Xb, \wb^\star, \yb)\Bigr|_{\epsilon=0}.
    \end{align*}
    This completes the proof.
\end{proof}

Afterward, we extend the general case of edge influence on model weights to the Linear GCN framework, providing a closed-form solution for the edge-to-weight influence in this model.

\begin{theorem}
[Edge Influence on Model Weight, Linear GCN Case]\label{thm:lgcn_weight_infl_spec}
    Given an arbitrary edge perturbation matrix $\Xib$ as defined in Definition~\ref{dfn:perturb_matrix}, the influence function for the optimal model weight $\wb^\star$ is:
    \begin{align*}
        \mathcal{I}_\epsilon(\wb^\star):= & ~ \frac{\d \wb^\star(\epsilon)}{\d \epsilon}\Bigr|_{\epsilon=0}\\
        = & ~  \Hb^{-1}(\Xb^\top \Xib^\top \yb - \Xb^\top \Xib^\top \Cb \Xb \wb^\star - \Xb^\top \Cb^\top \Xib \Xb \wb^\star).
    \end{align*}
\end{theorem}
\begin{proof}
We start from the general result in Lemma~\ref{lem:lgcn_weight_infl} which states that
\begin{align*}
\mathcal{I}_\epsilon(\wb^\star) := \frac{\mathrm{d}\wb^\star(\epsilon)}{\mathrm{d}\epsilon}\Biggr|_{\epsilon=0} = -\Hb^{-1}\frac{\partial}{\partial\epsilon}\nabla_{\wb}\Lcal(\Cb+\epsilon\Xib, \Xb, \wb^\star, \yb)\Biggr|_{\epsilon=0},
\end{align*}
where $\Hb = \nabla^2_{\wb} \Lcal(\Cb, \Xb, \wb^\star, \yb)$ is the Hessian of the loss with respect to $\wb$ evaluated at $\wb^\star$.

Recalling the loss function in Definition~\ref{dfn:lin_gcn_optim}, when we introduce the perturbation $\Cb(\epsilon) = \Cb+\epsilon\Xib$, the loss becomes
\begin{align*}
\Lcal(\Cb+\epsilon\Xib, \Xb, \wb, \yb) = \frac{1}{2}\|\yb - (\Cb+\epsilon\Xib)\Xb \wb\|_2^2.
\end{align*}
Thus, the gradient with respect to $\wb$ is
\begin{align*}
\nabla_\wb \Lcal(\Cb+\epsilon\Xib, \Xb, \wb, \yb) = -\Xb^\top (\Cb+\epsilon\Xib)^\to\Pb\bigl[\yb - (\Cb+\epsilon\Xib)\Xb \wb\bigr].
\end{align*}
Evaluating at $\wb=\wb^\star$ yields
\begin{align*}
\nabla_\wb \Lcal(\Cb+\epsilon\Xib, \Xb, \wb^\star, \yb) = -\Xb^\top (\Cb+\epsilon\Xib)^\to\Pb\bigl[\yb - (\Cb+\epsilon\Xib)\Xb\wb^\star\bigr].
\end{align*}

We now differentiate this expression with respect to $\epsilon$ and evaluate at $\epsilon=0$. Writing the gradient as the sum of two terms, we have:
\begin{align*}
& ~ \nabla_\wb \Lcal(\Cb+\epsilon\Xib, \Xb, \wb^\star, \yb) \\
= &~ -\underbrace{\Xb^\top (\Cb+\epsilon\Xib)^\top  \yb}_{:= \Tb_1} + \underbrace{\Xb^\top (\Cb+\epsilon\Xib)^\top (\Cb+\epsilon\Xib)\Xb\wb^\star}_{:= \Tb_2} \\ 
=&~ -\Tb_1 + \Tb_2
.
\end{align*}

Now we differentiate each term with respect to $\epsilon$. Specifically, for the first term $\Tb_1$, we have the following result:
\begin{align*}
    \frac{\mathrm{d}\Tb_1}{\mathrm{d}\epsilon} =&~
    \frac{\mathrm{d}}{\mathrm{d}\epsilon}\Bigl[-\Xb^\top (\Cb+\epsilon\Xib)^\top y\Bigr]\\
    =&~ -\Xb^\top \Xib^\top y.
\end{align*}

For the second term, we can conclude by basic matrix algebra that:
\begin{align*}
    \frac{\mathrm{d}\Tb_2}{\mathrm{d}\epsilon} =&~ \frac{\mathrm{d}}{\mathrm{d}\epsilon}\Bigl[ \Xb^\top (\Cb+\epsilon\Xib)^\top (\Cb+\epsilon\Xib)\Xb\wb^\star \Bigr] \\ 
    =& ~\Xb^\top (\Cb^\top\Xib + \Xib^\top \Cb + 2\epsilon \Xib^\top \Xib )\Xb\wb^\star
\end{align*}

Combining both terms, we obtain
\begin{align*}
\frac{\partial}{\partial\epsilon}\nabla_\wb \Lcal(\Cb+\epsilon\Xib, \Xb, \wb^\star, \yb)\Biggr|_{\epsilon=0} = -\Xb^\top \Xib^\top \yb + \Xb^\top (\Cb^\top\Xib + \Xib^\top \Cb)\Xb\wb^\star.
\end{align*}
Substituting back into the expression for $\mathcal{I}_\epsilon(\wb^\star)$ gives:
\begin{align*}
\mathcal{I}_\epsilon(\wb^\star) 
&= -\Hb^{-1}\Biggl(-\Xb^\top \Xib^\top \yb + \Xb^\top (\Cb^\top\Xib + \Xib^\top \Cb)\Xb\wb^\star\Biggr)\\ 
&= \Hb^{-1}\Bigl(\Xb^\top \Xib^\top \yb - \Xb^\top \Cb^\top\Xib \Xb\wb^\star - \Xb^\top \Xib^\top \Cb \Xb\wb^\star\Bigr).
\end{align*}
This finishes the proof.

\end{proof}

\subsection{Edge Influence on Model Output}

In this section, we compute the influence function for edges on the final model output, considering both direct edge influence and the effect of model parameters on the output. First, we calculate the influence on all model outputs.

\begin{lemma}[Edge Influence on Model Output, Linear GCN Case]\label{thm:lgcn_out_infl_spec}
    Let $\Xib^{\Delta\mathcal{E}}\in\R^{n\times n}$ be a perturbation matrix as defined in Definition~\ref{dfn:perturb_matrix}. Let $\epsilon\in\R$ be a perturbation magnitude and $\wb^\star(\epsilon)$ be the optimal model weight after perturbation. Considering an edge perturbation $\Cb(\epsilon) := \Cb + \epsilon\Xib^{u, v}$ on the original propagation matrix $\Cb$, the sensitivity of the Linear GCN model output $f_{\mathrm{LGN}}(\Cb + \epsilon\Xib, \Xb; \wb^\star(\epsilon))$ as defined in Definition~\ref{dfn:linear_gcn} can be characterized by:
    \begin{align*}
        &~\mathcal{I}_\epsilon(f_{\mathrm{LGN}})\\ :=&~\frac{\d f_{\mathrm{LGN}}(\Cb + \epsilon\Xib, \Xb; \wb^\star(\epsilon))}{\d \epsilon}\Bigr|_{\epsilon=0}\\ 
        =&~ \Xib \Xb \wb^\star + \Cb \Xb \Hb^{-1}\Bigl(\Xb^\top \Xib^\top \yb - \Xb^\top \Cb^\top\Xib \Xb\wb^\star - \Xb^\top \Xib^\top \Cb \Xb\wb^\star\Bigr), 
    \end{align*}
    where $\Hb = \nabla^2_{\wb} \Lcal(\Cb, \Xb, \wb^\star, \yb)$ is the Hessian of the loss with respect to $\wb$ evaluated at $\wb^\star$.
\end{lemma}
\begin{proof}
    The edge perturbation $\epsilon\Xib$ influences $f_{\mathrm{LGN}}$ from two different aspects: propagation matrix and optimal model weights. Thus, the influence function can be derived by chain rule as follows:
    \begin{align*}
        &~\mathcal{I}_\epsilon(f_{\mathrm{LGN}}) \\
        =&~ \frac{\d f_{\mathrm{LGN}}(\Cb + \epsilon\Xib, \Xb; \wb^\star(\epsilon))}{\d \epsilon}\Bigr|_{\epsilon=0} \\ 
        =&~ (\frac{\partial f_{\mathrm{LGN}}(\Cb + \epsilon\Xib, \Xb; \wb^\star(\epsilon))}{\partial (\Cb+\epsilon\Xib)}\cdot \frac{\partial(\Cb+\epsilon\Xib)}{\partial\epsilon})\Bigr|_{\epsilon=0} \\
        &\quad\quad\quad+ (\frac{\partial f_{\mathrm{LGN}}(\Cb + \epsilon\Xib, \Xb; \wb^\star(\epsilon))}{\partial \wb^\star(\epsilon)} \cdot \frac{\partial \wb^\star(\epsilon)}{\partial\epsilon})\Bigr|_{\epsilon=0} \\ 
        =&~ \Xib \Xb \wb^\star + (\frac{\partial f_{\mathrm{LGN}}(\Cb + \epsilon\Xib, \Xb; \wb^\star(\epsilon))}{\partial \wb^\star(\epsilon)} \cdot \frac{\partial \wb^\star(\epsilon)}{\partial\epsilon})\Bigr|_{\epsilon=0} \\ 
        =&~ \Xib \Xb \wb^\star + \Cb \Xb (\frac{\partial \wb^\star(\epsilon)}{\partial\epsilon})\Bigr|_{\epsilon=0} \\
        =&~ \Xib \Xb \wb^\star + \Cb \Xb \Hb^{-1}\Bigl(\Xb^\top \Xib^\top \yb - \Xb^\top \Cb^\top\Xib \Xb\wb^\star - \Xb^\top \Xib^\top \Cb \Xb\wb^\star\Bigr).
    \end{align*}
where the first equality follows from the definition of the influence function, the second equality follows from the chain rule, the third and fourth equality follow from basic matrix calculus, and the last equality follows from Theorem~\ref{thm:lgcn_weight_infl_spec}.
\end{proof}

Next, we present an immediate corollary of this lemma and determine the influence of a single edge on all nodes' output.

\begin{corollary}[Single Edge Influence on Model Output, Linear GCN Case]\label{cor:lgcn_out_infl_spec_one_edge}
    Let $\Zb := \Cb\Xb$ and $\zb_l$ is the $l$‐th row of $\Zb$. Considering the influence of one specific edge $(v_i, v_j)$ on undirected graphs (i.e., $\Xib=\eb_i \eb_j^\top + \eb_j \eb_i^\top$), the influence function for the model output is:
    \begin{align*}
        \mathcal{I}_\epsilon(f_{\mathrm{LGN}}) = (\xb_j^\top \wb^\star)\eb_i + (\xb_i^\top \wb^\star)\eb_j + (\qb \otimes \Cb \Xb )\mathrm{vec}(\Hb^{-1}),
    \end{align*}
    where $\Hb = \nabla^2_{\wb} \Lcal(\Cb, \Xb, \wb^\star, \yb)$ is the Hessian of the loss with respect to $\wb$ evaluated at $\wb^\star$ and  
    $$\qb:= (y_j - \zb_j^\top \wb^\star)\xb_i + (y_i - \zb_i^\top \wb^\star)\xb_j -((\xb_j^\top \wb^\star)\zb_i+(\xb_i^\top \wb^\star)\zb_j).$$
\end{corollary}
\begin{proof}
    This follows from basic algebra and the fact that \\$\mathrm{vec}(\Ab \Xb \Bb)=(\Bb \otimes \Ab) \mathrm{vec}(\Xb)$.
\end{proof}

Then, taking an element $\mathcal{I}_\epsilon(f_{\mathrm{LGN}})_k$ from the vector-form influence function over all model outputs $\mathcal{I}_\epsilon(f_{\mathrm{LGN}})$, we obtain the final result on the influence of a single edge on a single node's output.

\begin{theorem}[Single Edge Influence on Single Node's Model Output, Linear GCN Case]\label{cor:lgcn_out_pernode_infl_spec_one_edge}
     Let $\Zb := \Cb\Xb$ and $\zb_l$ is the $l$‐th row of $\Zb$. Considering the influence of one specific edge $(v_i, v_j)$ on undirected graphs (i.e., $\Xib=\eb_i\eb_j^\top + \eb_j \eb_i^\top$), the influence function for the model output for specific node $v_k$ is:
    \begin{align*}
        \mathcal{I}_\epsilon(f_{\mathrm{LGN}})_k =&~ \boldsymbol{1}\{v_k=v_i\}\cdot (\xb^\top_v \wb^\star) + \boldsymbol{1}\{v_k=v_j\}\cdot (\xb^\top_u \wb^\star) \\
         &\quad\quad+ \qb^\top \Hb^{-1}\zb_k \\ 
        =&~  \underbrace{\boldsymbol{1}{\{v_k=v_i\}}\cdot (\xb^\top_i \wb^\star) + \boldsymbol{1}{\{v_k=v_j\}}\cdot (\xb^\top_j \wb^\star)}_{\mathrm{edge~influence}} \\
        &~~~~- \underbrace{\langle(\xb_j^\top \wb^\star)\zb_i+(\xb_i^\top \wb^\star) \zb_j, \zb_k\rangle_{\Hb^{-1}}}_{\mathrm{magnitude~weight~influence}} \\
        &~~~~ 
        + \underbrace{\langle (y_j - \zb_j^\top \wb^\star)\xb_i+ (y_i - \zb_i^\top \wb^\star) \xb_j, \zb_k\rangle_{\Hb^{-1}}}_{\mathrm{error~weight~influence}},
    \end{align*}
     where $\Hb = \nabla^2_{\wb} \Lcal(\Cb, \Xb, \wb^\star, \yb)$ is the Hessian of the loss with respect to $\wb$ evaluated at $\wb^\star$.
\end{theorem}
\begin{proof}
    This follows from Corollary~\ref{cor:lgcn_out_infl_spec_one_edge}, basic algebra, the definition of $z$, and the definition of inner product w.r.t. a PSD matrix ($\langle \cdot,~\cdot\rangle_{\Ab}$). 
\end{proof}
\section{Preliminary Study}\label{sec:prelim_exp}

In this section, we present empirical evidence from preliminary experiments to support our main claims, probability similarity gap (Claim~\ref{clm:prob_sim_gap}) and confidence pitfall (Claim~\ref{clm:prob_trend}), in this paper.

\subsection{Probability Similarity Gap}

\begin{table*}[!ht]
\caption{\textbf{Average black-box probability similarity between different types of edges}. We consider three types of edges: negative edges ($\Qcal^-$), unlearned edges ($\Qcal^+_{\mathrm{un}}$), and other membership edges ($\Qcal^+_{\mathrm{mem}}$). The similarity measure is based on JS Divergence.} \label{tab:js_divergence}
  \centering
  \vskip -1em
  \begin{tabular}{llccc}
    \toprule
    \textbf{Dataset}      & \textbf{Unlearn Method} & $\mathrm{ProbSim}(\Qcal^-)$           & $\mathrm{ProbSim}(\Qcal^+_{\mathrm{un}})$      & $\mathrm{ProbSim}(\Qcal^+_{\mathrm{mem}})$           \\
    \midrule
    \multirow{3}{*}{\textbf{Cora}}        
                 & GIF             & 0.1979 ± 0.3147   & 0.6552 ± 0.3457   & 0.8001 ± 0.2689   \\
                 & CEU             & 0.1997 ± 0.3172   & 0.6553 ± 0.3502   & 0.8122 ± 0.2528   \\
                 & GA              & 0.1955 ± 0.3119   & 0.6511 ± 0.3486   & 0.8101 ± 0.2625   \\
    \midrule
    \multirow{3}{*}{\textbf{Citeseer}}   
                 & GIF             & 0.2689 ± 0.2997   & 0.6446 ± 0.3046   & 0.8250 ± 0.2116   \\
                 & CEU             & 0.2492 ± 0.3145   & 0.6248 ± 0.3317   & 0.8251 ± 0.2311   \\
                 & GA              & 0.2695 ± 0.3071   & 0.6397 ± 0.3200   & 0.8340 ± 0.2051   \\
    \midrule
    \multirow{3}{*}{\textbf{Pubmed}}     
                 & GIF             & 0.5833 ± 0.3507   & 0.7450 ± 0.2752   & 0.8843 ± 0.1682   \\
                 & CEU             & 0.5799 ± 0.3677   & 0.7494 ± 0.2856   & 0.8954 ± 0.1741   \\
                 & GA              & 0.5997 ± 0.3687   & 0.7548 ± 0.2865   & 0.9031 ± 0.1677   \\
    \midrule
    \multirow{3}{*}{\textbf{LastFM-Asia}}
                 & GIF             & 0.1608 ± 0.3413   & 0.8529 ± 0.3278   & 0.9308 ± 0.2234   \\
                 & CEU             & 0.1631 ± 0.3433   & 0.8564 ± 0.3224   & 0.9330 ± 0.2129   \\
                 & GA              & 0.1638 ± 0.3451   & 0.8439 ± 0.3411   & 0.9238 ± 0.2373   \\
    \bottomrule
  \end{tabular}
\end{table*}

To validate that different types of edges require different levels of similarity thresholds (Claim~\ref{clm:prob_sim_gap}), we present an important preliminary study in this section. Specifically, we adopt a GCN backbone and define a specific query set $\Qcal = \Qcal^+_{\mathrm{un}} \cup \Qcal^+_{\mathrm{mem}} \cup \Qcal^-$, which consists of 5\% of edges marked as unlearned ($\Qcal^+_{\mathrm{un}}$), 5\% of other membership edges ($\Qcal^+_{\mathrm{mem}}$), and 10\% of negative edges ($\Qcal^-$). For any two nodes $(v_i, v_j) \in \Qcal$, we compute the following similarity metric based on Jensen-Shannon divergence:
\begin{align*}
    \phi(\pb_i, \pb_j) 
    &= 1 - \frac{1}{2} \left[ \mathrm{KL}(\pb_i \,\|\, \mathbf{m}) + \mathrm{KL}(\pb_j \,\|\, \mathbf{m}) \right],
\end{align*}
where $\pb_i$ and $\pb_j$ are the predictive probability distributions of $v_i$ and $v_j$ from the unlearned victim model, and $\mb = \frac{1}{2}(\pb_i + \pb_j)$. The higher $\phi(\pb_i, \pb_j)$ is, the more similar the model's black-box predictions are on the two nodes. 

We report the mean and standard deviation of the probability similarity for each group in Table~\ref{tab:js_divergence}. The key observations are:

(i) There exists a clear and consistent gap in the probability similarity among the three edge types, aligning with Claim~\ref{clm:prob_sim_gap}:
\begin{align*}
    \mathrm{ProbSim}(\Qcal^-) < \mathrm{ProbSim}(\Qcal^+_{\mathrm{un}}) < \mathrm{ProbSim}(\Qcal^+_{\mathrm{mem}}).
\end{align*}
This supports the need for an adaptive prediction mechanism that distinguishes between unlearned and other membership edges, potentially improving prediction accuracy. Accordingly, our model in Eq.~\eqref{eq:attack_model} incorporates a learnable transformation on trend features to adjust predicted similarities obtained from MIA methods.

(ii) The variance within each group is relatively large, indicating that although a similarity gap exists, simple probability similarity computation may not be sufficient for membership inference. For instance, on the Cora dataset, the average probability similarity in $\Qcal^-$ is around 0.195, while its standard variance is nearly 0.31. This effect arises because, although most edges in $\Qcal^-$ have similarity near 0, there are outliers with high similarity levels. To address this large variance, we incorporate a broad range of probability-based similarity features, alongside node feature similarities used in prior MIA methods. These additional features enhance similarity representation and help stabilize membership inference.

\subsection{Confidence Pitfall}

\begin{figure*}[!ht]
  \centering
  \begin{subfigure}[b]{0.32\linewidth}
    \centering
    \includegraphics[width=\linewidth]{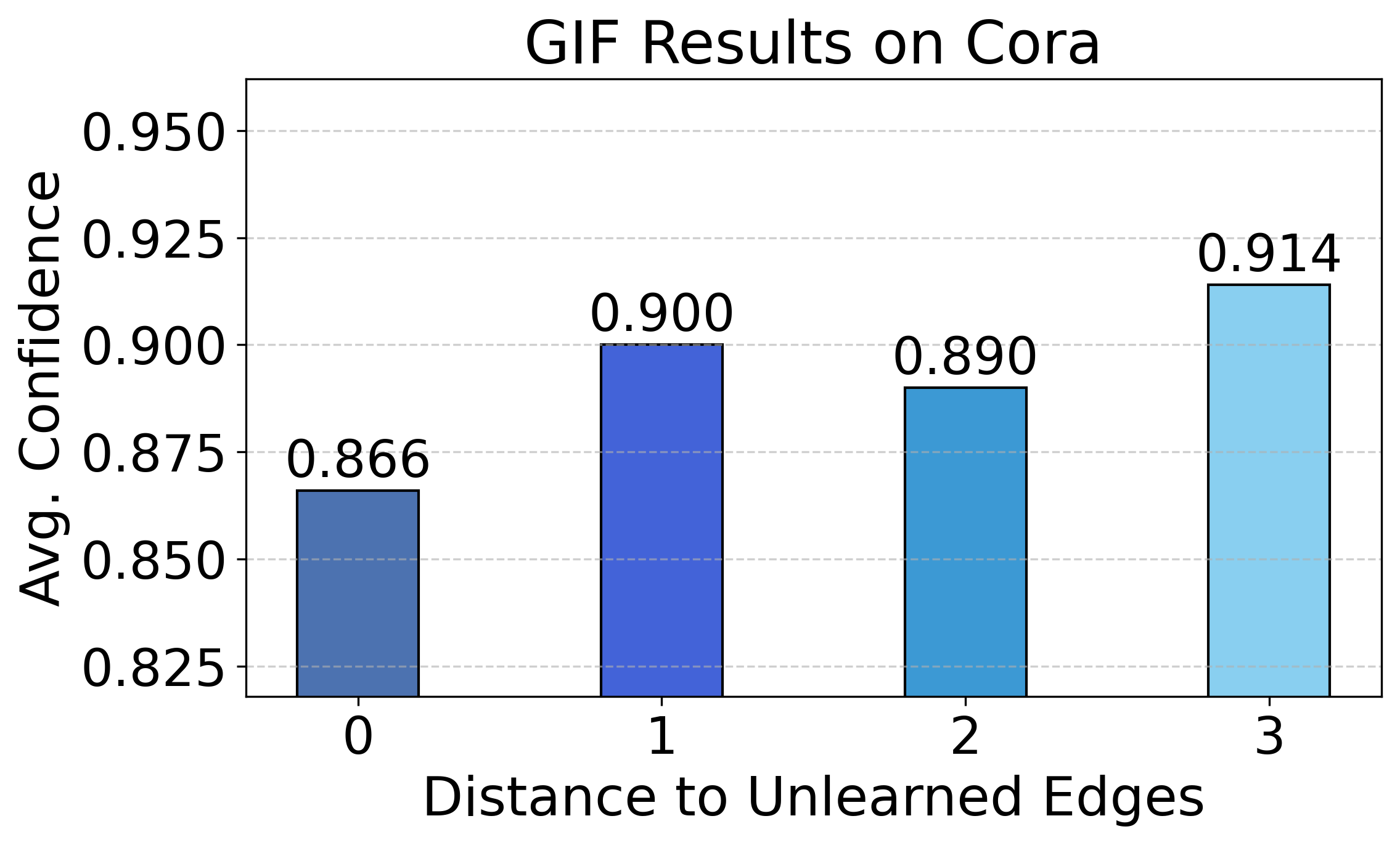}
    \caption{GIF results.}
  \end{subfigure}
  \hfill
  \begin{subfigure}[b]{0.32\linewidth}
    \centering
    \includegraphics[width=\linewidth]{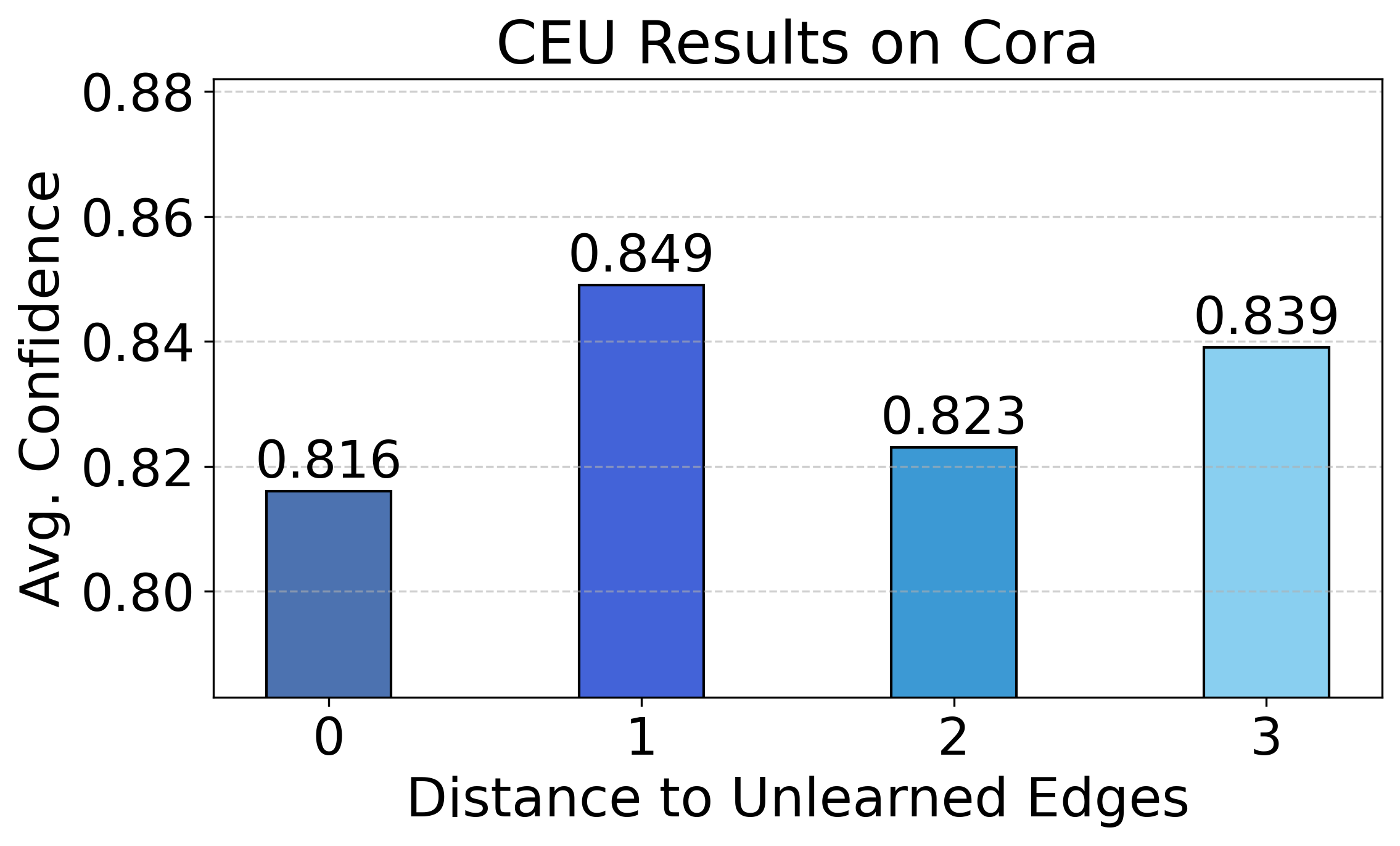}
    \caption{CEU results.}
  \end{subfigure}
  \hfill
  \begin{subfigure}[b]{0.32\linewidth}
    \centering
    \includegraphics[width=\linewidth]{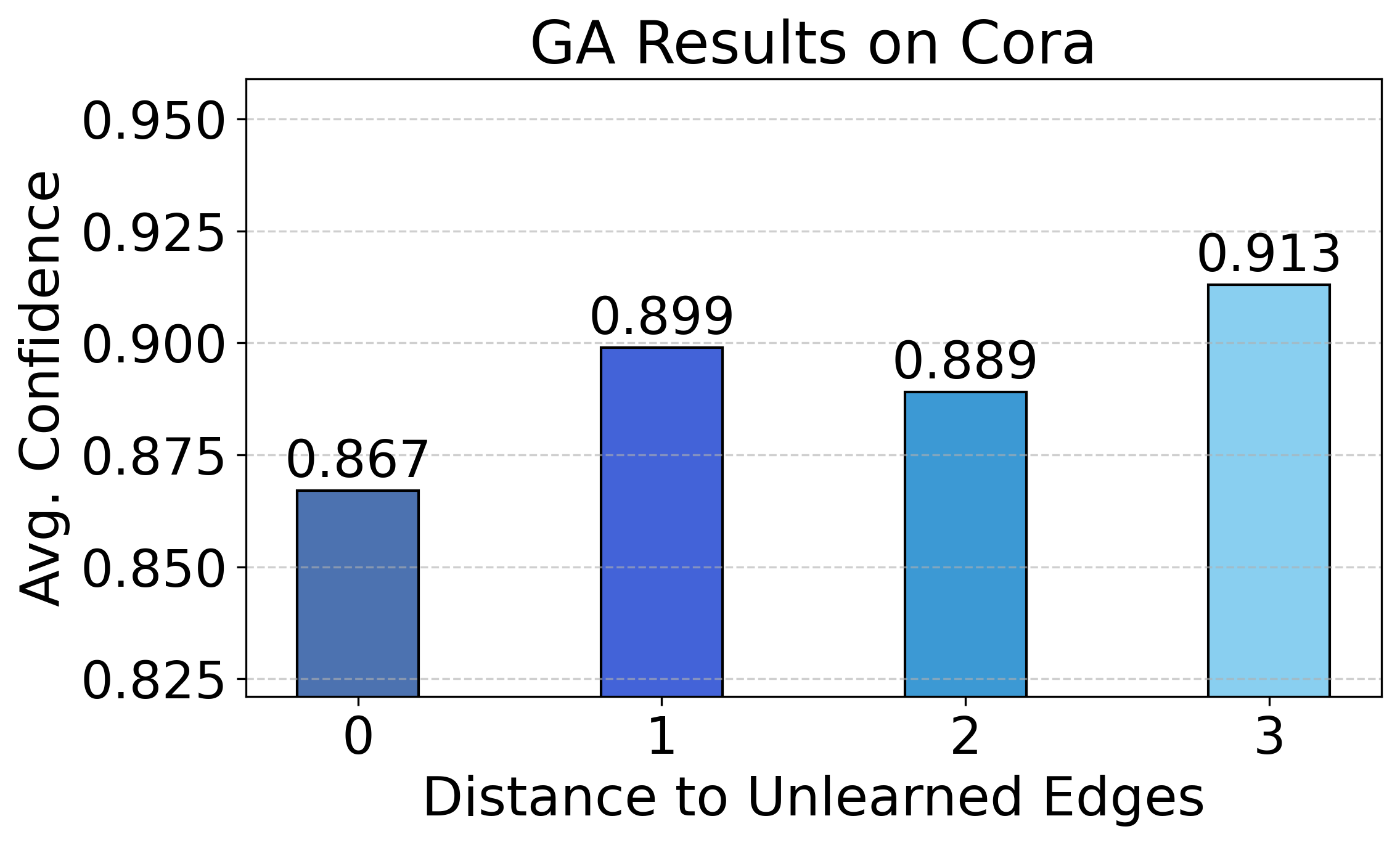}
    \caption{GA results.}
  \end{subfigure}
  \caption{\textbf{The relation between average model confidence and distance to unlearned edges on Cora.}}
  \label{fig:confidence_cora}
\end{figure*}

\begin{figure*}[!ht]
  \centering
  \begin{subfigure}[b]{0.32\linewidth}
    \centering
    \includegraphics[width=\linewidth]{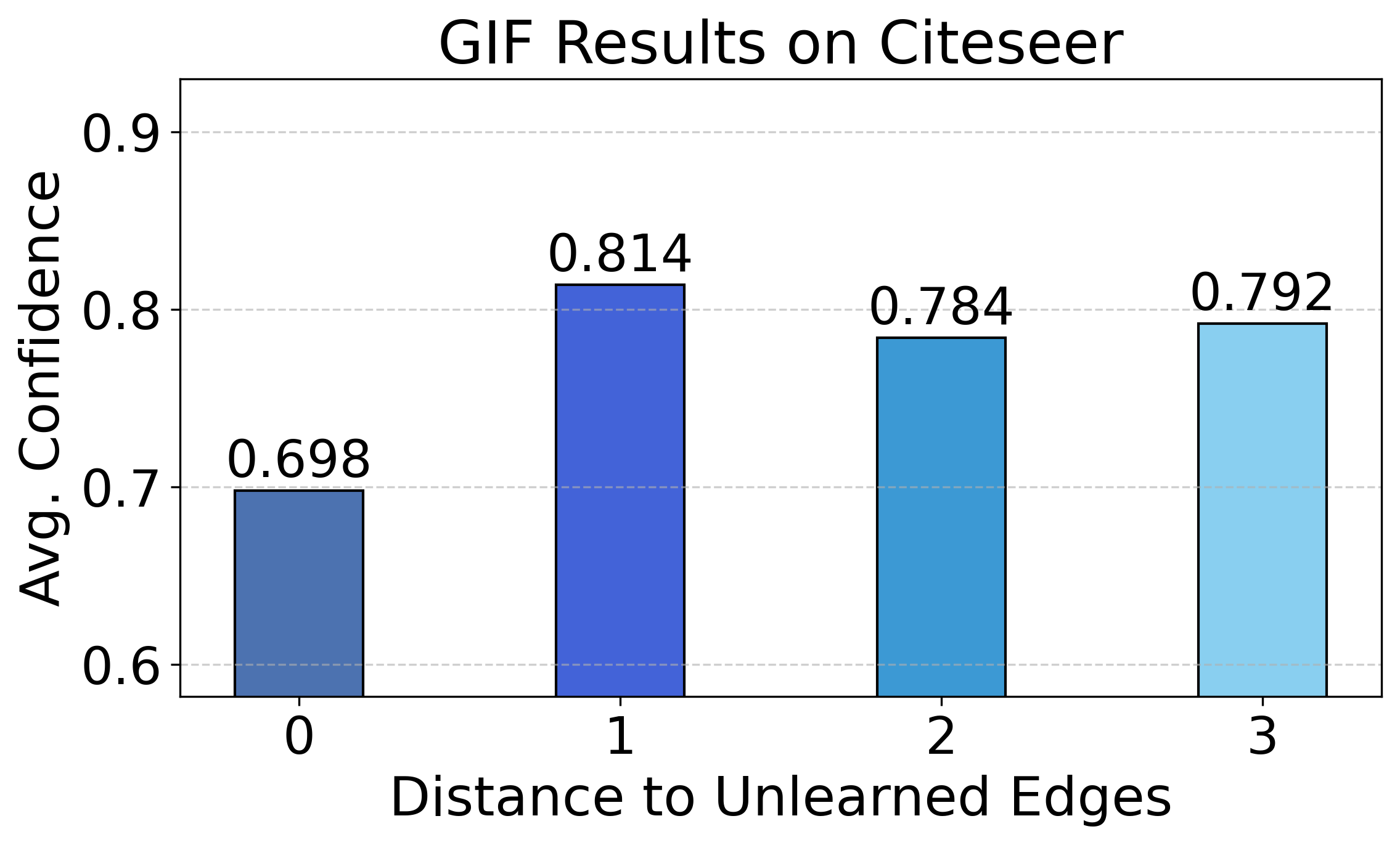}
    \caption{GIF results.}
  \end{subfigure}
  \hfill
  \begin{subfigure}[b]{0.32\linewidth}
    \centering
    \includegraphics[width=\linewidth]{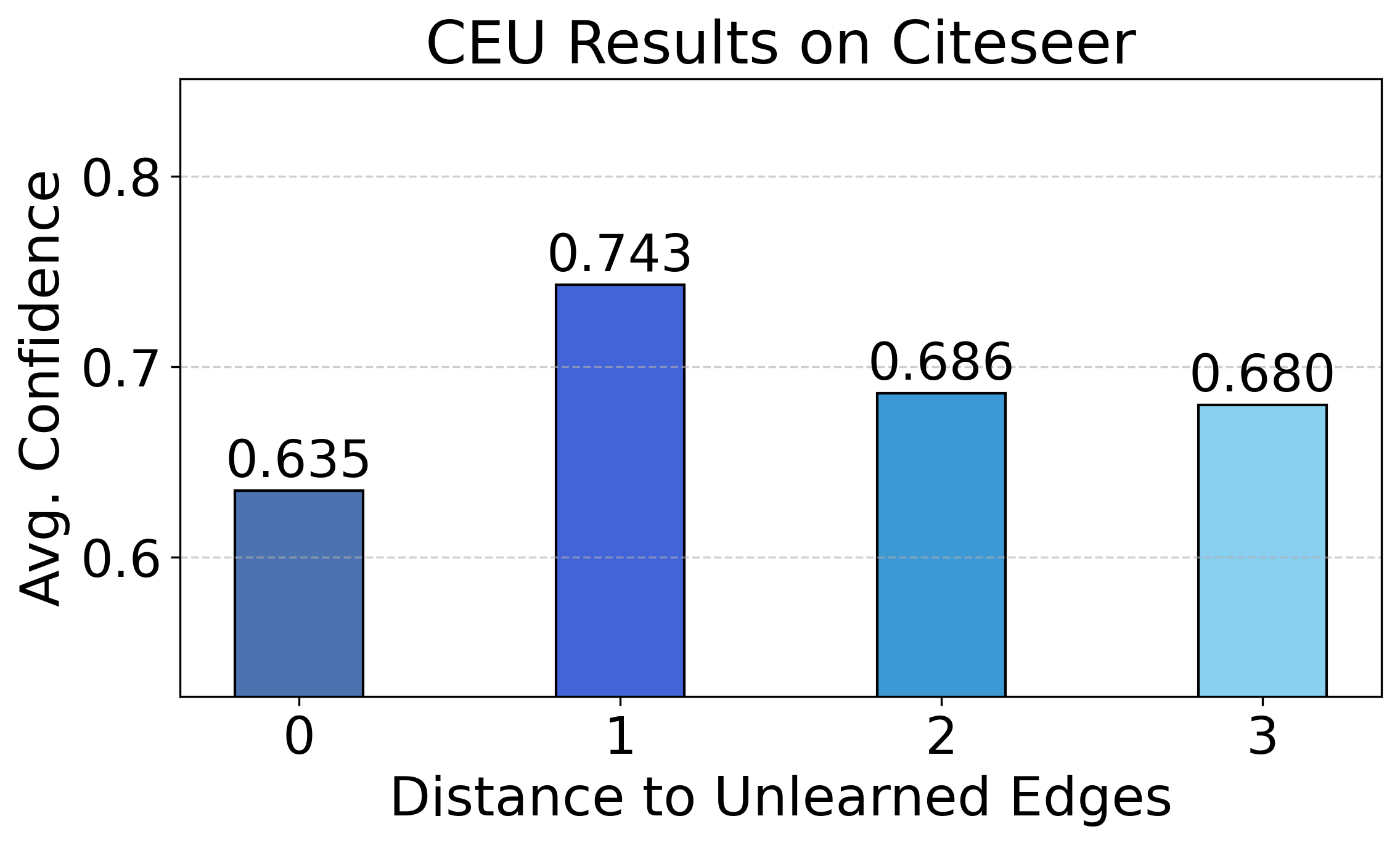}
    \caption{CEU results.}
  \end{subfigure}
  \hfill
  \begin{subfigure}[b]{0.32\linewidth}
    \centering
    \includegraphics[width=\linewidth]{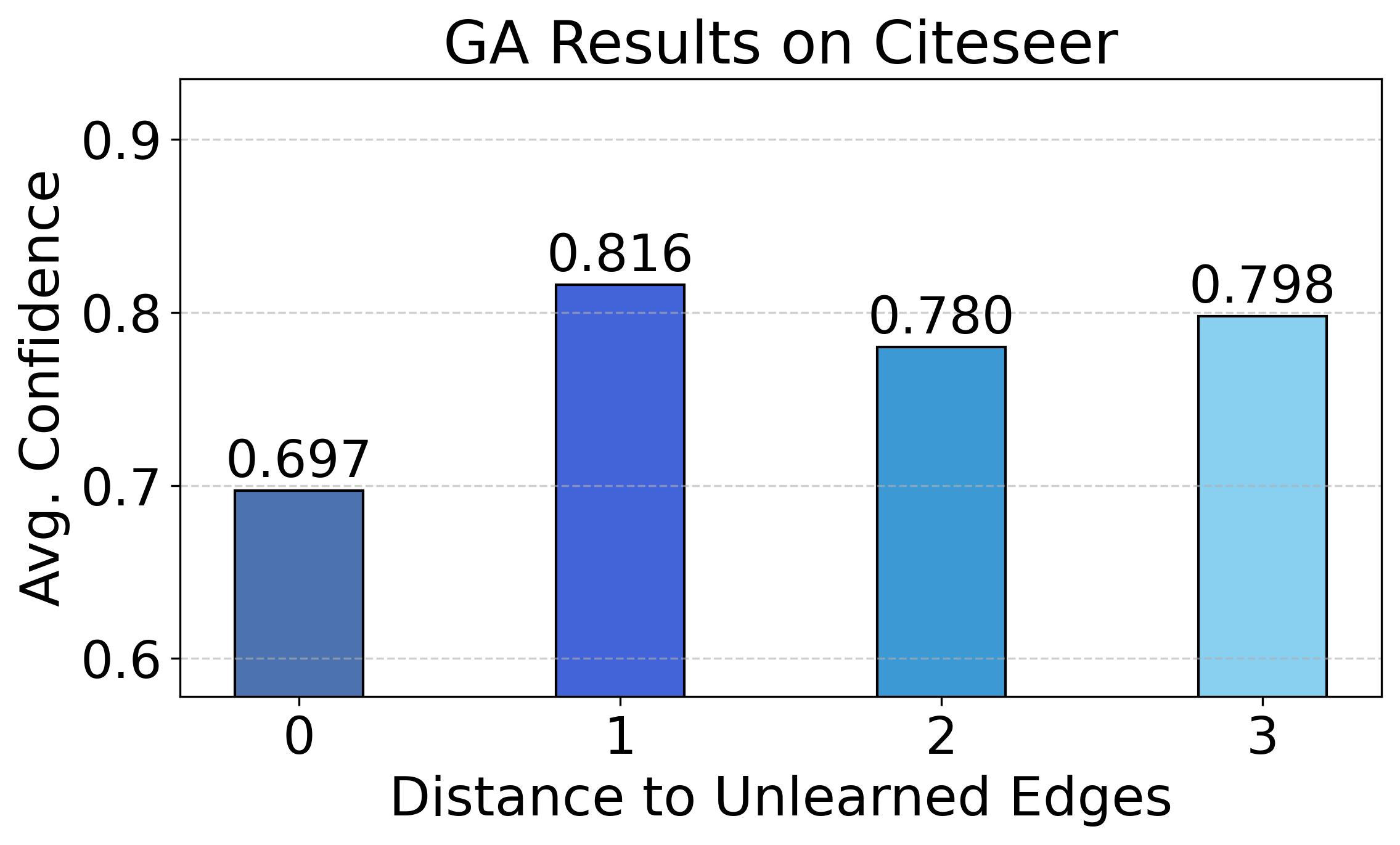}
    \caption{GA results.}
  \end{subfigure}
  \caption{\textbf{The relation between average model confidence and distance to unlearned edges on Citeseer.}}
  \label{fig:confidence_citeseer}
\end{figure*}

To demonstrate that the confidence of nodes connected to unlearned edges drops significantly compared to other nodes, we conduct a preliminary experiment. Specifically, we use GCN as the victim model. We randomly select 5\% of the edges as unlearned edges $\Delta\Ecal$ and apply three unlearning methods, GIF, CEU, and GA, to unlearn the trained GCN. The selected edges are also removed from the graph.

We then perform inference on all nodes using the unlearned GNN and compute the average confidence for nodes grouped by their shortest path distance to the endpoints of unlearned edges.

The results on two datasets are shown in Figure~\ref{fig:confidence_cora} and Figure~\ref{fig:confidence_citeseer}. We summarize the following observation:

For all datasets, nodes directly connected to unlearned edges (distance = 0) exhibit substantially lower confidence than their immediate neighbors (distance $\geq$ 1). 
This aligns with our theoretical analysis in Section~\ref{sec:design_motivation}, which shows that the influence of edge unlearning on the endpoint is significantly greater than on other nodes. 
This empirical finding suggests that the confidence gap between a node and its neighborhood is a strong indicator of its connection to an unlearned edge. Consequently, an attacker can adjust the similarity threshold accordingly to enhance membership inference accuracy.

\section{Experimental Settings}\label{sec:append_exp_settings}

{\bf Model Parameters.} All victim models use an embedding size of 16 and are trained for 100 epochs. For learning rates and weight decays, we follow the settings of GIF~\cite{wu2023gif}. All models are optimized using the Adam optimizer~\cite{kingma2015adam}. For all neural attack models, including StealLink~\cite{he2021stealing}, MIA-GNN~\cite{olatunji2021membership}, and TrendAttack, we follow the settings of StealLink~\cite{he2021stealing}, using a hidden dimension of 64 and a 2-layer MLP to encode features. All membership inference models are trained using learning rate = 0.01, weight decay = 0.0001, and the Adam optimizer~\cite{kingma2015adam}. They are trained with the binary cross-entropy loss.

\noindent{\bf Unlearning Method Settings.} We now supplement the missing details of our unlearning methods.
\begin{itemize}[leftmargin=*]
    \item \textbf{GIF}~\cite{wu2023gif}: The number of estimation iterations is 100 and the damping factor is 0. On the Cora and Citeseer datasets, the scale factor $\lambda$ is set to 500, following the official implementation. On the Pubmed and LastFM-Asia datasets, we search $\lambda$ in $\{10^1, 10^2, 10^3, 10^4, \ldots\}$, following the original search space. This ensures that the unlearned models maintain utility and produce meaningful outputs. We select the smallest $\lambda$ that results in a meaningful model and find that $\lambda=10^3$ works well for Pubmed and $\lambda=10^4$ for LastFM-Asia.
    \item \textbf{CEU}~\cite{wu2023certified}: For a fair comparison, we use the same parameter search space as GIF for all influence-function-related parameters, including iterations, damping factor, and $\lambda$. For the noise variance, we search in $\{0.1, 0.5, 0.01, 0.005, 0.001, 0.0005, 0.0001\}$ to ensure meaningful unlearning performance.
    \item \textbf{GA}~\cite{wu2024graphguard}: We follow the official GraphGuard settings and use 1 gradient ascent epoch. We tune the unlearning magnitude $\alpha$ in $\{1, 0.5, 0.1, 0.05, 0.01, 0.005, 0.001\}$ and find that $\alpha = 0.5$ gives the best results.
\end{itemize}

\noindent{\bf Baselines.} Details for our baselines are as follows:
\begin{itemize}[leftmargin=*]
    \item \textbf{GraphSAGE}~\cite{hamilton2017inductive}: We use a simple link prediction model trained on the shadow graph to perform link prediction on the unlearned graph. The embedding size is 64, with 2 layers. The negative-to-positive link ratio is 1:1.
    \item \textbf{NCN}~\cite{wang2024neural}: This is a state-of-the-art link prediction method. We follow the official settings from their code repository. We search across the provided backbones: GCN~\cite{kipf2017semi}, GIN~\cite{xu2019powerful}, and SAGE~\cite{wang2023inductive}. The embedding dimension is set to the default value of 256. The batch size is 16, the learning rate is 0.01, and the number of training epochs is 100.
    \item \textbf{MIA-GNN}~\cite{olatunji2021membership}: Since no official implementation is available, we follow the settings described in the paper. We re-implement the model using a 2-layer MLP feature extractor with a hidden dimension of 64. We use another 2-layer MLP with a dimension of 16 as the link predictor. Training uses binary cross-entropy loss and the same setup as StealLink, ensuring a fair comparison.
    \item \textbf{StealLink}~\cite{he2021stealing}: This is a representative membership inference attack. We use its strongest variant, Attack-7, which has access to the shadow dataset and partial features and connectivity from the target dataset. We follow the official codebase and use all 8 distance metrics (e.g., cosine, Euclidean). For the probability metric, we search among entropy, KL divergence, and JS divergence, and report the best result. The reference model is a 2-layer MLP with an embedding size of 16.
    \item \textbf{GroupAttack}~\cite{zhang2023demystifying}: This is a recent membership inference attack that relies solely on hard labels and thresholding. We search the threshold hyperparameter $\alpha$ on the shadow dataset in the range $[0, 1]$ with a step size of 0.05.
\end{itemize}

\noindent{\bf Reproducibility.}
All experiments are conducted using Python 3.12.2, PyTorch 2.3.1+cu121, and PyTorch Geometric 2.5.3. The experiments run on a single NVIDIA RTX A6000 GPU. The server used has 64 CPUs, and the type is AMD EPYC 7282 16-Core CPU. All source code is provided in the supplementary materials.
\section{Additional Experiments}\label{sec:more_experiments}

\begin{table*}[!ht]
\caption{\textbf{Main Comparison Results on Cora and Citeseer (with variance)}. We present the AUC scores for attack methods across different edge groups, now including mean $\pm$ variance. The best results are highlighted in \textbf{bold}, while the second-best results are \underline{underlined}.}
\label{tab:main_with_var_cora_citeseer}
\centering
\vskip -1em
\resizebox{\textwidth}{!}{
\begin{tabular}{llccc|ccc}
\toprule
\multirow{2}{*}{\textbf{Unlearn method}} & \multirow{2}{*}{\textbf{Attack}}
  & \multicolumn{3}{c}{\textbf{Cora}}
  & \multicolumn{3}{c}{\textbf{Citeseer}} \\
\cmidrule(lr){3-5} \cmidrule(lr){6-8}
 & & Unlearned & Original & All
   & Unlearned & Original & All \\
\midrule
\textbf{GIF} 
 & GraphSAGE    & $0.5356 \pm 0.0124$ & $0.5484 \pm 0.0180$ & $0.5420 \pm 0.0129$  
                & $0.5275 \pm 0.0452$ & $0.5216 \pm 0.0535$ & $0.5246 \pm 0.0474$ \\
 & NCN          & $0.7403 \pm 0.0309$ & $0.7405 \pm 0.0222$ & $0.7404 \pm 0.0237$  
                & $0.6750 \pm 0.0881$ & $0.6872 \pm 0.0799$ & $0.6811 \pm 0.0816$ \\
 & MIA-GNN      & $0.7547 \pm 0.0809$ & $0.7916 \pm 0.0939$ & $0.7732 \pm 0.0865$  
                & $0.7802 \pm 0.0251$ & $0.8245 \pm 0.0230$ & $0.8023 \pm 0.0210$ \\
 & StealLink    & $0.7841 \pm 0.0357$ & $0.8289 \pm 0.0576$ & $0.8065 \pm 0.0395$  
                & $0.7369 \pm 0.0463$ & $\underline{0.8404 \pm 0.0360}$ & $0.7887 \pm 0.0373$ \\
 & GroupAttack  & $0.7982 \pm 0.0072$ & $0.8053 \pm 0.0190$ & $0.8018 \pm 0.0125$  
                & $0.7771 \pm 0.0092$ & $0.7618 \pm 0.0069$ & $0.7695 \pm 0.0051$ \\
 & \textbf{TrendAttack-MIA} 
                & $\underline{0.8240 \pm 0.0209}$ & $\underline{0.8448 \pm 0.0185}$ & $\underline{0.8344 \pm 0.0189}$  
                & $\underline{0.8069 \pm 0.0185}$ & $0.8078 \pm 0.0284$ & $\underline{0.8073 \pm 0.0228}$ \\
 & \textbf{TrendAttack-SL}  
                & $\mathbf{0.8309 \pm 0.0250}$ & $\mathbf{0.8527 \pm 0.0140}$ & $\mathbf{0.8418 \pm 0.0188}$  
                & $\mathbf{0.8410 \pm 0.0169}$ & $\mathbf{0.8430 \pm 0.0276}$ & $\mathbf{0.8420 \pm 0.0207}$ \\
\midrule
\textbf{CEU} 
 & GraphSAGE    & $0.5356 \pm 0.0124$ & $0.5484 \pm 0.0180$ & $0.5420 \pm 0.0129$  
                & $0.5275 \pm 0.0452$ & $0.5216 \pm 0.0535$ & $0.5246 \pm 0.0474$ \\
 & NCN          & $0.7403 \pm 0.0309$ & $0.7405 \pm 0.0222$ & $0.7404 \pm 0.0237$  
                & $0.6750 \pm 0.0881$ & $0.6872 \pm 0.0799$ & $0.6811 \pm 0.0816$ \\
 & MIA-GNN      & $0.7458 \pm 0.0761$ & $0.7810 \pm 0.0928$ & $0.7634 \pm 0.0840$  
                & $0.7718 \pm 0.0262$ & $0.8248 \pm 0.0264$ & $0.7983 \pm 0.0219$ \\
 & StealLink    & $0.7901 \pm 0.0224$ & $\underline{0.8486 \pm 0.0099}$ & $0.8193 \pm 0.0068$  
                & $0.7643 \pm 0.0242$ & $\mathbf{0.8450 \pm 0.0257}$ & $\underline{0.8046 \pm 0.0217}$ \\
 & GroupAttack  & $0.7941 \pm 0.0086$ & $0.7976 \pm 0.0137$ & $0.7958 \pm 0.0105$  
                & $0.7557 \pm 0.0092$ & $0.7458 \pm 0.0131$ & $0.7508 \pm 0.0107$ \\
 & \textbf{TrendAttack-MIA} 
                & $\underline{0.8194 \pm 0.0170}$ & $0.8333 \pm 0.0213$ & $\underline{0.8263 \pm 0.0178}$  
                & $\underline{0.7933 \pm 0.0206}$ & $0.8041 \pm 0.0261$ & $0.7987 \pm 0.0226$ \\
 & \textbf{TrendAttack-SL}  
                & $\mathbf{0.8467 \pm 0.0229}$ & $\mathbf{0.8612 \pm 0.0113}$ & $\mathbf{0.8539 \pm 0.0149}$  
                & $\mathbf{0.8514 \pm 0.0214}$ & $\underline{0.8400 \pm 0.0216}$ & $\mathbf{0.8457 \pm 0.0208}$ \\
\midrule
\textbf{GA}  
 & GraphSAGE    & $0.5356 \pm 0.0124$ & $0.5484 \pm 0.0180$ & $0.5420 \pm 0.0129$  
                & $0.5275 \pm 0.0452$ & $0.5216 \pm 0.0535$ & $0.5246 \pm 0.0474$ \\
 & NCN          & $0.7403 \pm 0.0309$ & $0.7405 \pm 0.0222$ & $0.7404 \pm 0.0237$  
                & $0.6750 \pm 0.0881$ & $0.6872 \pm 0.0799$ & $0.6811 \pm 0.0816$ \\
 & MIA-GNN      & $0.7676 \pm 0.0584$ & $0.8068 \pm 0.0489$ & $0.7872 \pm 0.0531$  
                & $0.7798 \pm 0.0146$ & $0.8353 \pm 0.0175$ & $0.8076 \pm 0.0117$ \\
 & StealLink    & $0.7862 \pm 0.0402$ & $0.8301 \pm 0.0667$ & $0.8082 \pm 0.0475$  
                & $0.7479 \pm 0.0512$ & $\underline{0.8431 \pm 0.0305}$ & $0.7955 \pm 0.0332$ \\
 & GroupAttack  & $0.7945 \pm 0.0133$ & $0.8042 \pm 0.0123$ & $0.7993 \pm 0.0122$  
                & $0.7662 \pm 0.0112$ & $0.7563 \pm 0.0080$ & $0.7613 \pm 0.0086$ \\
 & \textbf{TrendAttack-MIA} 
                & $\underline{0.8193 \pm 0.0276}$ & $\mathbf{0.8397 \pm 0.0219}$ & $\underline{0.8295 \pm 0.0244}$  
                & $\underline{0.8080 \pm 0.0135}$ & $0.8249 \pm 0.0331$ & $\underline{0.8165 \pm 0.0227}$ \\
 & \textbf{TrendAttack-SL}  
                & $\mathbf{0.8270 \pm 0.0307}$ & $\underline{0.8382 \pm 0.0308}$ & $\mathbf{0.8326 \pm 0.0287}$  
                & $\mathbf{0.8628 \pm 0.0131}$ & $\mathbf{0.8614 \pm 0.0298}$ & $\mathbf{0.8621 \pm 0.0209}$ \\
\bottomrule
\end{tabular}
}
\vskip -0.5em
\end{table*}

\begin{table*}[!ht]
\centering
\caption{\textbf{Main Comparison Results on Pubmed and LastFM-Asia (with variance)}. We present the AUC scores for attack methods across different edge groups, now including mean $\pm$ variance. The best results are highlighted in \textbf{bold}, while the second-best results are \underline{underlined}.}
\label{tab:main_with_var_pubmed_lastfm}
\vskip -1em
\resizebox{\textwidth}{!}{
\begin{tabular}{llccc|ccc}
\toprule
\multirow{2}{*}{\textbf{Unlearn method}} & \multirow{2}{*}{\textbf{Attack}}
  & \multicolumn{3}{c}{\textbf{Pubmed}}
  & \multicolumn{3}{c}{\textbf{LastFM-Asia}} \\
\cmidrule(lr){3-5} \cmidrule(lr){6-8}
 & & Unlearned & Original & All
   & Unlearned & Original & All \\
\midrule
\textbf{GIF} 
 & GraphSAGE    & $0.6503 \pm 0.0114$ & $0.6457 \pm 0.0134$ & $0.6480 \pm 0.0118$  
                & $0.6914 \pm 0.0620$ & $0.6853 \pm 0.0613$ & $0.6884 \pm 0.0611$ \\
 & NCN          & $0.6661 \pm 0.0321$ & $0.6718 \pm 0.0314$ & $0.6690 \pm 0.0312$  
                & $0.7283 \pm 0.0177$ & $0.7273 \pm 0.0120$ & $0.7278 \pm 0.0141$ \\
 & MIA-GNN      & $0.7028 \pm 0.0069$ & $0.7902 \pm 0.0041$ & $0.7465 \pm 0.0044$  
                & $0.5955 \pm 0.0498$ & $0.5744 \pm 0.0751$ & $0.5850 \pm 0.0614$ \\
 & StealLink    & $0.8248 \pm 0.0098$ & $0.8964 \pm 0.0027$ & $0.8606 \pm 0.0052$  
                & $\underline{0.8472 \pm 0.0099}$ & $\underline{0.9037 \pm 0.0097}$ & $\underline{0.8755 \pm 0.0089}$ \\
 & GroupAttack  & $0.6497 \pm 0.0021$ & $0.6554 \pm 0.0049$ & $0.6525 \pm 0.0033$  
                & $0.7858 \pm 0.0044$ & $0.7850 \pm 0.0033$ & $0.7854 \pm 0.0027$ \\
 & \textbf{TrendAttack-MIA} 
                & $\underline{0.8950 \pm 0.0054}$ & $\underline{0.9171 \pm 0.0039}$ & $\underline{0.9060 \pm 0.0032}$  
                & $0.7795 \pm 0.0690$ & $0.7649 \pm 0.0944$ & $0.7722 \pm 0.0814$ \\
 & \textbf{TrendAttack-SL}  
                & $\mathbf{0.9524 \pm 0.0026}$ & $\mathbf{0.9535 \pm 0.0025}$ & $\mathbf{0.9529 \pm 0.0024}$  
                & $\mathbf{0.9078 \pm 0.0069}$ & $\mathbf{0.9134 \pm 0.0039}$ & $\mathbf{0.9106 \pm 0.0050}$ \\
\midrule
\textbf{CEU} 
 & GraphSAGE    & $0.6503 \pm 0.0114$ & $0.6457 \pm 0.0134$ & $0.6480 \pm 0.0118$  
                & $0.6914 \pm 0.0620$ & $0.6853 \pm 0.0613$ & $0.6884 \pm 0.0611$ \\
 & NCN          & $0.6661 \pm 0.0321$ & $0.6718 \pm 0.0314$ & $0.6690 \pm 0.0312$  
                & $0.7283 \pm 0.0177$ & $0.7273 \pm 0.0120$ & $0.7278 \pm 0.0141$ \\
 & MIA-GNN      & $0.6626 \pm 0.0237$ & $0.6561 \pm 0.0256$ & $0.6593 \pm 0.0243$  
                & $0.6004 \pm 0.0363$ & $0.5811 \pm 0.0568$ & $0.5908 \pm 0.0459$ \\
 & StealLink    & $0.8467 \pm 0.0168$ & $0.9088 \pm 0.0102$ & $0.8777 \pm 0.0134$  
                & $\underline{0.8416 \pm 0.0163}$ & $\underline{0.9021 \pm 0.0084}$ & $\underline{0.8719 \pm 0.0114}$ \\
 & GroupAttack  & $0.6388 \pm 0.0052$ & $0.6430 \pm 0.0054$ & $0.6409 \pm 0.0053$  
                & $0.7845 \pm 0.0038$ & $0.7817 \pm 0.0010$ & $0.7831 \pm 0.0023$ \\
 & \textbf{TrendAttack-MIA} 
                & $\underline{0.8982 \pm 0.0038}$ & $\underline{0.9184 \pm 0.0025}$ & $\underline{0.9083 \pm 0.0018}$  
                & $0.7676 \pm 0.0722$ & $0.7576 \pm 0.0986$ & $0.7626 \pm 0.0850$ \\
 & \textbf{TrendAttack-SL}  
                & $\mathbf{0.9550 \pm 0.0032}$ & $\mathbf{0.9579 \pm 0.0028}$ & $\mathbf{0.9565 \pm 0.0028}$  
                & $\mathbf{0.9037 \pm 0.0062}$ & $\mathbf{0.9088 \pm 0.0020}$ & $\mathbf{0.9062 \pm 0.0040}$ \\
\midrule
\textbf{GA}  
 & GraphSAGE    & $0.6503 \pm 0.0114$ & $0.6457 \pm 0.0134$ & $0.6480 \pm 0.0118$  
                & $0.6914 \pm 0.0620$ & $0.6853 \pm 0.0613$ & $0.6884 \pm 0.0611$ \\
 & NCN          & $0.6661 \pm 0.0321$ & $0.6718 \pm 0.0314$ & $0.6690 \pm 0.0312$  
                & $0.7283 \pm 0.0177$ & $0.7273 \pm 0.0120$ & $0.7278 \pm 0.0141$ \\
 & MIA-GNN      & $0.7242 \pm 0.0099$ & $0.8039 \pm 0.0128$ & $0.7641 \pm 0.0112$  
                & $0.6200 \pm 0.0636$ & $0.6057 \pm 0.0884$ & $0.6129 \pm 0.0756$ \\
 & StealLink    & $0.8203 \pm 0.0128$ & $0.8898 \pm 0.0056$ & $0.8550 \pm 0.0080$  
                & $\underline{0.8342 \pm 0.0182}$ & $\underline{0.8947 \pm 0.0040}$ & $\underline{0.8644 \pm 0.0106}$ \\
 & GroupAttack  & $0.6458 \pm 0.0076$ & $0.6493 \pm 0.0111$ & $0.6475 \pm 0.0093$  
                & $0.7746 \pm 0.0069$ & $0.7760 \pm 0.0031$ & $0.7753 \pm 0.0048$ \\
 & \textbf{TrendAttack-MIA} 
                & $\underline{0.8932 \pm 0.0052}$ & $\underline{0.9158 \pm 0.0048}$ & $\underline{0.9045 \pm 0.0037}$  
                & $0.7255 \pm 0.0715$ & $0.7099 \pm 0.0778$ & $0.7177 \pm 0.0742$ \\
 & \textbf{TrendAttack-SL}  
                & $\mathbf{0.9531 \pm 0.0027}$ & $\mathbf{0.9537 \pm 0.0034}$ & $\mathbf{0.9534 \pm 0.0029}$  
                & $\mathbf{0.9041 \pm 0.0054}$ & $\mathbf{0.9119 \pm 0.0034}$ & $\mathbf{0.9080 \pm 0.0040}$ \\
\bottomrule
\end{tabular}
}
\vskip -0.5em
\end{table*}

{\bf Variance of Comparison Results.} Due to space limitations, the comparison results in Table~\ref{tab:main} in Section~\ref{sec:experiments} do not include the standard variance from five repeated experiments. We now supplement all variance results in Table~\ref{tab:main_with_var_cora_citeseer} and Table~\ref{tab:main_with_var_pubmed_lastfm}. We make the following observations regarding the stability of our proposed TrendAttack: Compared with the no-trend-feature counterparts MIA-GNN and StealLink, our TrendAttack-MIA and TrendAttack-SL exhibit smaller variances, indicating better stability. This highlights the effectiveness of the trend features, which not only improve attack performance on both edge groups but also reduce variance, enhancing model stability.

\begin{figure*}[!ht]
  \centering
  \begin{subfigure}[b]{0.27\textwidth}
    \centering
    \includegraphics[width=\linewidth]{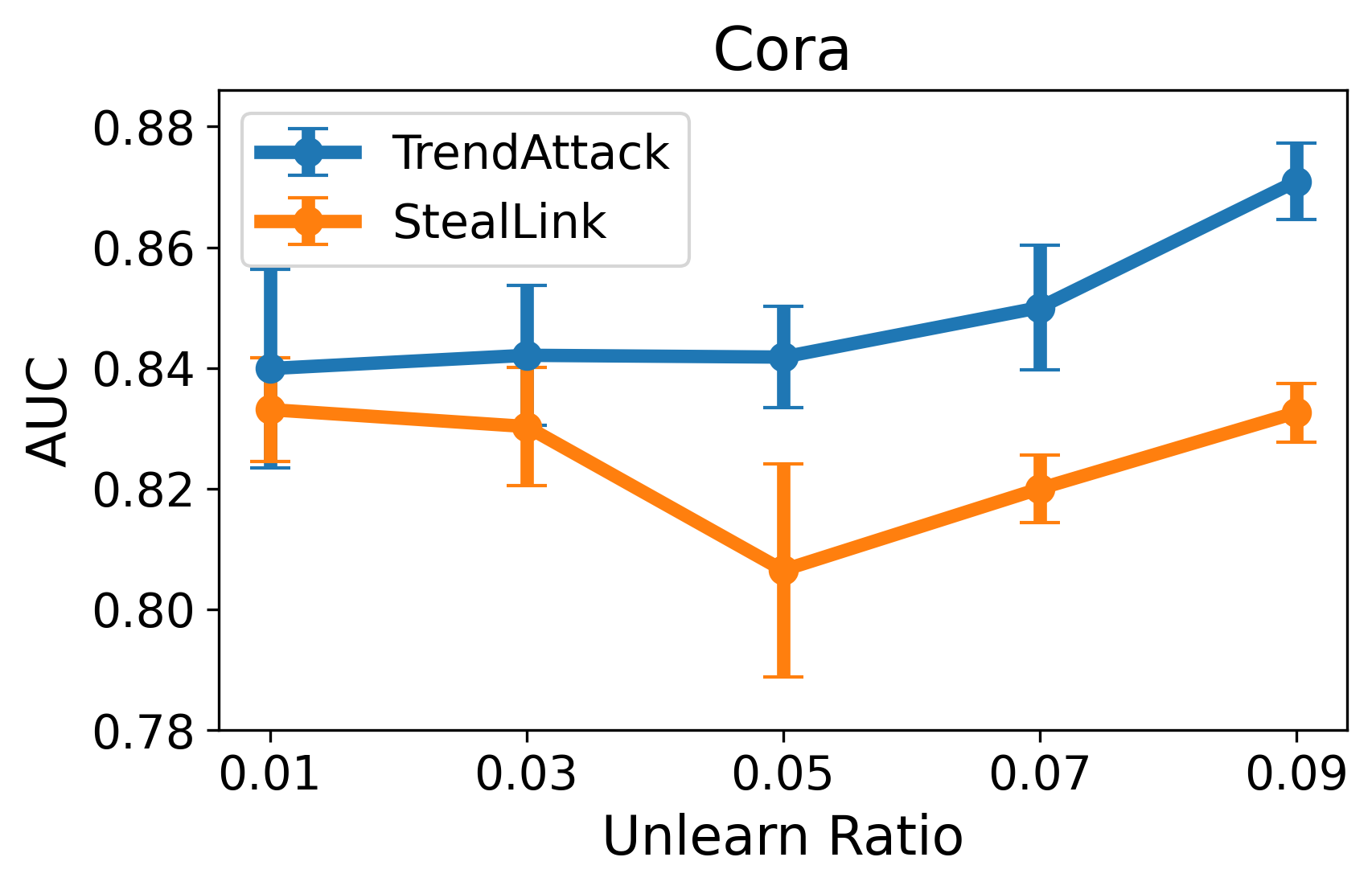}
    \caption{Cora results.}
    \label{fig:ratio_cora}
  \end{subfigure}\hfill
  \begin{subfigure}[b]{0.27\textwidth}
    \centering
    \includegraphics[width=\linewidth]{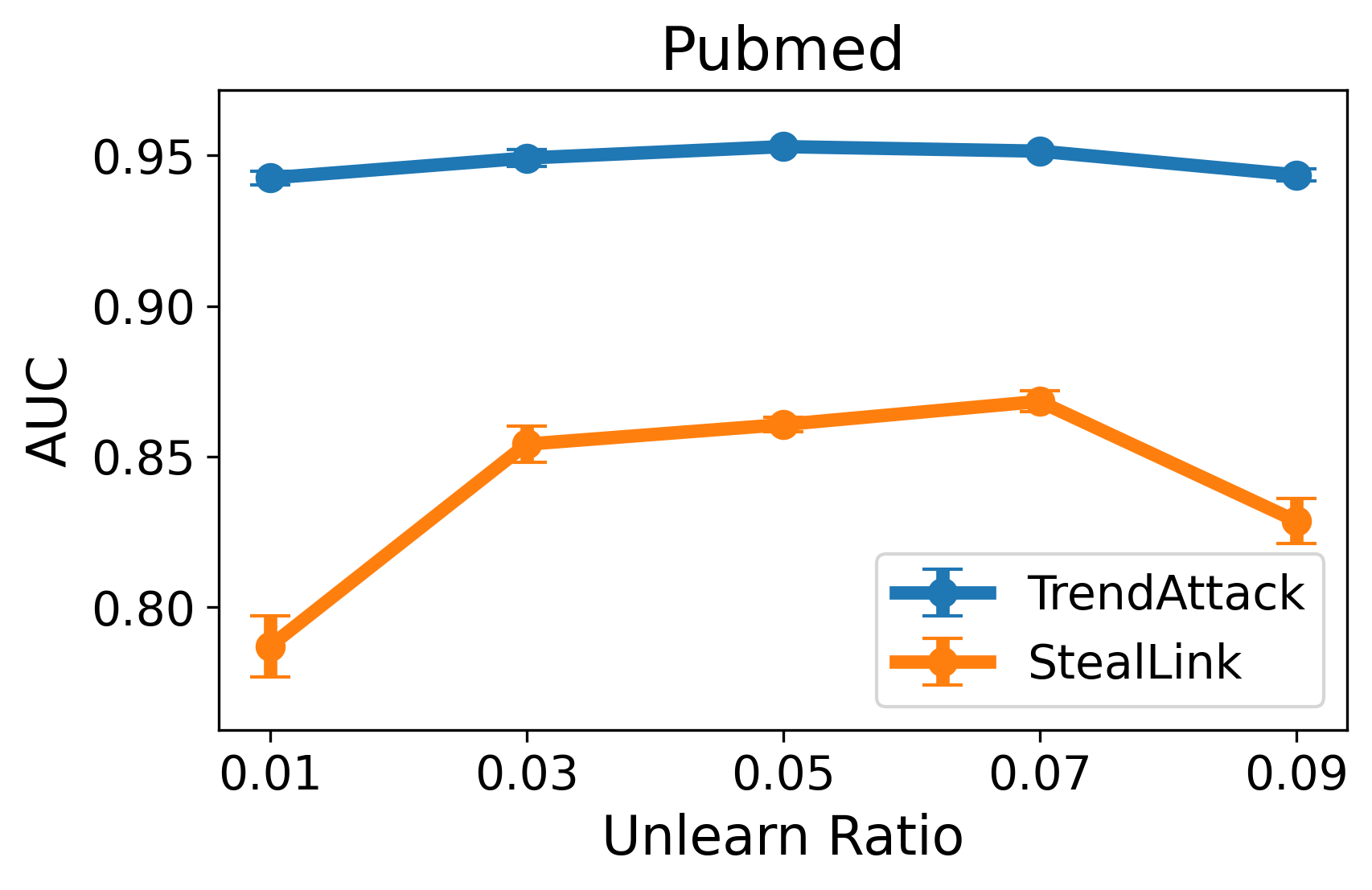}
    \caption{Pubmed results.}
    \label{fig:ratio_pubmed}
  \end{subfigure}\hfill
  \begin{subfigure}[b]{0.27\textwidth}
    \centering
    \includegraphics[width=\linewidth]{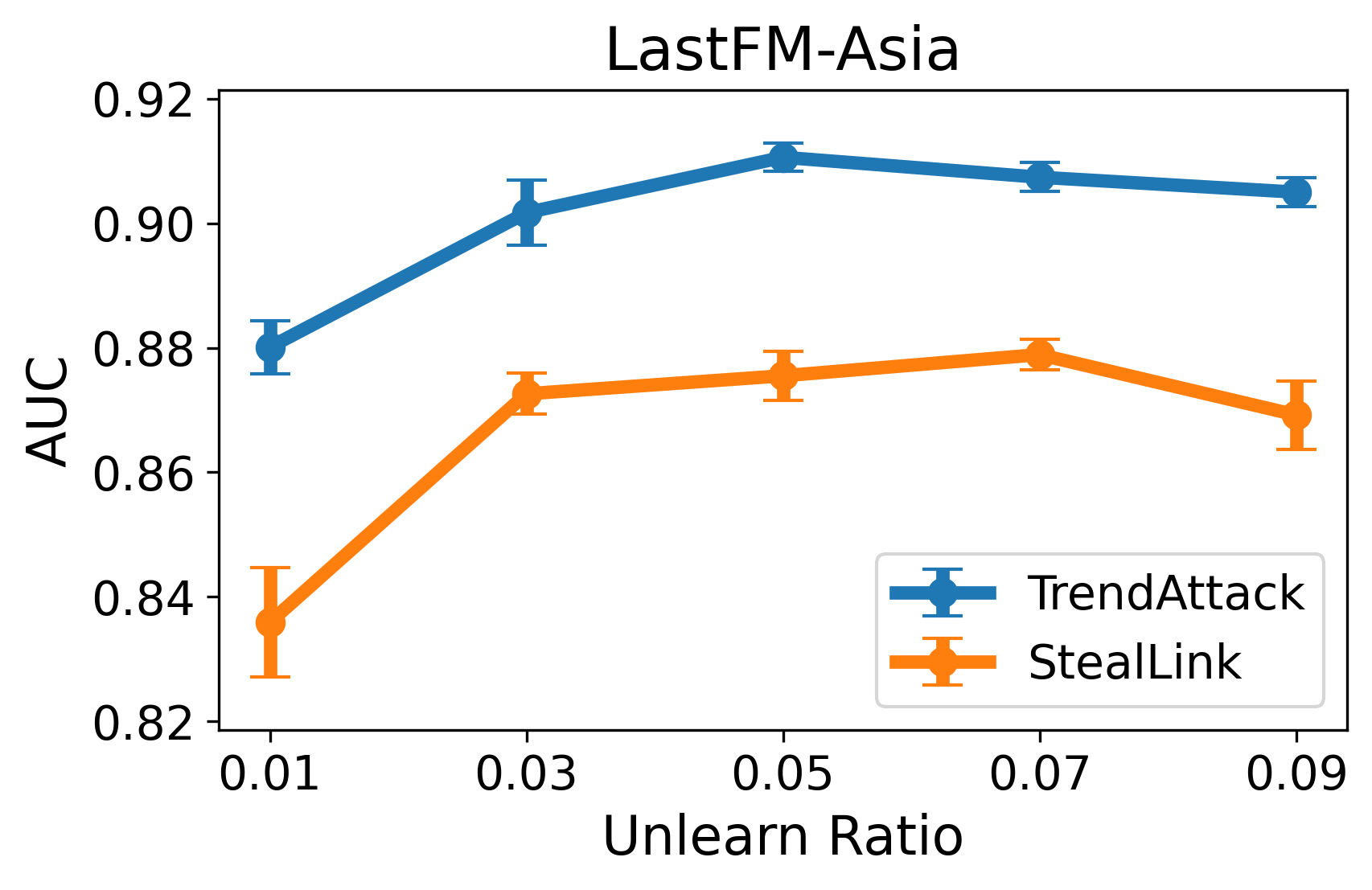}
    \caption{LastFM-Asia results.}
    \label{fig:ratio_lastfm}
  \end{subfigure}
  \vskip -1em
  \caption{\textbf{Ablation study on the impact of unlearning ratio}.}
  \label{fig:ablation_ratio}
\end{figure*}

\begin{figure*}[!ht]
    \centering
    \begin{subfigure}{0.3\textwidth}
        \centering
        \includegraphics[width=\linewidth]{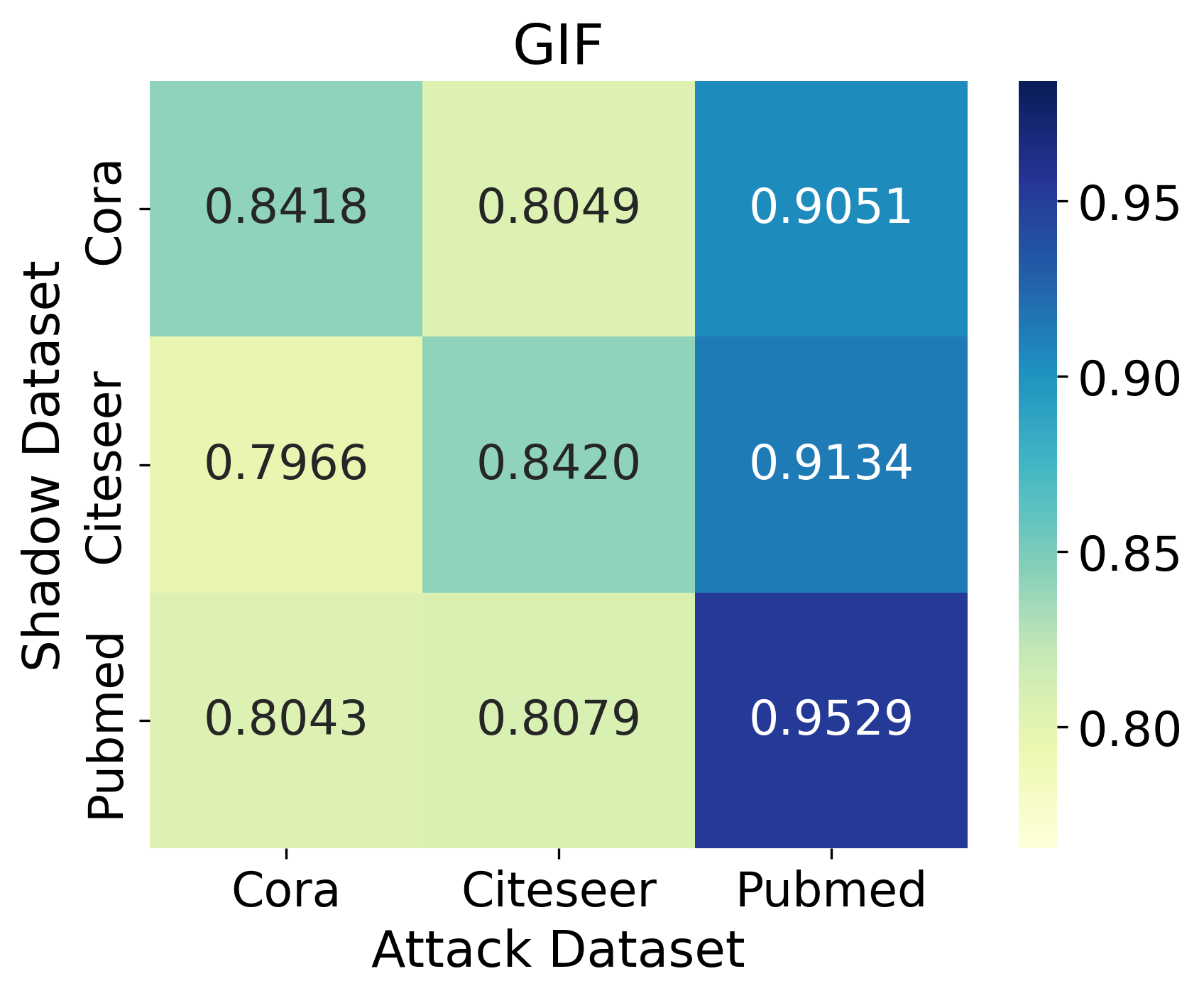}
        \caption{GIF results.}
    \end{subfigure}
    \hfill
    \begin{subfigure}{0.3\textwidth}
        \centering
        \includegraphics[width=\linewidth]{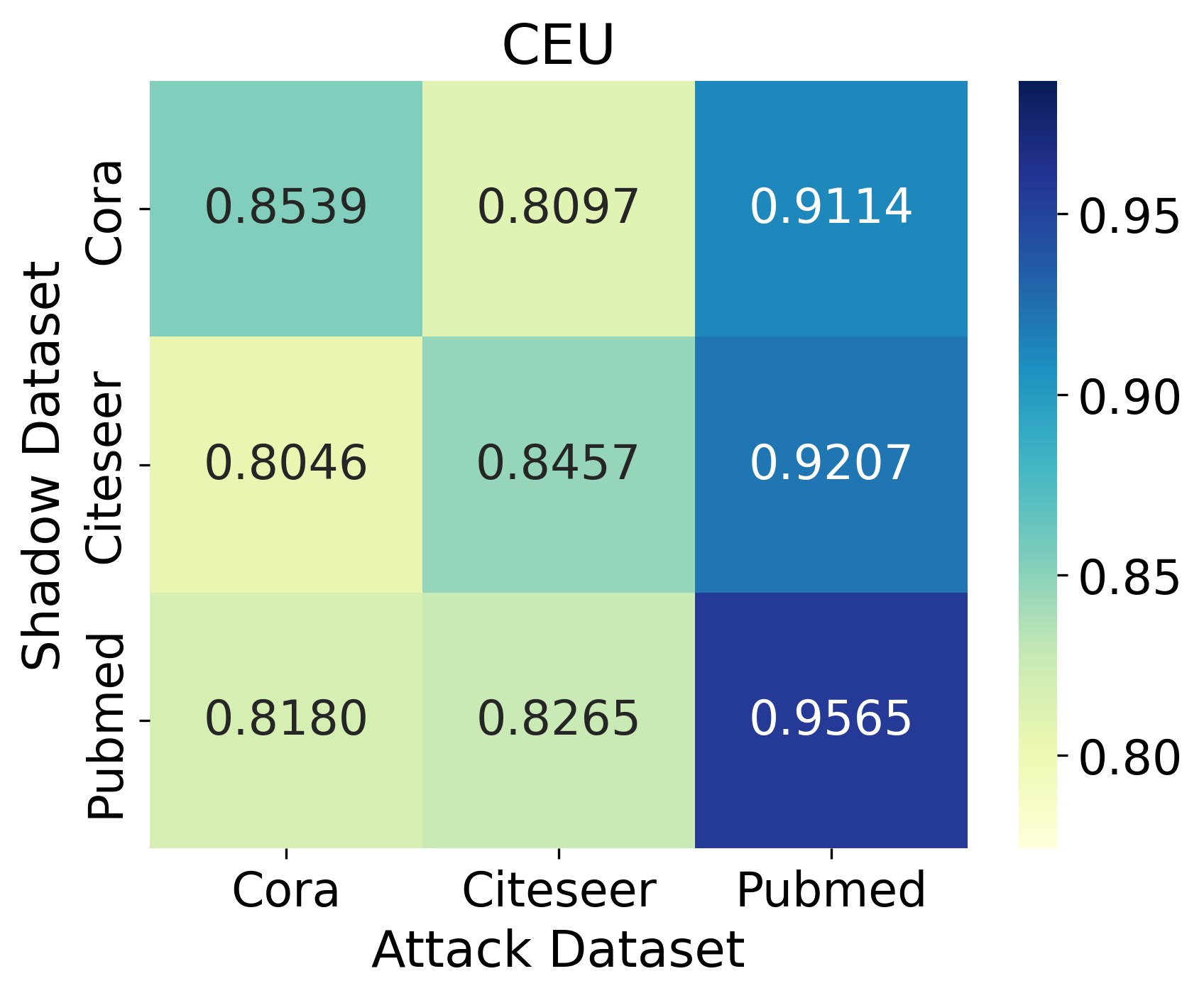}
        \caption{CEU results.}
    \end{subfigure}
    \hfill
    \begin{subfigure}{0.3\textwidth}
        \centering
        \includegraphics[width=\linewidth]{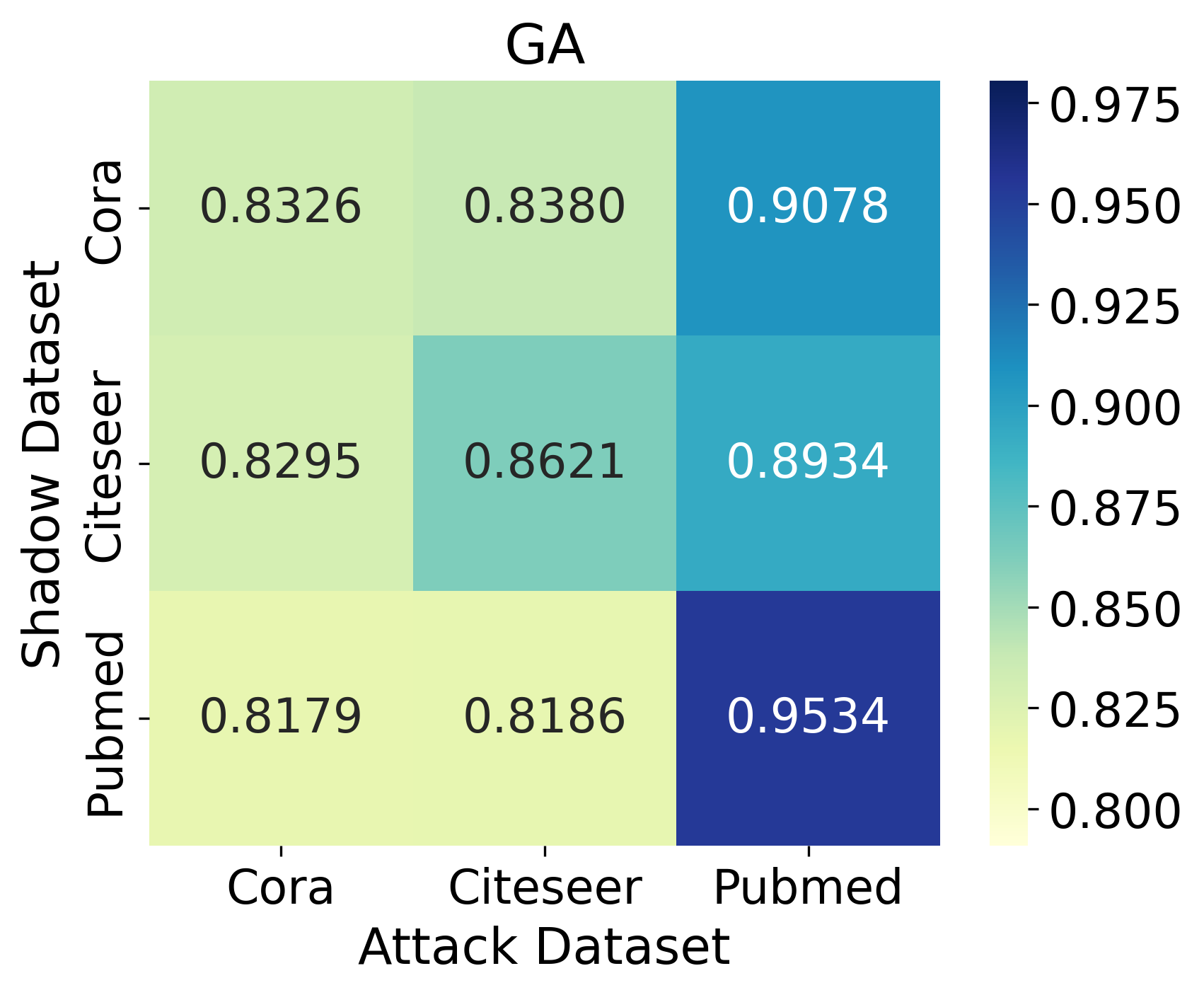}
        \caption{GA results.}
    \end{subfigure}
    \vskip -0.1in
    \caption{Transferability of TrendAttack across different datasets on Cora, Citeseer, and Pubmed. Each heatmap shows attack AUC when the shadow dataset (rows) and attack dataset (columns) differ.}
    \label{fig:data_transfer}
\end{figure*}

\noindent\textbf{Impact of Unlearning Ratio.} In this study, we examine whether the victim model's randomly selected unlearning edge ratio affects the overall attack AUC. Specifically, we fix the unlearning method to GIF and use TrendAttack-SL as our model variant, with results presented in Figure~\ref{fig:ablation_ratio}. From the figure, we find that compared to our most important baseline, StealLink, TrendAttack consistently performs better with less fluctuation. This indicates that our model remains relatively stable across different unlearning ratios, making it more universally applicable to various attack settings.

\noindent{\bf Dataset Transferability.} We evaluate the ability of TrendAttack-SL to transfer across datasets using a GCN backbone. For each unlearning method (GIF, CEU, GA), we train the attack on one dataset (shadow dataset) and apply it to a different dataset (attack dataset). Figures~\ref{fig:data_transfer} (a)-(c) present the results on Cora, Citeseer, and Pubmed. Overall, the attack maintains strong performance even when the shadow and attack datasets differ, with only slight decreases in some cases. This demonstrates the stability and generalization ability of TrendAttack-SL across datasets, indicating that the attack does not strongly depend on having the same dataset for training the shadow model.

\begin{table*}[!ht]
\centering
\caption{Comparison of our attack setting with previous unlearning inversion attacks. Our setting targets GNNs and assumes no access to the original model $f(\cdot;\thetab_{\mathrm{orig}})$, which is unlikely to be shared by responsible model owners due to its memorization of sensitive contents. Although we use a shadow dataset, our assumption of only black-box access to the unlearned model $f(\cdot;\thetab_{\mathrm{un}})$ makes our setting one of the strictest and most practical in the unlearning inversion literature.}
\label{tab:attack_comparison}
\vskip -1em
\begin{tabular}{lccccc}
\toprule
\textbf{Method} & \textbf{Access to $f(\cdot;\thetab_{\mathrm{orig}})$} & \textbf{Access to $f(\cdot;\thetab_{\mathrm{un}})$} & \textbf{Shadow Dataset} & \textbf{Victim Model} & \textbf{Attacker’s Goal} \\
\midrule
ReconAttack~\cite{bertran2024reconstruction} & White-box & White-box & No & Linear regression & Unlearned features \\
UnlearnInv-Feature~\cite{hu2024learn} & White-box & White-box & No & DNNs & Unlearned features \\
UnlearnInv-Label~\cite{hu2024learn} & Black-box & Black-box & No & DNNs & Unlearned labels \\
\textbf{Ours} & \textbf{No} & \textbf{Black-box} & \textbf{Yes} & \textbf{GNNs} & \textbf{Unlearned edges} \\
\bottomrule
\end{tabular}
\end{table*}

\section{Discussions} \label{sec:discussion}

In this section, we discuss the rationale of our attack settings and present potential future directions for potential defence strategies of our work.

\subsection{Comparison to Previous Unlearning Inversion Attacks} \label{sec:compare_prev_inv_atk}

Compared to prior unlearning inversion attacks~\cite{bertran2024reconstruction,hu2024learn} (see Table~\ref{tab:attack_comparison}), our attack is conducted under one of the strictest threat models in the literature. The main difference lies in model access assumptions: both ReconAttack~\cite{bertran2024reconstruction} and UnlearnInv~\cite{hu2024learn} adopt a \textbf{two-model} setting, where the attacker can access both the original model $f_{\mathrm{orig}}$ (before unlearning) and the unlearned model $f_{\mathrm{un}}$ (after unlearning). In contrast, our method uses a \textbf{one-model} setting, where the attacker can only access the unlearned model $f_{\mathrm{un}}$. This difference is important, as in realistic deployments the pre-unlearning model, which memorizes sensitive data, is rarely released. Therefore, our attack setting is fundamentally different and reflects practical scenarios. Detailed comparisons are as follows.

{\bf Difference to ReconAttack~\cite{bertran2024reconstruction}.}  
This work targets linear regression models with different loss functions, while our attack addresses a much more complex setting involving GNNs. Our method explicitly considers how unlearning effects propagate over graph structures and tackles challenges from deep, non-linear architectures, which go beyond simple regression models. Moreover,~\cite{bertran2024reconstruction} assumes white-box access to both the original and unlearned models’ parameters (see “...the parameter difference between the two models…” on page 2 of~\cite{bertran2024reconstruction}), whereas our method only requires black-box access to the unlearned model’s outputs.

{\bf Difference to UnlearnInv~\cite{hu2024learn}.}  
This work presents two settings. The feature attack (Section 3.2.1 of~\cite{hu2024learn}) assumes white-box access to both models and aims to recover training features. The label attack (Section 3.2.2 of~\cite{hu2024learn}) assumes black-box access to both models and infers training labels. Our TrendAttack, by contrast, requires only black-box access to $f_{\mathrm{un}}$ and uses a shadow dataset. Instead of recovering features or labels, we focus on reconstructing unlearned edges, an attack goal uniquely important for graph-structured data.

\subsection{The Necessity of Shadow Datasets} 

In this work, we first train the attack model on the shadow dataset $\Gcal^{\mathrm{sha}}$, and then transfer the trained attack model to the target dataset and its corresponding victim model (i.e., the unlearned GNN trained on the target dataset). 
An alternative strategy is to discard the shadow dataset and instead train a link prediction model directly on the target dataset $\mathcal{G}_{\mathrm{un}}$, which is a more intuitive way to recover unlearned edges from the cleaned graph. 
However, due to API query limitations imposed by many model-serving platforms and online social networks (e.g., see the query limitation policies of Twitter\footnote{\url{https://developer.x.com/en/docs/x-api}} and TikTok\footnote{\url{https://developers.tiktok.com/doc/research-api-faq}}), attackers can typically collect only a small number of user profiles and social connections within a given time window. As a result, it is difficult to gather sufficient training data directly on the target dataset, making such direct link prediction approaches less practical in real-world settings. Thus, our setting of training the attack model on a shadow dataset with a similar distribution to the target dataset (e.g., Facebook in the US $\rightarrow$ Facebook in the UK) provides access to more training data and is therefore more practical under such widespread query limitations.

\subsection{Potential Defence Strategies} \label{sec:potential_defence}

{\bf Leveraging Existing Defences for MIAs.} 
Since our proposed attack is a significantly more challenging variant of traditional membership inference attacks (MIAs), defences developed for MIAs may still offer protection. For example, training differentially private GNNs, as suggested in~\cite{hu2024learn}, could help mitigate our attack.

{\bf Strengthening Unlearning Mechanisms.} 
Our attack exploits the residual effects left by imperfect unlearning. Increasing the unlearning magnitude (i.e., allocating a stronger privacy budget) could help reduce such effects. However, finding a good balance between model utility and privacy remains a challenge.

{\bf Post-Unlearning Mitigation Strategies.} 
It may also be helpful to adopt mitigation strategies for general-purpose unlearning inversion attacks on Euclidean (i.e., non-graph) data. Techniques such as parameter pruning or fine-tuning have been shown to reduce privacy leakage~\cite{hu2024learn}, and could be applied to reduce the effectiveness of inversion attacks like ours.

\ifdefined\isarxiv

\balance
\clearpage
\section*{Ethical Consideration}

In this work, we propose a novel privacy attack targeting unlearned GNNs, revealing critical privacy vulnerabilities in existing graph unlearning methods. Our goal is to raise awareness about privacy protection in Web services and to inspire future research on privacy-preserving graph machine learning. While our method introduces a new attack strategy, all experiments are conducted on publicly available benchmark datasets. Given the gap between this pioneering study and real-world deployment scenarios, we do not foresee significant negative societal implications from this work.
\bibliographystyle{ACM-Reference-Format}
\bibliography{ref}

\else

\fi

\end{document}